\documentclass[twoside,11pt]{article}
\usepackage{journal2e_v1}

\usepackage{epsf}
\usepackage{epsfig}
\usepackage{amsmath}
\usepackage{subfigure}
\usepackage{amssymb}
\usepackage{algorithm}
\usepackage{multirow}
\usepackage{multicol}

\usepackage{graphicx} % more modern
\usepackage{algorithm}
\usepackage{algorithmic}

\def\A{{\bf A}}
\def\Amp{{\A^{\dag}}}
\def\a{{\bf a}}
\def\B{{\bf B}}
\def\bb{{\bf b}}
\def\C{{\bf C}}
\def\Cmp{{\C^{\dag}}}

\def\D{{\bf D}}
\def\d{{\bf d}}
\def\E{{\bf E}}

\def\F{{\bf F}}

\def\I{{\bf I}}

\def\R{{\bf R}}

\def\s{{\bf s}}

\def\U{{\bf U}}
\def\u{{\bf u}}
\def\V{{\bf V}}
\def\v{{\bf v}}
\def\W{{\bf W}}
\def\w{{\bf w}}
\def\X{{\bf X}}
\def\x{{\bf x}}
\def\Y{{\bf Y}}

\def\Z{{\bf Z}}

\def\0{{\bf 0}}
\def\1{{\bf 1}}

\def\UA{{{\U}_{\A}}}

\def\UAk{{{\U}_{\A, k}}}
\def\UAkp{{{\U}_{\A, k \perp}}}

\def\VAk{{{\V}_{\A, k}}}

\def\SiA{{{\Si}_{\A}}}

\def\SiAk{{{\Si}_{\A, k}}}
\def\SiAkp{{{\Si}_{\A, k \perp}}}

\def\OM{{\mathcal O}}

\def\UM{{\mathcal U}}
\def\VM{{\mathcal V}}

\def\RB{{\mathbb R}}
\def\RBmn{{\RB^{m\times n}}}
\def\EB{{\mathbb E}}

\def\Si{\mbox{\boldmath$\Sigma$\unboldmath}}

\def\Lam{\mbox{\boldmath$\Lambda$\unboldmath}}

\def\argmin{\mathop{\rm argmin}}

\def\span{\mathrm{span}}
\def\tr{\mathrm{tr}}
\def\rk{\mathrm{rank}}

\def\Diag{\mathsf{Diag}}

\def\CUR{{CUR} }

\begin{document}
\title{A Scalable CUR Matrix Decomposition Algorithm: Lower Time Complexity and Tighter Bound\thanks{An extended abstract of this paper has been accepted by NIPS 2012.}}

\author{\name Shusen Wang and Zhihua Zhang \\
\addr College of Computer Science \& Technology \\
Zhejiang University \\
Hangzhou, China 310027 \\
\texttt{\{wss,zhzhang\}@zju.edu.cn} \\
\name Jian Li \\
\addr Google \\
Beijing, China 100084 \\
\texttt{lijian@google.com}
}

\maketitle

\begin{abstract}
The \CUR matrix decomposition is an important extension of Nystr\"{o}m approximation to a general matrix.
It approximates any data matrix in terms of a small number of its columns and rows.
In this paper we propose a novel randomized \CUR algorithm with an expected relative-error bound.
The proposed algorithm has the advantages over the existing relative-error \CUR algorithm that it possesses
tighter theoretical bound and lower time complexity, and that it can avoid maintaining the whole data matrix in main memory.
Finally, experiments on several real-world datasets  demonstrate significant improvement over the existing relative-error algorithms.
\end{abstract}


\section{Introduction}

Large-scale  matrices emerging from stocks, genomes, web documents, web images and videos
everyday bring new challenges in modern data analysis.
Most efforts have been focused on manipulating, understanding and interpreting large-scale data matrices.
In many cases, matrix factorization methods are employed to construct compressed and
informative representations to facilitate computation and interpretation.
A principled approach is  the truncated singular value decomposition (SVD)
which finds the best low-rank approximation of a  data matrix.
Applications of SVD such as eigenface~\citep{Sirovich87eigenface,turk1991eigenface} and
latent semantic analysis~\citep{deerwester1990lsa} have been illustrated to be very successful.

However, the  basis vectors resulting from SVD have little concrete meaning, which
makes it very difficult for us to understand and interpret the data in question.
An example in \citep{drineas2008cur,mahoney2009matrix} has well shown this viewpoint; that is,
the vector $[(1/2)\textrm{age} - (1/\sqrt{2})\textrm{height} + (1/2)\textrm{income}]$,
 the sum of the significant uncorrelated features from a dataset of people's features,
is not particularly informative.
\citet{kuruvilla2002vector} have also claimed:
``it would be interesting to try to find basis vectors for all experiment vectors,
using actual experiment vectors and not artificial bases that offer little insight.''
Therefore, it is of great interest to represent a data matrix in terms of
a small number of actual columns and/or actual rows of the matrix.

The {\it  \CUR matrix decomposition}  provides such techniques,
and it has been shown to be very useful in high dimensional data analysis~\citep{mahoney2009matrix}.
Given a matrix $\A$, the \CUR technique  selects a subset of columns
of $\A$ to construct a matrix $\C$ and a subset of rows of $\A$ to construct a matrix $\R$,
and computes a matrix $\U$ such that $\tilde{\A} = \C \U \R$ best approximates $\A$.
The typical \CUR algorithms~\citep{drineas2003pass,drineas04fastmonte,drineas2008cur}
work in a two-stage manner. Stage~1 is a standard column selection procedure,
and Stage~2 does  row selection from $\A$ and $\C$ simultaneously. Thus,
implementing Stage~2 is  much more difficult than doing Stage~1.

The CUR matrix decomposition problem is widely studied in the literature \citep{goreinov1997pseudoskeleton,goreinov1997maximalvolume,tyrtyshnikov2000incompletecross,drineas2003pass,drineas2005nystrom,drineas04fastmonte,drineas2008cur,mahoney2009matrix,mackey2011divide,hopcroft2012computer}.
Among the existing work, several recent work are of particular interest.
\citet{drineas04fastmonte} proposed a \CUR algorithm with additive-error bound.
Later on, \citet{drineas2008cur} devised randomized \CUR algorithms
with relative error by sampling sufficiently many columns and rows. Particularly,
the algorithm has $(1+\epsilon)$ relative-error ratio with high probability (w.h.p.).
Recently, \citet{mackey2011divide} established a divide-and-conquer method which solves the \CUR problem in parallel.

Unfortunately, all the existing \CUR algorithms require a large number of columns and rows to be chosen.
For example, for an $m\times n$ matrix $\A$ and a target rank $k \leq \min\{m,n\}$,
the state-of-the-art \CUR algorithm --- {\it the subspace sampling algorithm} in \citet{drineas2008cur} ---
requires exactly $\OM(k^4 \epsilon^{-6})$ rows or $\OM(k \epsilon^{-4} \log^2 k)$
rows in expectation to achieve $(1+\epsilon)$ relative-error ratio w.h.p.
Moreover, the computational cost of this algorithm is at least the cost of the truncated SVD of $\A$,
that is, $\OM(\min\{mn^2, nm^2\})$.\footnote{Although some partial SVD algorithms, such as Krylov subspace methods, require only $\OM(mnk)$ time, they are all numerical unstable. See \cite{halko2011ramdom} for more discussions.}
The algorithms are therefore impractical for large-scale matrices.

In this paper we develop a \CUR algorithm which beats the state-of-the-art algorithm in both theory and experiments.
In particular, we show in Theorem~\ref{cor:fast_cur} a novel randomized \CUR algorithm with lower time complexity
and tighter theoretical bound in comparison with the state-of-the-art \CUR algorithm in \citet{drineas2008cur}.

The rest of this paper is organized as follows.
Section~\ref{sec:notation} lists some notations that will be used in this paper and Section~\ref{sec:related} reviews two classes of \CUR algorithms.
Section~\ref{sec:theoretical} mainly introduces a column selection algorithm to which our work is closely related.
Section~\ref{sec:fast_cur} describes and analyzes our novel \CUR algorithm.
Section~\ref{sec:experiments} empirically compares our proposed algorithm with the state-of-the-art algorithm.
All proofs are deferred to Appendix~\ref{sec:proofs}.

\section{Notations}
\label{sec:notation}

For a matrix $\A=[a_{ij}] \in \RB^{m\times n}$,
let $\a^{(i)}$ be its $i$-th row and $\a_j$ be its $j$-th column.
%and $a_{i, j}$ be the $(i,j)$-th entry of $\A$.
%Let $\A[k\! :\! l, :]$ ($k\leq l$) denote the $(l-k)\times n$ submatrix
%consisting of the $k$-th to $l$-th rows of $\A$
%and $\A[k\! :\! l, p\! :\! q]$ denote the submatrix consisting of the $p$-th to $q$-th columns of $\A[k\! :\! l, :]$.
Let $\|\A\|_1 = \sum_{i,j} |a_{i j}|$ be the $\ell_1$-norm,
$\|\A\|_F= (\sum_{i, j} a_{ij}^2)^{1/2}$ be the Frobenius norm,
and $\|\A\|_{2} = \max_{\|\x\|_2=1} \|\A \x \|_2$ be the spectral norm.
Moreover, let $\I_m$ denote the $m\times m$ identity matrix,
and $\0$ denotes the  zero matrix whose size dependents on the context.
Let $\rho = \rk(\A)$ and $k \leq \rho$, the SVD of $\A$ can be written as
\[
\A = \sum_{i=0}^{\rho} \sigma_{\A,i} \u_{\A,i} \v^T_{\A,i} = \UA \SiA \V_{\A}^T = \UAk \SiAk \V_{\A, k}^T + \UAkp \SiAkp \V_{\A, k\perp}^T,
\]
where $\UAk$, $\SiAk$, and $\VAk$ correspond to the top $k$ singular values.
We denote $\A_k = \UAk \SiAk \V_{\A, k}^T$.
Furthermore, let $\Amp = \U_{\A,\rho} \Si_{\A,\rho}^{-1} \V_{\A,\rho}^T$ be the Moore-Penrose inverse of $\A$ \citep{adi2003inverse}.

Given matrices $\A \in \RBmn$, $\X \in \RB^{m\times p}$, and $\Y \in \RB^{q\times n}$,
$\X \X^\dag \A = \U_\X \U_{\X}^T \A \in \RBmn$ is the projection of $\A$ onto the column space of $\X$,
and  $\A \Y^\dag \Y = \A \V_\Y \V_{\Y}^T \in \RBmn$ is the projection of $\A$ onto the row space of $\Y$.
Finally, given an integer $k \leq p$,
we define the matrix $\Pi_{\X,k} (\A) \in \RBmn$ as the best approximation to $\A$ within the column space of $\X$ that has rank at most $k$.
We have  $\Pi_{\X,k} (\A) = \X \hat{\Z}$ where $\hat{\Z} = \argmin_{\rk(\Z) \leq k} \|\A - \X \Z\|_F$.
We also have that $\|\A - \X \X^\dag \A \|_F \leq \|\A - \Pi_{\X,k} (\A)\|_F$.

\section{Previous Work in \CUR Matrix Decomposition}\label{sec:related}

This section discusses two recent developments of the \CUR algorithms.
Section~\ref{sec:add_error_cur} introduces an additive-error \CUR algorithm in \cite{drineas04fastmonte}, and
Section~\ref{sec:stat_leverage} describes two relative-error \CUR algorithms in \cite{drineas2008cur}.

\subsection{The Linear-Time \CUR Algorithm} \label{sec:add_error_cur}

The linear-time \CUR algorithm is proposed by \cite{drineas04fastmonte}.
It is a highly efficient algorithm.
Given a matrix $\A$ and a constant $k < \rk(\A)$,
by sampling $c = 64k\epsilon^{-4}$ columns and $r = 4k\epsilon^{-2}$ rows of $\A$
and computing an intersection matrix $\U$,
the resulting \CUR decomposition satisfies the following additive-error bound
\begin{equation}
\EB \|\A - \C \U \R\|_F \leq \|\A - \A_k\|_F + \epsilon \|\A\|_F \textrm{.} \nonumber
\end{equation}
Furthermore, the decomposition also satisfies $\rk(\C \U \R) \leq k$.
Here we give its main results \cite[Theorem~4 of ][]{drineas04fastmonte} in the following proposition.

\begin{proposition}[The Linear-Time \CUR Algorithm]
Given a matrix $\A \in \RB^{m\times n}$, we let $p_i = \|\a^{(i)}\|_2^2 / \|\A\|_F^2$ and $q_j = \|\a_j\|_2^2 / \|\A\|_F^2$.
The linear-time \CUR algorithm randomly samples $c$ columns of $\A$ with probabilities $\{q_j\}_{j=1}^n$ and $r$ rows of $\A$
with probabilities $\{p_i\}_{i=1}^m$. Then
\begin{equation}
\EB \|\A - \C \U \R\|_F \leq \|\A - \A_k\|_F + \Big( (4k/c)^{1/4} + (k/r)^{1/2} \Big) \|\A\|_F \textrm{.} \nonumber
\end{equation}
The algorithm costs $\OM(mc^2 + nr + c^2r + c^3)$ time, which is linear in $(m+n)$ by assuming $c$ and $r$ are constants.
\end{proposition}

\subsection{The Subspace Sampling \CUR Algorithm} \label{sec:stat_leverage}

\cite{drineas2008cur} proposed a two-stage randomized \CUR algorithm
which has a relative-error bound w.h.p.
In the first stage the algorithm samples $c$ columns of $\A$ to construct $\C$,
and in the second stage it samples $r$ rows from $\A$ and $\C$ simultaneously to construct $\R$ and $\U^\dag$.
In the first stage the sampling probabilities are proportional to the squared $\ell_2$-norm of the rows of $\V_{\A,k}$,
in the second stage the sampling probabilities are proportional to the squared $\ell_2$-norm of the rows of $\U_{\C,k}$.
That is why it is called the ``subspace sampling algorithm''.
Here we show the main results of the subspace sampling algorithms in the following proposition.

\begin{proposition}[The Subspace Sampling \CUR Algorithm]
Given a matrix $\A \in \RB^{m\times n}$ and an integer $k \ll \min\{m,n\}$,
the subspace sampling algorithm uses {\it exactly sampling} to select exactly $c = \OM(k^2 \epsilon^{-2} \log(1/\delta))$ columns of $\A$ to construct $\C$,
and then exactly $r = \OM(c^2 \epsilon^{-2} \log(1/\delta))$ rows of $\A$ to construct $\R$,
or uses {\it expected sampling} to select $c = \OM(k \epsilon^{-2} \log k \log(1/\delta))$ columns
and $r = \OM(c \epsilon^{-2} \log c \log(1/\delta))$ rows in expectation.
Then with probability at least $(1-\delta)$,
\begin{equation}
\|\A - \C \U \R\|_F \leq (1+\epsilon) \|\A - \A_k\|_F \textrm{.} \nonumber
\end{equation}
Here, the matrix $\U$ is a weighted Moore-Penrose inverse of the intersection between $\C$ and $\R$.
The running time of both algorithms is dominated by the truncated SVD of $\A$.
\end{proposition}

Although the algorithm is $\epsilon$-optimal with high probability,
it requires too many rows get chosen:
at least $r = \OM (k \epsilon^{-4} \log^2 k)$ rows in expectation.
%Moreover, this algorithm has no guarantee of expected bound.
In this paper we seek to devise an algorithm with mild requirement on column and row numbers.

\section{Theoretical Backgrounds} \label{sec:theoretical}

Section~\ref{sec:connection} considers the connections between the column selection problem and the \CUR matrix decomposition problem.
Section~\ref{sec:col_selection} introduces a near-optimal relative-error column selection algorithm.
Our proposed \CUR algorithm is motivated by and partly based on the near-optimal column selection algorithm.

\subsection{Connections between Column Selection and CUR Matrix Decomposition} \label{sec:connection}

Column selection is a well-established problem which has been widely studied in the literature: \citep{frieze2004fast,deshpande2006matrix,drineas2008cur,deshpande2010efficient,boutsidis2011NOCshort,Guruswami2012optimal}.

Given a matrix $\A \in \RBmn$,
column selection aims to choose $c$
columns of $\A$ to construct $\C \in \RB^{m\times c}$
so that $\|\A - \C \Cmp \A \|_F$ achieves the minimum.
%whose span nearly contains $\span(\U_{\A,k})$.
Since there are $(^n_c)$ possible choices of constructing $\C$,
so selecting the best subset is a hard problem.  In recent years,
many polynomial-time approximate algorithms have been proposed,
among which we are particularly interested in those algorithms with relative-error bounds; that is,
with $c \geq k$ columns selected from $\A$, there is a constant $\eta$ such that
\begin{equation}
\|\A - \C \Cmp \A \|_F \leq \eta \| \A - \A_k \|_F \textrm{.} \nonumber
\end{equation}
We call $\eta$ the {\it relative-error ratio}.
For some randomized algorithms, the inequality holds either w.h.p. or in expectation w.r.t. $\C$.

The \CUR  matrix decomposition problem has a close connection with the column selection problem.
As aforementioned, the first stage of existing \CUR algorithms is simply a column selection procedure.
However, the second stage is more complicated.
If the second stage is na\"ively solved by a column selection algorithm on $\A^T$,
then the error ratio will trivially be $2\eta$.

For a relative-error \CUR algorithm,
the first stage seeks to bound a construction error ratio of $\frac{\|\A - \C \C^\dag \A \|_F}{\| \A - \A_k \|_F}$,
while the section stage seeks to bound $\frac{\|\A - \C \C^\dag \A \R^\dag \R \|_F}{\| \A - \C \C^\dag \A \|_F}$ given $\C$.
Actually, the first stage is a special case of the second stage where $\C = \A_k$.
%\begin{equation}
%\frac{\|\A - \A_k \A_k^\dag \A \R^\dag \R \|_F}{\| \A - \A_k \A_k^\dag \A \|_F}
%= \frac{\|\A - \A_k \R^\dag \R \|_F}{\| \A - \A_k \|_F}
%\geq \frac{\|\A - \A \R^\dag \R \|_F}{\| \A - \A_k \|_F} \textrm{.}
%\end{equation}
Given a matrix $\A$, if an algorithm solving the second stage results in a bound $\frac{\|\A - \C \C^\dag \A \R^\dag \R \|_F}{\| \A - \C \C^\dag \A \|_F} \leq \eta$,
then this algorithm also solves the column selection problem for $\A^T$ with an $\eta$ relative-error ratio.
Thus the second stage of \CUR is a generalization of the column selection problem.

\subsection{The Near-Optimal Column Selection Algorithm} \label{sec:col_selection}

Recently, \citet{boutsidis2011NOC} proposed a randomized algorithm which selects only
$c = { 2k \epsilon^{-1}}(1 + o(1))$ columns to achieve the expected relative-error ratio ($1+\epsilon$).
\citet{boutsidis2011NOC} also proved the lower bound of the column selection problem; that is,
at least $c = k\epsilon^{-1}$ columns are selected to achieve the $(1+\epsilon)$ ratio.
Thus this algorithm is near optimal.
Though an optimal algorithm recently proposed by \citet{Guruswami2012optimal} achieves the the lower bound,
the optimal algorithm is quite inefficient compared with the near-optimal algorithm.

The near-optimal algorithm has three steps:
the approximate SVD via random projection~\citep{halko2011ramdom},
the dual set sparsification algorithm~\citep{boutsidis2011NOC},
and the adaptive sampling algorithm~\citep{deshpande2006matrix}.
Here we present the main results of this algorithm in Lemma~\ref{prop:fast_column_selection}.
To better understand the algorithm, we also give the details of the three steps, respectively.

\begin{lemma}[Near-Optimal Column Selection Algorithm] \label{prop:fast_column_selection}
Given a matrix $\A \in \RBmn$ of rank $\rho$,
a target rank $k$ $(2 \leq k < \rho)$,
and $0 < \epsilon < 1$,
there exists a randomized algorithm to select at most
\begin{equation}
c = \frac{2k}{\epsilon} \Big( 1 + o(1) \Big) \nonumber
\end{equation}
columns of $\A$ to form a matrix $\C \in \RB^{m \times c}$ such that
\begin{equation}
\EB^2	\|\A - \C \Cmp \A \|_F \leq \EB \|\A - \C \Cmp \A \|_F^2 \leq (1 + \epsilon) \|\A - \A_k \|_F^2 \textrm{,} \nonumber
\end{equation}
where the expectations are taken w.r.t. $\C$.
Furthermore, the matrix $\C$ can be obtained in $\OM((mnk + nk^3)\epsilon^{-2/3})$.
\end{lemma}

The dual set sparsification algorithm requires the top $k$ right singular vectors of $\A$ as inputs.
Since SVD is time consuming, \citet{boutsidis2011NOC} employed an approximation SVD algorithm~\citep{halko2011ramdom} to speedup computation.
We give the theoretical analysis of the approximation SVD via random projection in Lemma~\ref{lem:rand_svd}.
The resulting matrix $\Z$ approximates $\V_{\A,k}$.

\begin{lemma}[Randomized SVD via Random Projection]\label{lem:rand_svd}
Given a matrix $\A \in \RB^{m\times n}$ of rank $\rho$, a target rank $k$ \emph{($k<\rho$)}, and $0 < \epsilon_0 < 1$,
the algorithm computes a factorization $\A = \B \Z^T + \E$ with $\B = \A \Z$, $\Z^T \Z = \I_k$, and $\E \Z = \0$ such that
\[
\EB \|\E\|_F^2 \leq (1+\epsilon_0) \|\A - \A_k\|_F^2.
\]
The algorithm runs in $\OM(m n k \epsilon_0^{-1})$ time.
\end{lemma}

The second step of the near-optimal column selection algorithm is the {\it dual set sparsification} proposed by \cite{boutsidis2011NOC}.
When ones take $\A$ and the top $k$ (approximate) right singular vectors of $\A$ as inputs,
the dual set sparsification algorithm can deterministically selects $c_1$ columns of $\A$ to construct $\C_1$.
We present their results in Lemma~\ref{prop:deterministic_fro} and attach the concrete algorithm in Appendix~\ref{sec:dualset}.

\begin{lemma}[Column Selection via Dual Set Sparsification Algorithm] \label{prop:deterministic_fro}
Given a matrix $\A \in \RBmn$ of rank $\rho$ and a target rank $k$ $(< \rho)$,
the dual set spectral-Frobenius sparsification algorithm deterministically selects $c_1$ $(> k)$ columns of $\A$ to form a matrix $\C_1 \in \RB^{m\times c_1}$ such that
\[
\Big\|\A - \Pi_{\C_1,k} (\A) \Big\|_F \leq \sqrt{ 1 + \frac{1}{(1 - \sqrt{k/c_1})^2} }\, \Big\|\A - \A_k \Big\|_F.
\]
Moreover, the matrix $\C_1$ can be computed in $T_{\V_{\A,k}} + \OM(mn + nc_1 k^2)$,
where $T_{\V_{\A,k}}$ is the time needed to compute the top $k$ right singular vectors of $\A$.
\end{lemma}

After sampling $c_1$ columns of $\A$, the near-optimal column selection algorithm uses the adaptive sampling of~\citet{deshpande2006matrix}
to select $c_2$ columns of $\A$ to further reduce the construction error.
We present Theorem~2.1 in \citet{deshpande2006matrix} in the following lemma.

\begin{lemma} [The Adaptive Sampling Algorithm] \label{lem:ada_sampling}
Given a matrix $\A \in \RBmn$, we let $\C_1\in \RB^{m\times c_1}$ consists of $c_1$ columns of $\A$,
and define the residual $\B = \A - \C_1 \C_1^\dag \A$. Additionally,
for $i = 1,\cdots, n$, we define
\[
p_i = \|\bb_i\|_2^2 / \|\B\|_F^2.
\]
We further  sample  $c_2$ columns  i.i.d. from $\A$,
 in each trial of which the $i$-th column is chosen with probability $p_i$.
Let $\C_2 \in \RB^{m\times c_2}$ contain the $c_2$ sampled rows
and let $\C = [\C_1,\C_2] \in \RB^{m\times (c_1+c_2)}$.
Then, for any integer $k > 0$, the following inequality holds:
\[
\EB \|\A - \C \Cmp \A\|_F^2 \leq \|\A - \A_k \|_F^2 + \frac{k}{r_2} \|\A - \C_1 \C_1^\dag \A \|_F^2,
\]
where the expectation is taken w.r.t. $\C_2$.
\end{lemma}

\section{Main Results} \label{sec:fast_cur}

In this section we develop a novel \CUR algorithm that
we call  {\it the fast \CUR algorithm} due to its lower time complexity in comparison with SVD.
We describe the procedure in Algorithm~\ref{alg:fast_cur}
and give theoretical analysis in Theorem~\ref{cor:fast_cur}.

The main results of our work are formally shown in  three theorems in this section.
The proofs are deferred to Appendix~\ref{sec:proofs}.
Theorem~\ref{cor:fast_cur} relies on Lemma~\ref{prop:fast_column_selection} and Theorem~\ref{thm:modified_fast_row_selection}, and
Theorem~\ref{thm:modified_fast_row_selection} relies on Theorem~\ref{thm:adaptive_bound}.
Theorem~\ref{thm:adaptive_bound} is a generalization of Lemma~\ref{lem:ada_sampling},
and Theorem~\ref{thm:modified_fast_row_selection} is a generalization of Lemma~\ref{prop:fast_column_selection}.

\begin{algorithm}[tb]
   \caption{The Fast \CUR Algorithm.}
   \label{alg:fast_cur}
\algsetup{indent=2em}
\begin{small}
\begin{algorithmic}[1]
   \STATE {\bf Input:} a real matrix $\A \in \RBmn$, target rank $k$, $\epsilon \in (0, 1]$,
   			target column number $c = \frac{2k}{\epsilon} \big(1+o(1)\big)$, target row number $r = \frac{2c}{\epsilon} \big(1+o(1)\big)$;
   \STATE {\it // Stage~1: select $c$ columns of $\A$ to construct $\C \in \RB^{m\times c}$}
   \STATE Compute approximate truncated SVD via random projection such that $\A_k \approx \tilde{\U}_k \tilde{\Si}_k \tilde{\V}_k$; \label{alg:fast_cur:line3}
   \STATE Construct $\UM_1 \leftarrow$ columns of $(\A - \tilde{\U}_k \tilde{\Si}_k \tilde{\V}_k)$; $\quad \VM_1 \leftarrow$ columns of $\tilde{\V}_k^T$; \label{alg:fast_cur:line4}
   \STATE Compute $\s_1 \leftarrow$ Dual Set Spectral-Frobenius Sparsification Algorithm ($\UM_1$, $\VM_1$, $c - 2k/\epsilon$);
   \STATE Construct $\C_1 \leftarrow \A \Diag(\s_1)$, and then delete the all-zero columns;  \label{alg:fast_cur:line6}
   \STATE Residual matrix $\D \leftarrow \A - \C_1 \C_1^\dag \A$;
   \STATE Compute sampling probabilities: $p_i = \|\d_i\|_2^2 / \|\D\|_F^2$, $i = 1, \cdots, n$;
   \STATE Sampling $c_2 = 2k/\epsilon$ columns from $\A$ with probability $\{p_1,\cdots,p_n\}$ to construct $\C_2$;
   \STATE {\it // Stage~2: select $r$ rows of $\A$ to construct $\R \in \RB^{r\times n}$}
   \STATE Construct $\UM_2 \leftarrow$ columns of $(\A - \tilde{\U}_k \tilde{\Si}_k \tilde{\V}_k)^T$; $\quad \VM_2 \leftarrow$ columns of $\tilde{\U}_k^T$;  \label{alg:fast_cur:line11}
   \STATE Compute $\s_2 \leftarrow$ Dual Set Spectral-Frobenius Sparsification Algorithm ($\UM_2$, $\VM_2$, $r - 2c/\epsilon$);
   \STATE Construct $\R_1 \leftarrow \Diag(\s_2) \A$, and then delete the all-zero rows;  \label{alg:fast_cur:line13}
   \STATE Residual matrix $\B \leftarrow \A - \A \R_1^\dag \R_1$;  \label{alg:fast_cur:line14}
   \STATE Compute sampling probabilities: $q_j = \|\bb^{(j)}\|_2^2 / \|\B\|_F^2$, $j = 1, \cdots, m$;
   \STATE Sampling $r_2 = 2c/\epsilon$ rows from $\A$ with probability $\{q_1, \cdots, q_m\}$ to construct $\R_2$;  \label{alg:fast_cur:line16}
   \RETURN $\C = [\C_1 , \C_2]$, $\R = [\R^T_1 , \R_2^T]^T$, and $\U = \C^\dag \A \R^\dag$.
\end{algorithmic}
\end{small}
\end{algorithm}

\subsection{Adaptive Sampling}

The relative-error adaptive sampling algorithm is established in Theorem~2.1 of \citet{deshpande2006matrix}.
The algorithm is based on the following idea:
after selecting a proportion of columns from $\A$ to form $\C_1$ by an arbitrary algorithm,
the algorithms randomly samples additional $c_2$ columns according to the residual $\A - \C_1 \C_1^\dag \A$.
\cite{boutsidis2011NOC} used the adaptive sampling algorithm to decrease the residual of the dual set sparsification algorithm and obtained an $(1+\epsilon)$ relative-error ratio.
Here we prove a new bound for the same adaptive sampling algorithm.
Interestingly, this new bound is a generalization of the original one in Theorem~2.1 of \citet{deshpande2006matrix}.
In other words, Theorem~2.1 of \citet{deshpande2006matrix} is a direct corollary of our following theorem when  $\C = \A_k$.

\begin{theorem} [The Adaptive Sampling Algorithm] \label{thm:adaptive_bound}
Given a matrix $\A \in \RBmn$ and a matrix $\C \in \RB^{m\times c}$
such that $\rk(\C) = \rk(\C \C^\dag \A) = \rho$ $(\rho \leq c \leq n)$,
we let $\R_1 \in \RB^{r_1 \times n}$ consist of $r_1$ rows of $\A$,
and define the residual $\B = \A - \A \R_1^\dag \R_1$. Additionally,
for $i = 1,\cdots, m$, we define
\[
p_i = \|\bb^{(i)}\|_2^2 / \|\B\|_F^2.
\]
We further  sample  $r_2$ rows  i.i.d. from $\A$,
 in each trial of which the $i$-th row is chosen with probability $p_i$.
Let $\R_2 \in \RB^{r_2\times n}$ contain the $r_2$ sampled rows
and let $\R = [\R_1^T,\R_2^T]^T \in \RB^{(r_1+r_2)\times n}$.
Then the following inequality holds:
\begin{equation}
\EB \|\A - \C \Cmp \A \R^\dag \R \|_F^2 \leq \|\A - \C \C^\dag \A \|_F^2 + \frac{\rho}{r_2} \|\A - \A \R_1^\dag \R_1\|_F^2 \textrm{,} \nonumber
\end{equation}
where the expectation is taken w.r.t. $\R_2$.
\end{theorem}

\subsection{The Fast \CUR Algorithm}

Based on the randomized SVD algorithm of Lemma~\ref{lem:rand_svd},
the dual set sparsification algorithm of Lemma~\ref{prop:deterministic_fro},
and the adaptive sampling algorithm of Theorem~\ref{thm:adaptive_bound},
we develop a randomized algorithm to solve the second stage of the \CUR problem.
We present the results of the algorithm in the following theorem.

\begin{theorem} [The Fast Row Selection Algorithm] \label{thm:modified_fast_row_selection}
Given a matrix $\A \in \RBmn$ and a matrix $\C \in \RB^{m\times c}$
such that $\rk(\C) = \rk(\C \C^\dag \A) = \rho$ $(\rho \leq c \leq n)$,
and a target rank $k$ $(\leq \rho)$,
the proposed randomized algorithm selects $r = \frac{2\rho}{\epsilon} ( 1 + o(1) )$ rows
of $\A$ to construct $\R\in \RB^{r\times n}$, such that
\begin{equation}
\EB \|\A - \C \Cmp \A \R^\dag \R \|_F^2
\leq \|\A - \C \Cmp \A\|_F^2 + \epsilon \|\A - \A_k\|_F^2 \textrm{,} \nonumber
\end{equation}
where the expectation is taken w.r.t. $\R$.
Furthermore, the matrix $\R$ can be computed in $\OM( (mnk + mk^3) \epsilon^{-2/3} )$ time.
\end{theorem}

Note that Lemma~\ref{prop:fast_column_selection}, i.e., Theorem~5 of \citet{boutsidis2011NOC}, is a special case of  Theorem~\ref{thm:modified_fast_row_selection}
when $\C = \A_k$. Based on Lemma~\ref{prop:fast_column_selection} and Theorem~\ref{thm:modified_fast_row_selection},
we have the main theorem for the fast \CUR algorithm as follows.

\begin{theorem} [The Fast \CUR Algorithm] \label{cor:fast_cur}
Given a matrix $\A \in \RBmn$ and a positive integer $k \ll \min\{m,n\}$,
the fast \CUR algorithm described in Algorithm~\ref{alg:fast_cur} randomly selects $c = \frac{2k}{\epsilon} (1 + o(1))$ columns of $\A$ to construct $\C\in \RB^{m\times c}$
with the near-optimal column selection algorithm of Lemma~\ref{prop:fast_column_selection},
and then selects $r=\frac{2c}{\epsilon} (1 + o(1))$ rows of $\A$ to construct $\R \in \RB^{r\times n}$
with the fast row selection algorithm of Theorem~\ref{thm:modified_fast_row_selection}.
Then we have
\begin{equation}
\EB \| \A - \C \U \R \|_F = \EB \| \A - \C (\C^\dag \A \R^\dag) \R \|_F \leq (1+ \epsilon) \|\A - \A_k\|_F \textrm{.} \nonumber
\end{equation}
Moreover, the algorithm runs in time $\OM \Big(mnk \epsilon^{-2/3} + (m+n)k^3 \epsilon^{-2/3} + mk^2 \epsilon^{-2} + nk^2\epsilon^{-4} \Big)$.
\end{theorem}

Since $k, c, r \ll \min\{m,n\}$ by the assumption,
so the time complexity of the fast \CUR algorithm is lower than that of the SVD of $\A$.
This is the main reason why we call it the fast \CUR algorithm.
% notice the difference between "the reason that" and "the reason why".

Another advantage of this algorithm is that it can avoid loading the whole $m\times n$ data matrix $\A$ into main memory.
None of three steps --- the randomized SVD, the dual set sparsification algorithm,
and the adaptive sampling --- requires loading the whole of $\A$ into memory.
The most memory-expensive operation throughout the fast \CUR Algorithm is computing the Moore-Penrose inverses of $\C$ and $\R$,
which requires maintaining an $m\times c$ matrix or an $r \times n$ matrix in memory.
In contrast, the subspace sampling algorithm requires loading the whole matrix into memory to compute its truncated SVD.

\section{Empirical Analysis} \label{sec:experiments}

In this section we conduct empirical comparisons among the relative-error
\CUR algorithms on several datasets.
We report the relative-error ratio and the running time of each algorithm on each data set.
The relative-error ratio is defined by
\[
\textrm{Relative-error ratio} = \frac{\| \A - \C \U \R\|_F}{\| \A - \A_k \|_F},
\]
where $k$ is a specified target rank.

\subsection{Datasets}

We implement experiments on five datasets, including natural images, biology data, and bags of words.
Table~\ref{tab:datasets}  briefly summarizes  some information of the datasets.
The Redrock and Edinburgh~\citep{agarwala2007efficient} are two large size natural images.
Arcene and Dexter are both from the UCI datasets~\citep{uci2010}.
Arcene is a biology dataset with $900$ instances and $10000$ attributes.
Dexter is a bag of words dataset with a $20000$-vocabulary and $2600$ documents.
PicasaWeb image dataset~\citep{wang2012data} contains $6.8$ million PicasaWeb images.
We use the HMAX features~\citep{serre07robustobject} and the SIFT features~\citep{lowe1999sift} of the first $50000$ images;
the features provided by \citet{wang2012data} are all of $3000$ dimensions.
Each dataset is actually represented as a data matrix, upon which we apply the \CUR algorithms.

When the data matrices become very large, e.g., say $8K \times 3K$,
the truncated SVD and the standard SVD are both infeasible in our experiment environment,
and so is the subspace sampling algorithm.
Therefore we do not conduct experiments on larger data matrices.
In contrast, our fast \CUR algorithm actually works well even for $30K \times 3K$ matrices.

\begin{table}[t]\setlength{\tabcolsep}{0.3pt}
\caption{A summary of the datasets.}
\label{tab:datasets}
\begin{center}
\begin{small}
\begin{tabular}{c c c c c }
\hline
	{\bf Dataset}	& 	{\bf Type}		&			{\bf size}		&		{\bf Source}	\\
\hline
		Redrock		&	natural image	&	$18000\times 4000$ 	&	\texttt{http://www.agarwala.org/efficient\_gdc/}	\\
		Edinburgh	&	natural image	&	$16500\times 1800$	&	\texttt{http://www.agarwala.org/efficient\_gdc/}	\\
		Arcene		&	biology			&	$10000\times 900$		&	\texttt{http://archive.ics.uci.edu/ml/datasets/Arcene}		\\
		Dexter		&	bag of words	&	$20000\times 2600$	&	\texttt{http://archive.ics.uci.edu/ml/datasets/Dexter}	\\
		PicasaWeb	&     features 	    &	$50000\times 3000$	&	\texttt{https://sites.google.com/site/picasawebdataset}	\\
%		KOS			&	bag of words	&	$6906 \times 3430$	&	\texttt{http://archive.ics.uci.edu/ml/datasets/Bag+of+Words}	\\
%		NIPS			&	bag of words	&	$12419\times 1500$	&	\texttt{http://archive.ics.uci.edu/ml/datasets/Bag+of+Words}	\\
\hline
\end{tabular}
\end{small}
\end{center}
\end{table}

\subsection{Setup}

We implement the subspace sampling algorithm and our fast CUR algorithm in MATLAB 7.10.0.
We do not compare with the linear-time \CUR algorithm for the following reason.
There is an implicit projection operation in the linear-time \CUR algorithm, so the result satisfies $\rk(\C \U \R) \leq k$.
However, this inequality does not hold for the subspace sampling algorithm and the fast CUR algorithm.
Thus, comparing the construction error among the three \CUR algorithm is very unfair for the linear-time \CUR algorithm.
Actually, the construction error of the linear-time \CUR algorithm is much worse than the other two algorithms.

We conduct experiments on a workstation with $12$ Intel Xeon $3.47$GHz CPUs, $12$GB memory, and Ubuntu 10.04 system.
According to the analysis in \citet{drineas2008cur} and this paper,
$k$, $c$, and $r$ should be integers much less than $m$ and $n$.
For each data set and each algorithm, we set $k = 10$, $20$, or $50$, and $c = \alpha k$, $r = \alpha c$,
where $\alpha$ ranges in each set of experiments.
We repeat each set of experiments for 20 times and report the average and the standard deviation of the error ratios.
The results are depicted  in Figures~\ref{fig:redrock},~\ref{fig:edinburgh},~\ref{fig:arcene},~\ref{fig:dexter},~\ref{fig:hmax}, and \ref{fig:sift}.

%---------------------------------Figure---------------------------------%
\begin{figure*}
\subfigtopskip = 0pt
\begin{center}
\centering
\includegraphics[width=48mm, height=45mm]{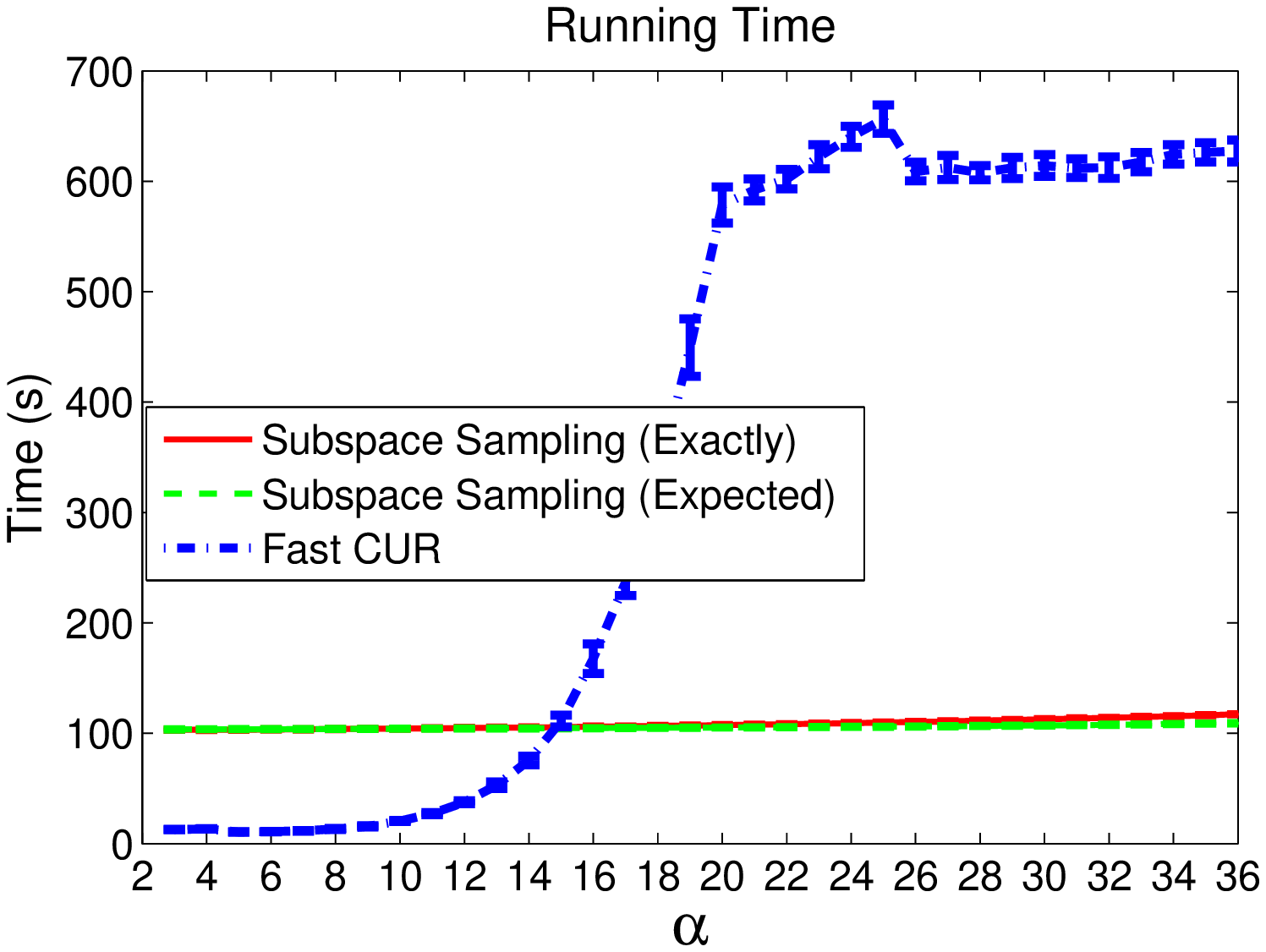}~
\includegraphics[width=48mm, height=45mm]{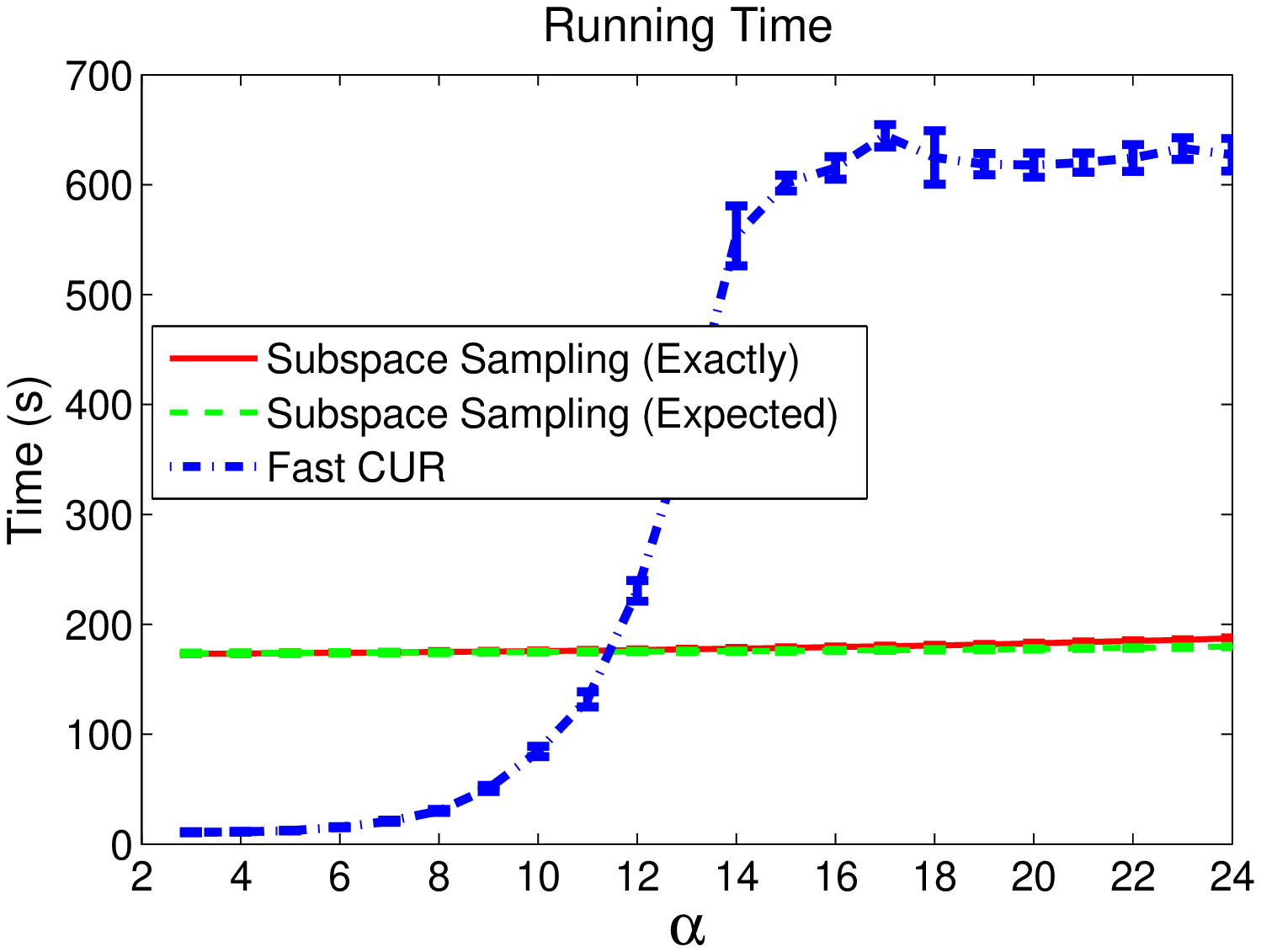}~
\includegraphics[width=48mm, height=45mm]{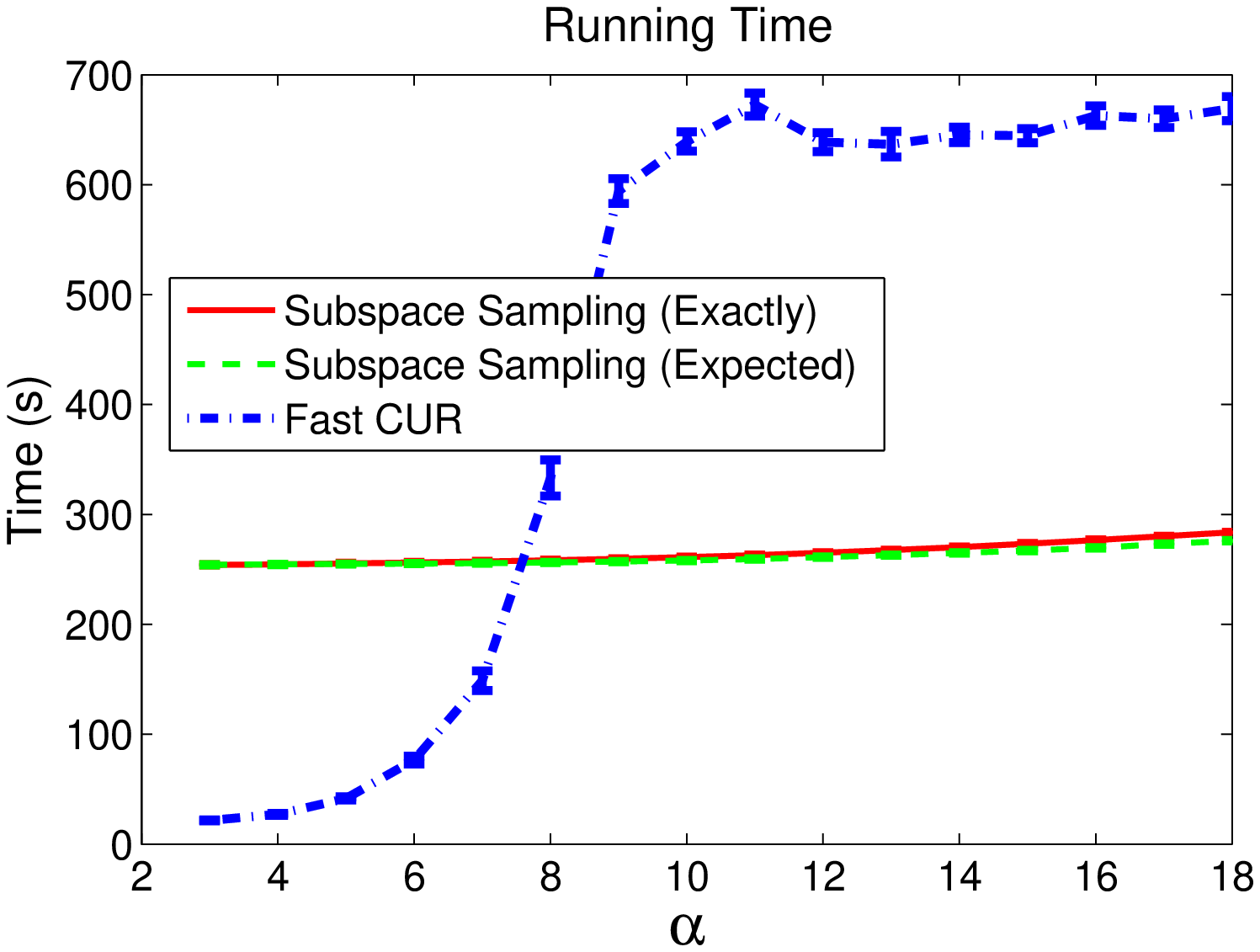} \\
\subfigure[\textsf{$k = 10$, $c=\alpha k$, and $r=\alpha c$.}]{\includegraphics[width=48mm, height=45mm]{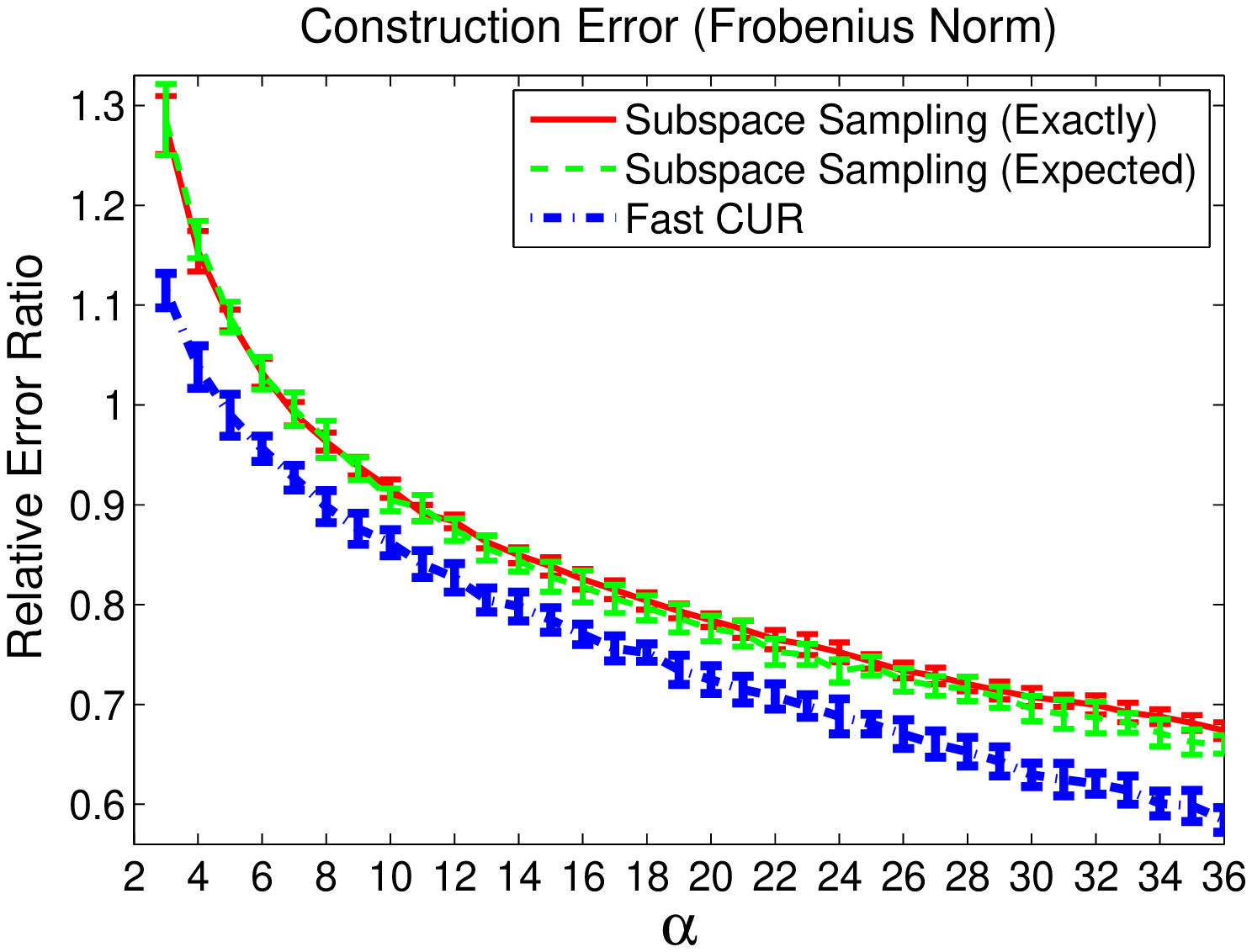}}~
\subfigure[\textsf{$k = 20$, $c=\alpha k$, and $r=\alpha c$.}]{\includegraphics[width=48mm, height=45mm]{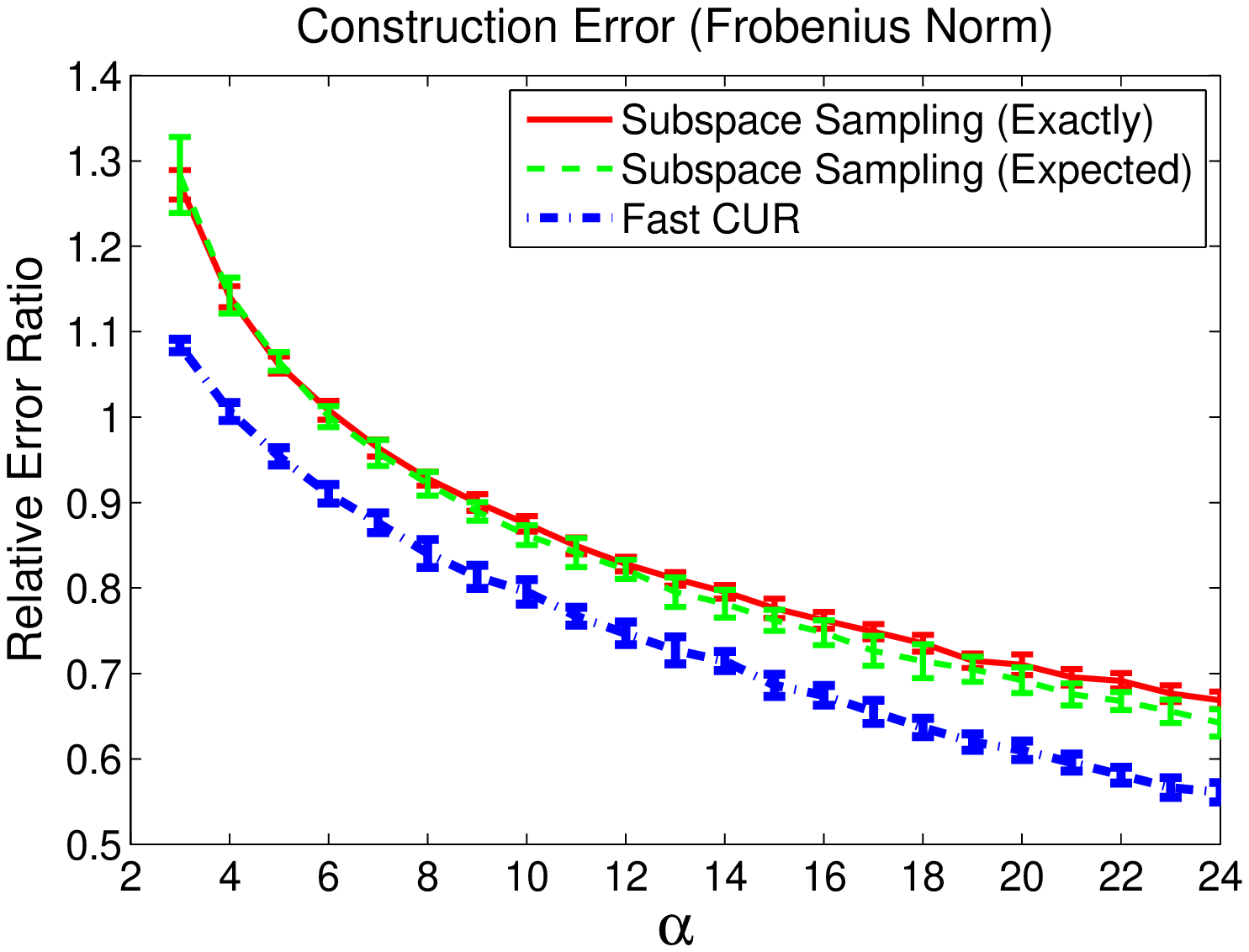}}~
\subfigure[\textsf{$k = 50$, $c=\alpha k$, and $r=\alpha c$.}]{\includegraphics[width=48mm, height=45mm]{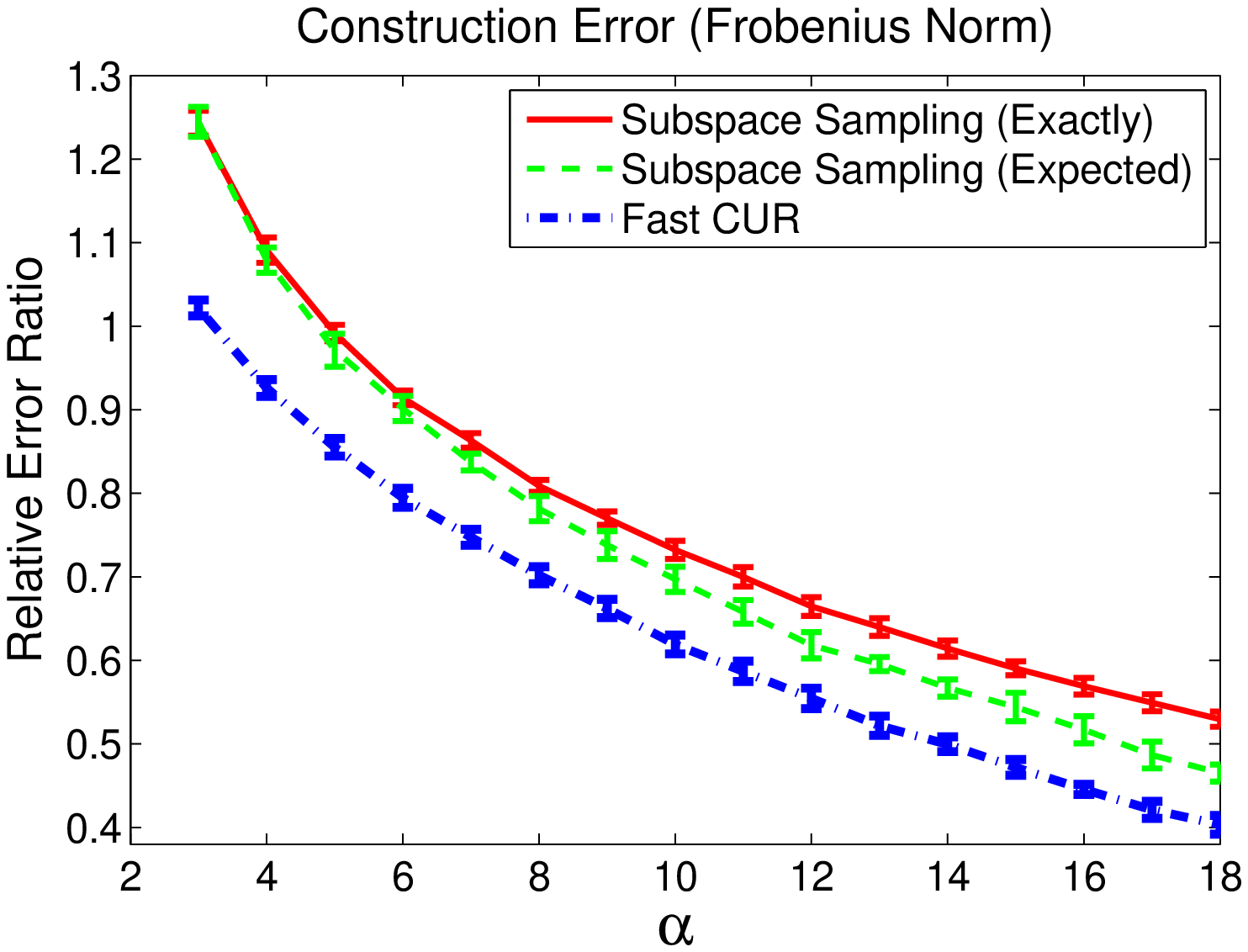}}
\end{center}
   \caption{Empirical results on the Redrock data set.}
\label{fig:redrock}
\end{figure*}
%---------------------------------Figure---------------------------------%

%---------------------------------Figure---------------------------------%
\begin{figure*}
\subfigtopskip = 0pt
\begin{center}
\centering
\includegraphics[width=48mm, height=40mm]{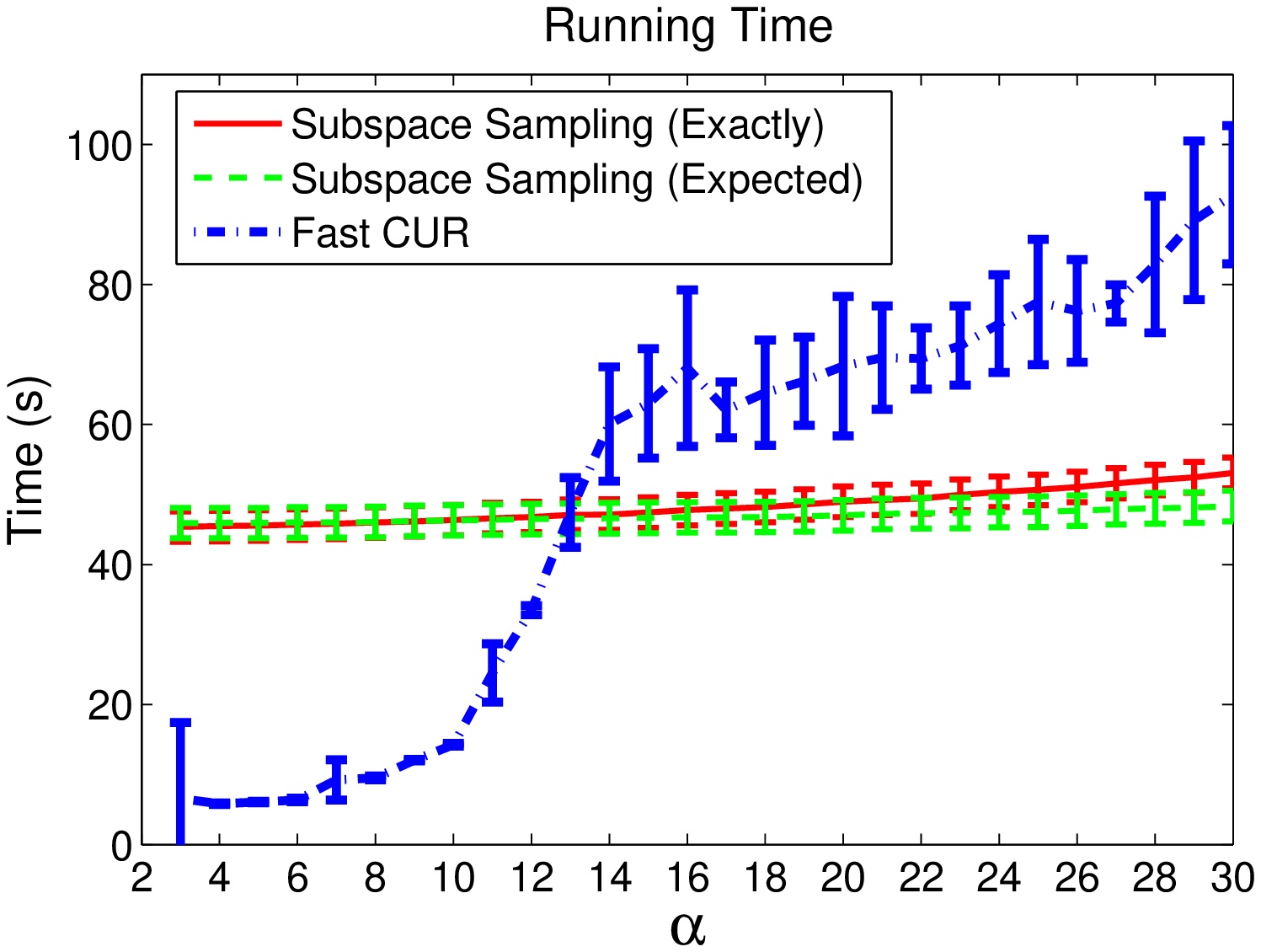}~
\includegraphics[width=48mm, height=40mm]{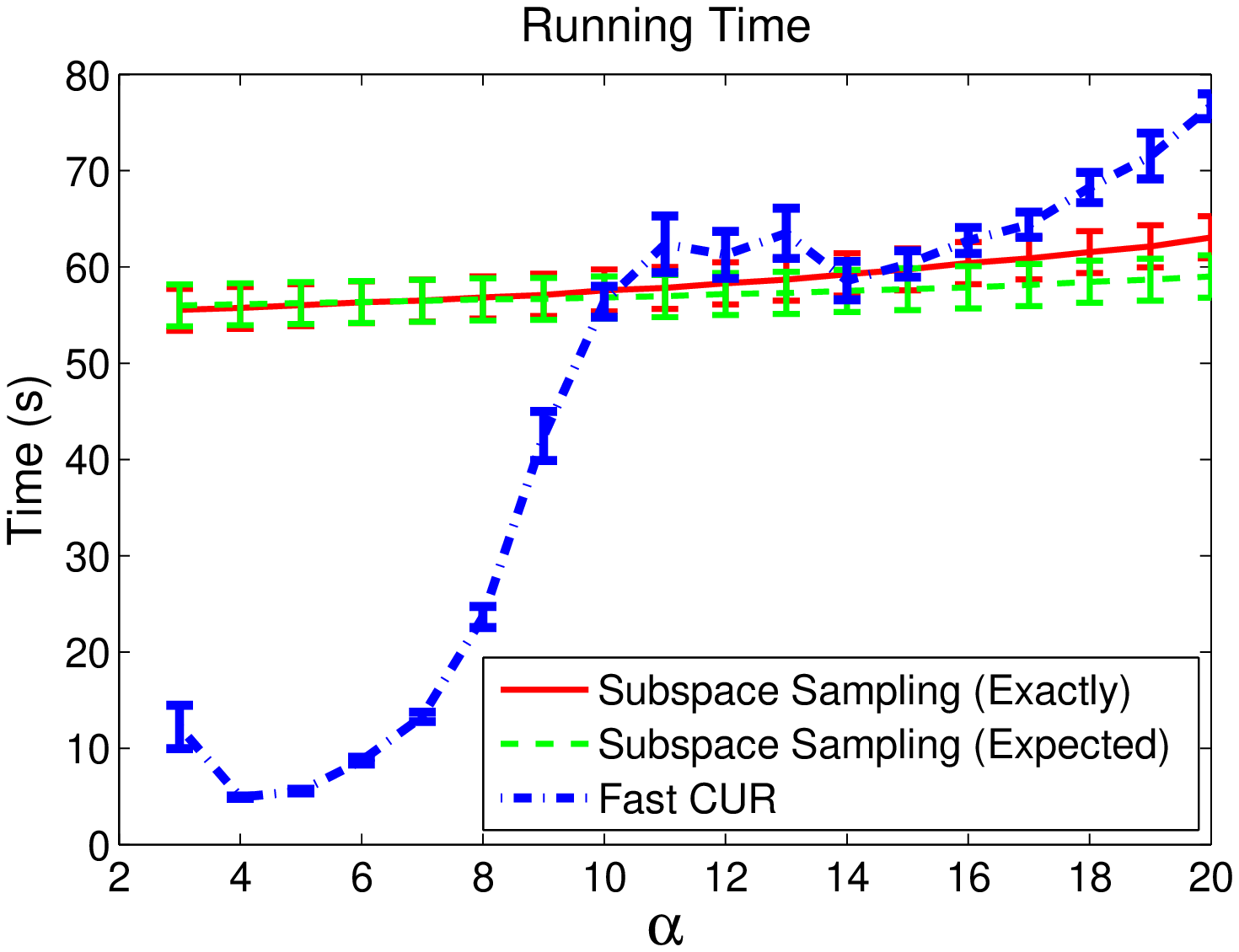}~
\includegraphics[width=48mm, height=40mm]{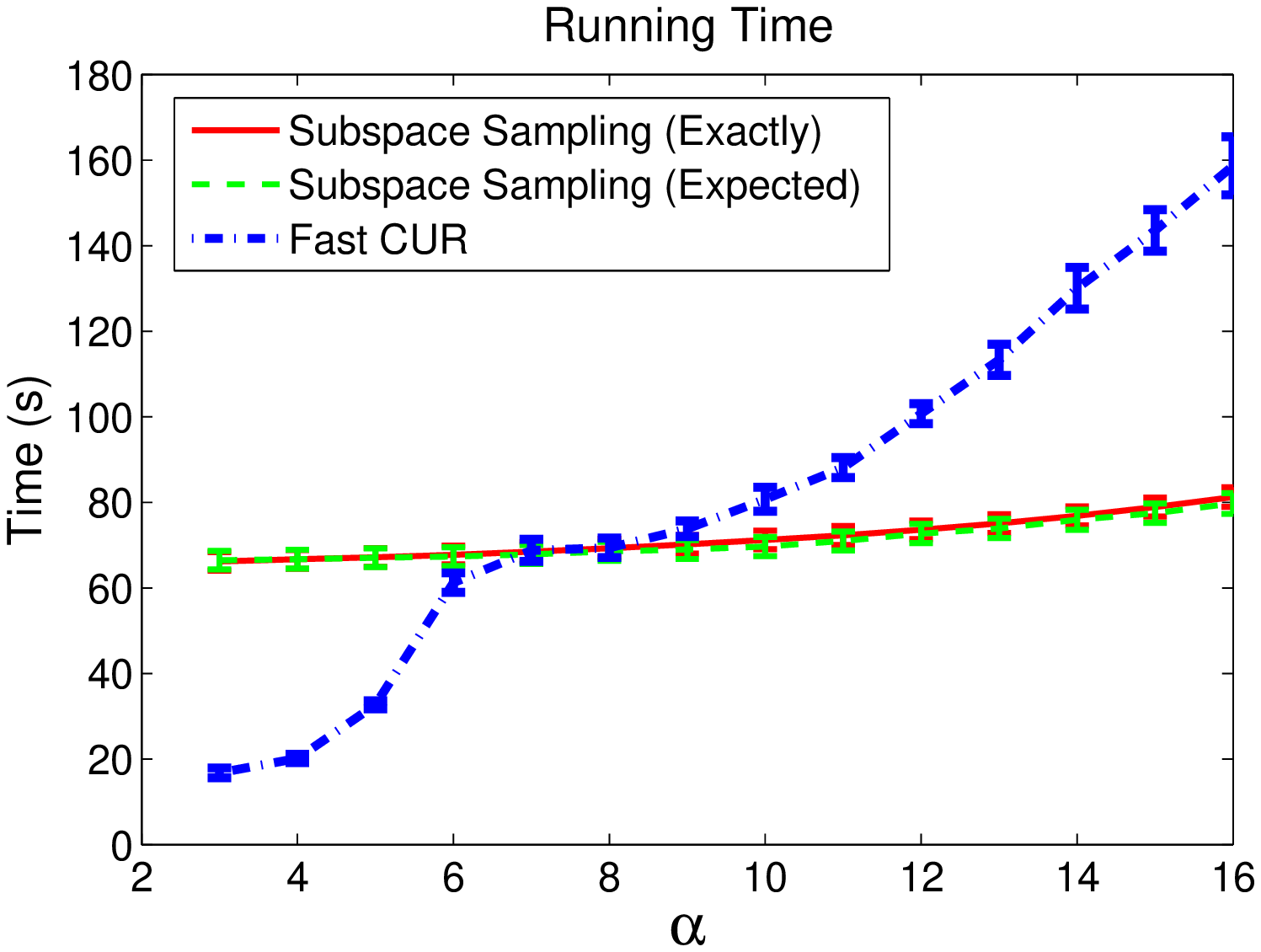} \\
\subfigure[\textsf{$k = 10$, $c=\alpha k$, and $r=\alpha c$.}]{\includegraphics[width=48mm, height=40mm]{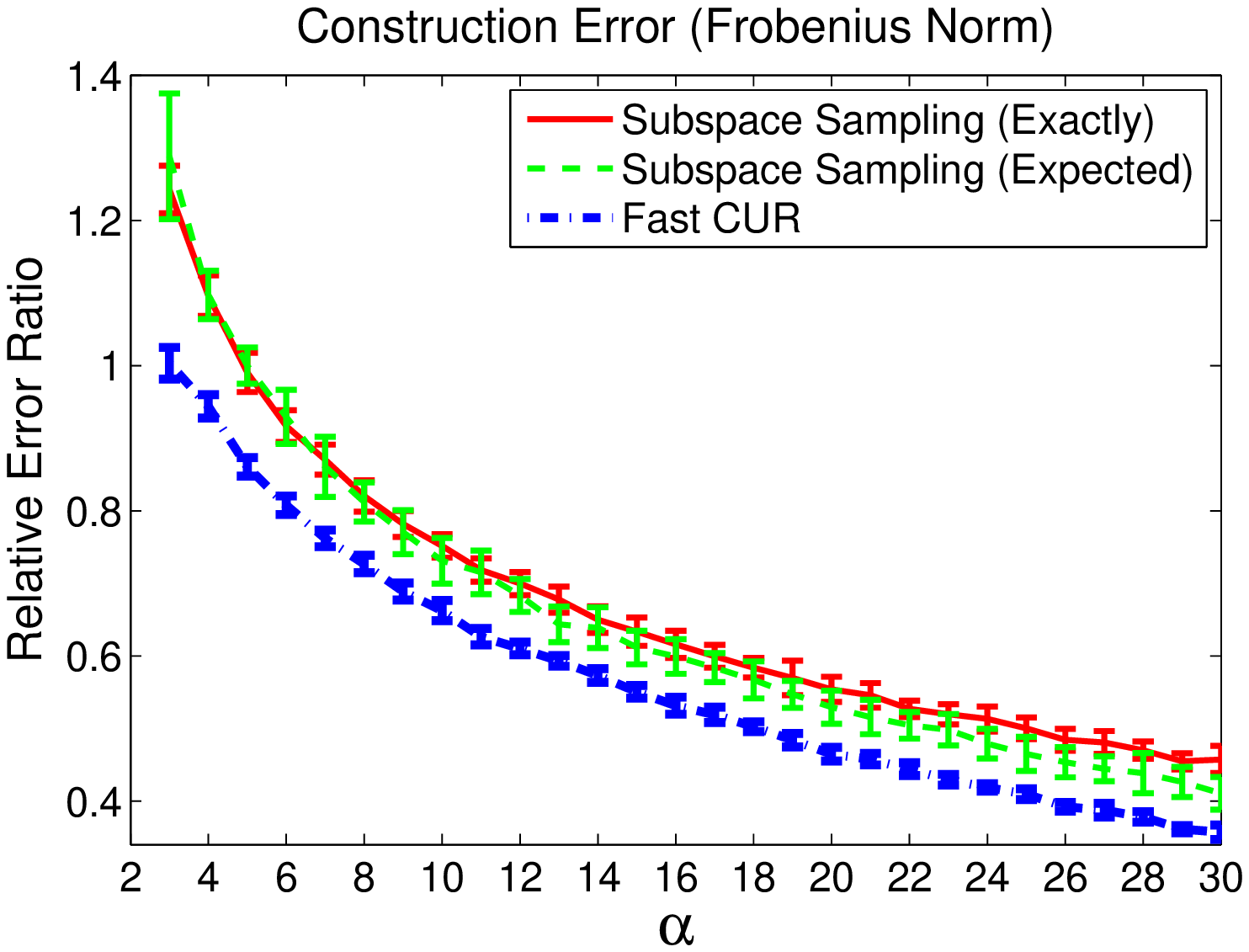}}~
\subfigure[\textsf{$k = 20$, $c=\alpha k$, and $r=\alpha c$.}]{\includegraphics[width=48mm, height=40mm]{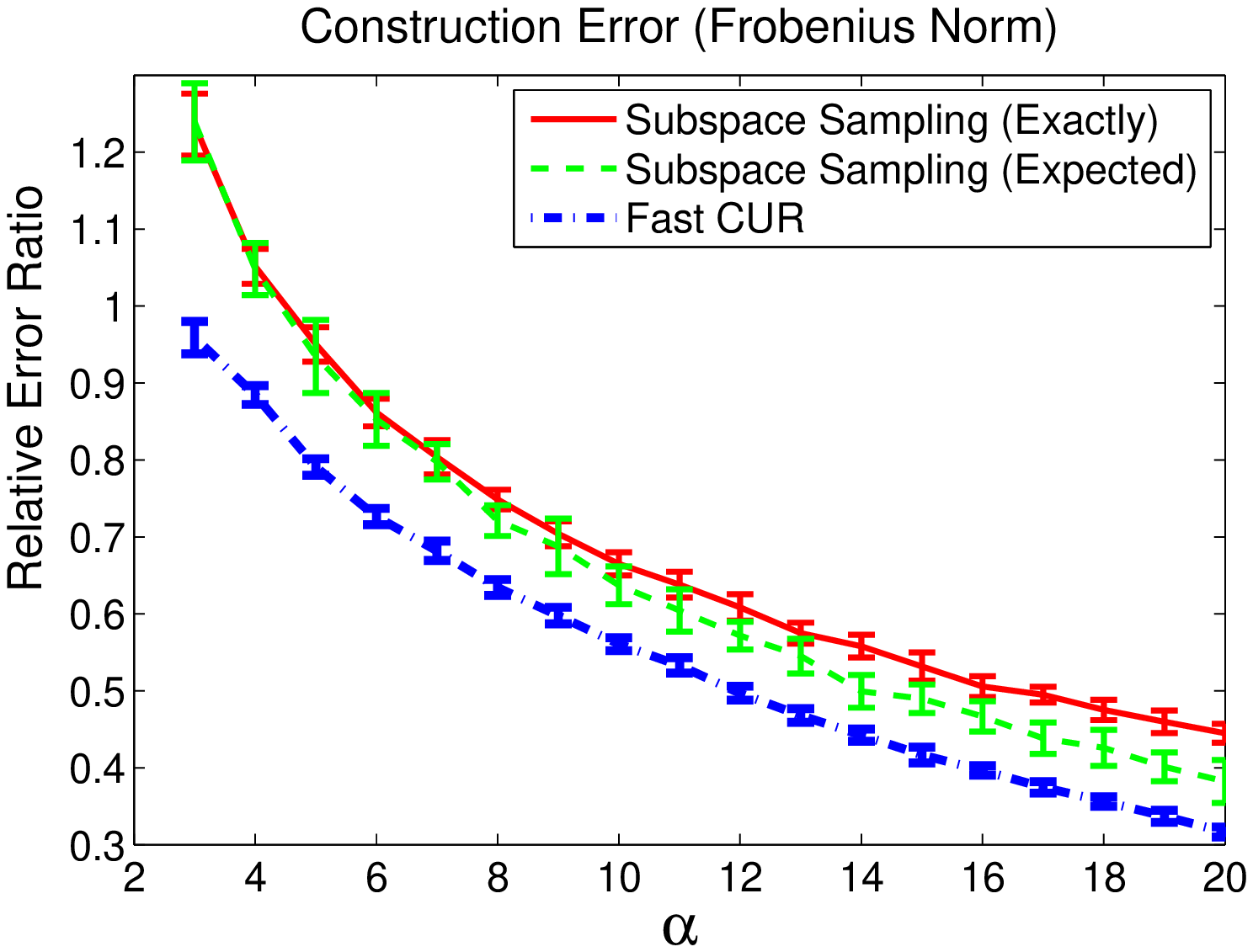}}~
\subfigure[\textsf{$k = 50$, $c=\alpha k$, and $r=\alpha c$.}]{\includegraphics[width=48mm, height=40mm]{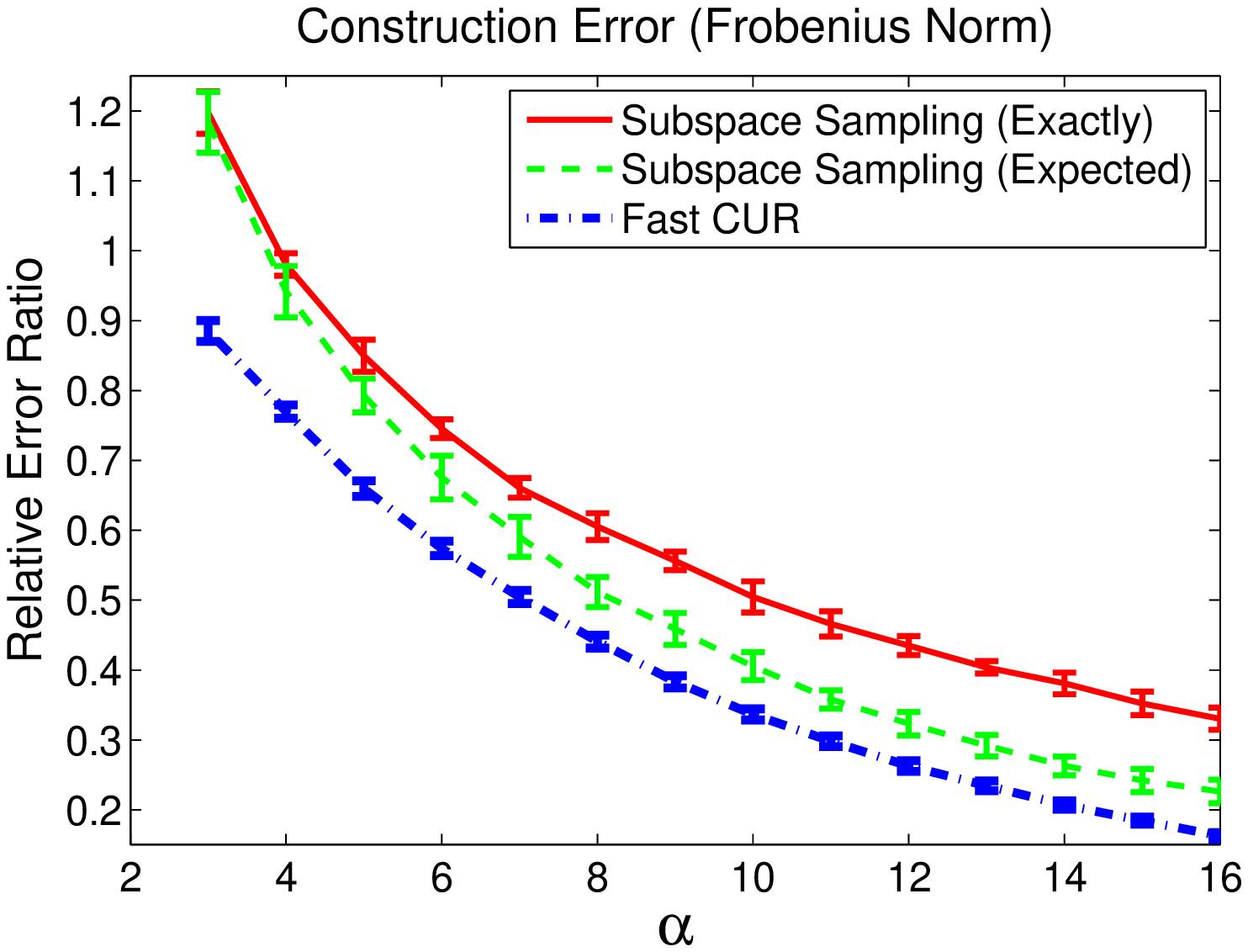}}
\end{center}
   \caption{Empirical results on the Edinburgh data set.}
\label{fig:edinburgh}
\end{figure*}
%---------------------------------Figure---------------------------------%

%---------------------------------Figure---------------------------------%
\begin{figure*}
\subfigtopskip = 0pt
\begin{center}
\centering
\includegraphics[width=48mm, height=40mm]{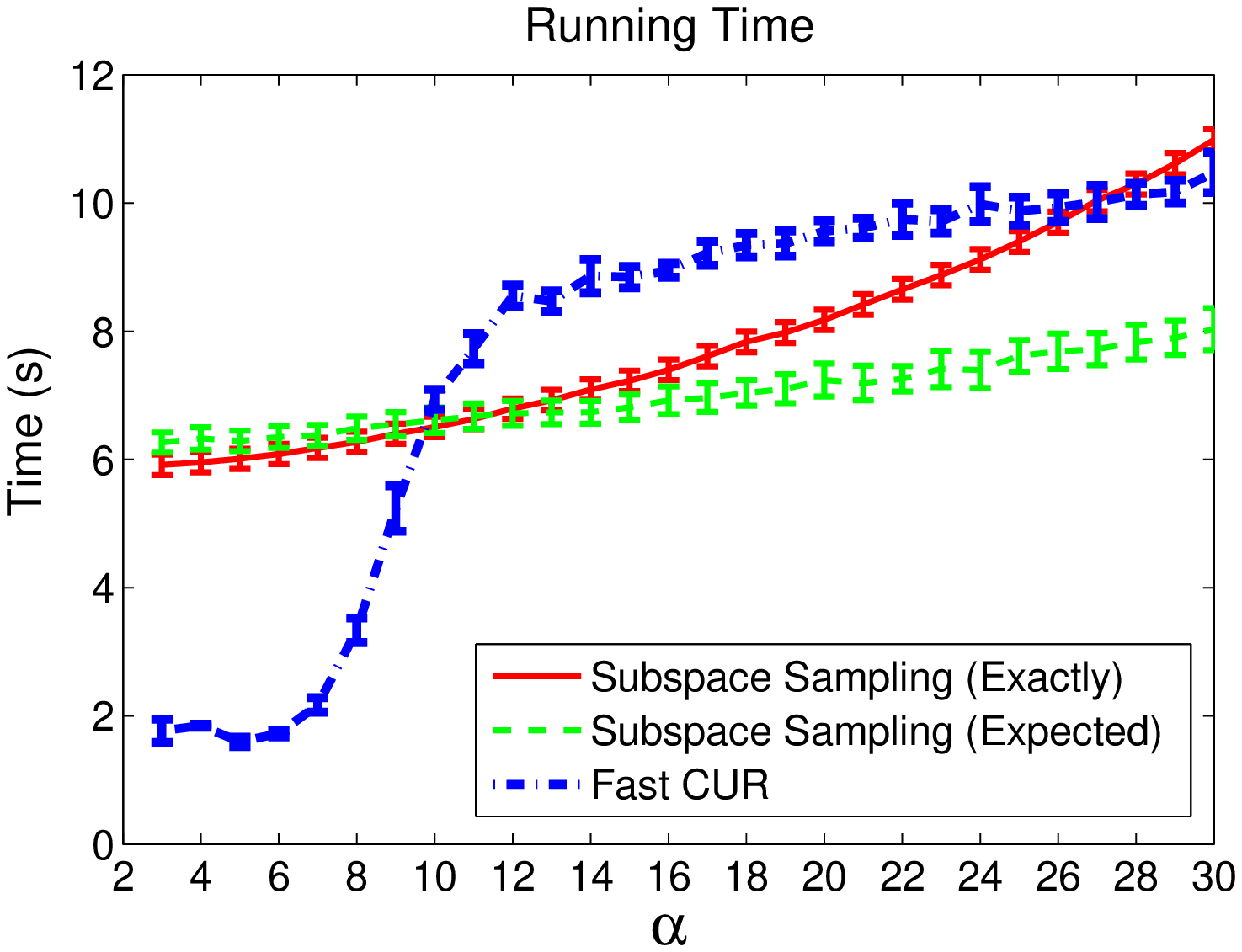}~
\includegraphics[width=48mm, height=40mm]{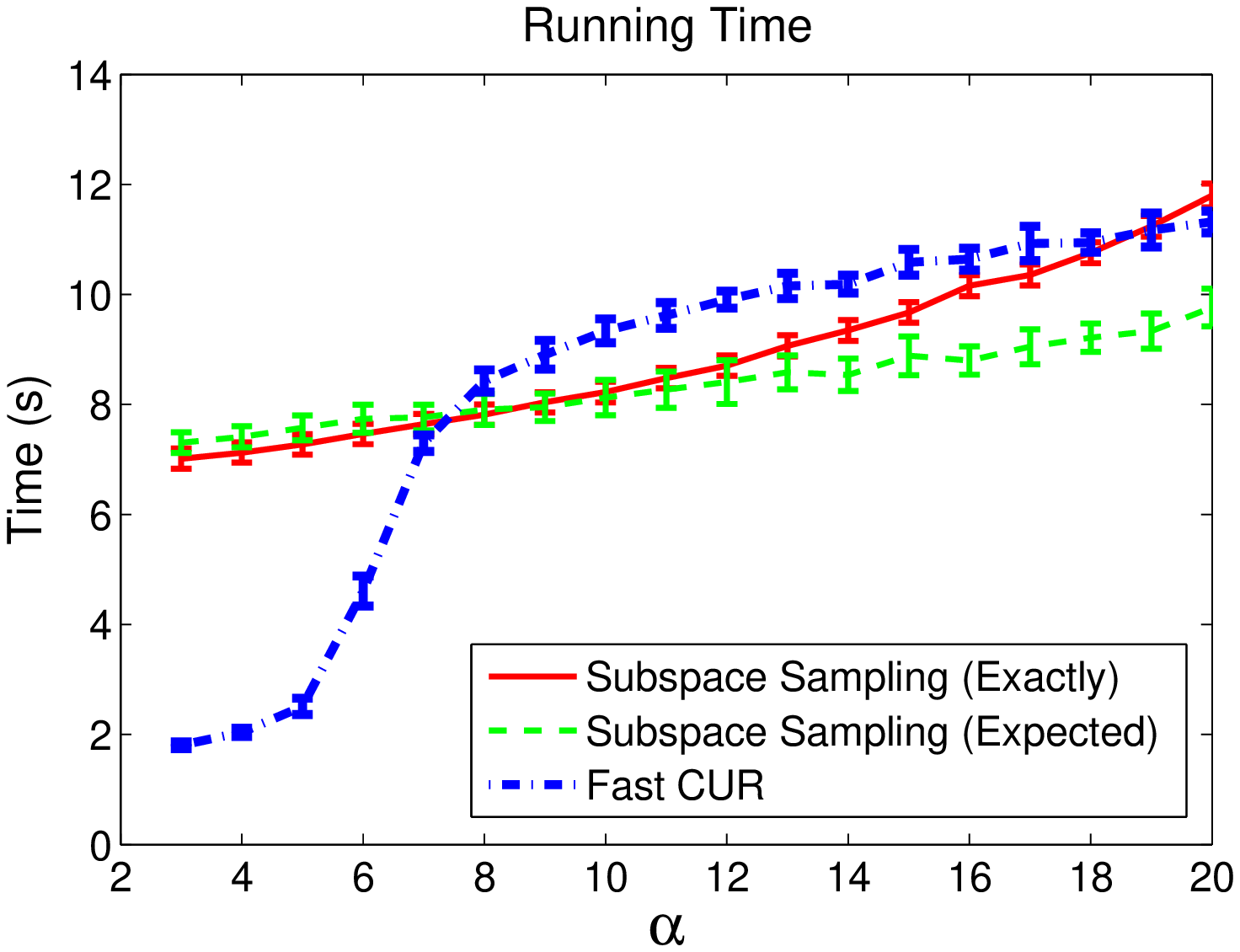}~
\includegraphics[width=48mm, height=40mm]{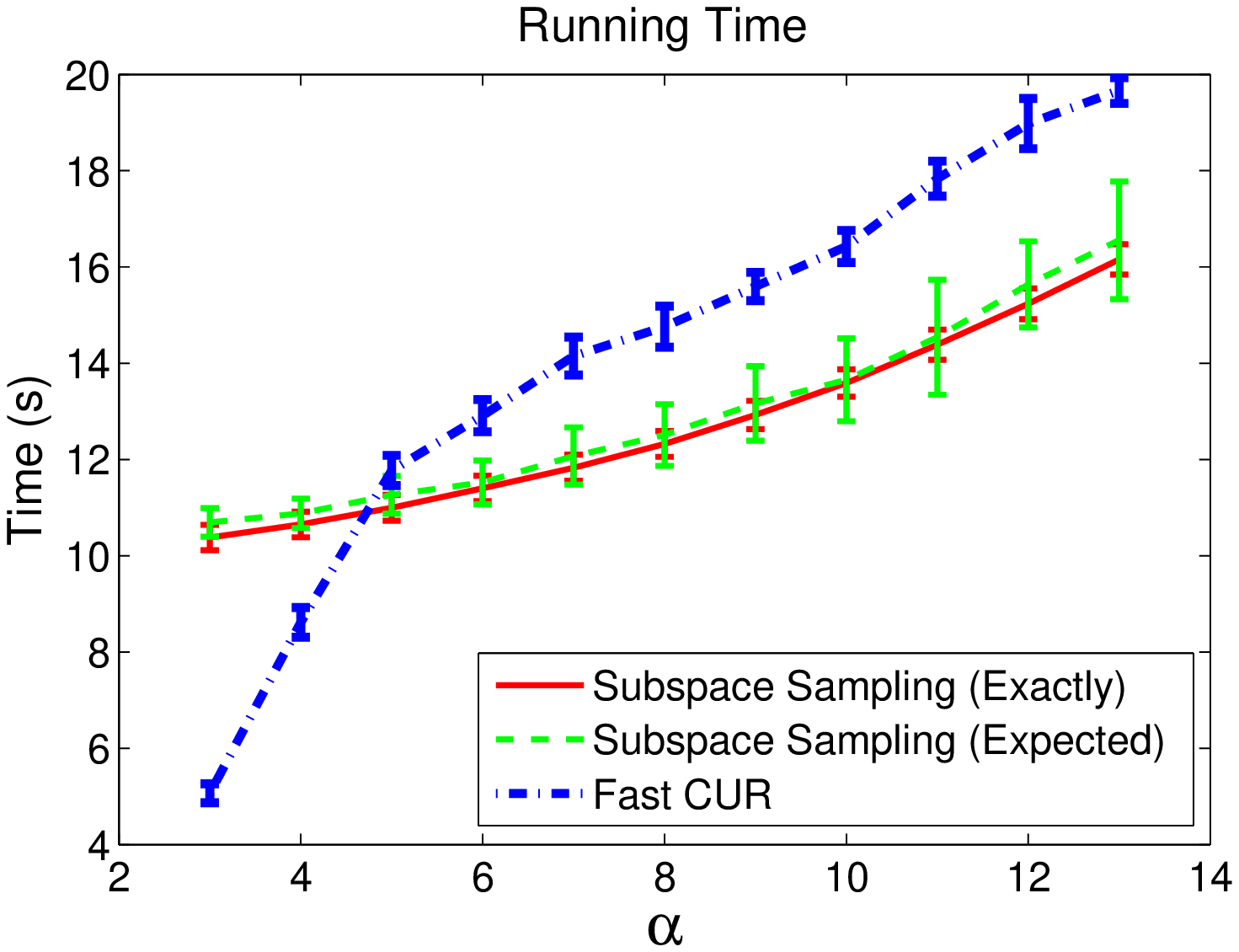} \\
\subfigure[\textsf{$k = 10$, $c=\alpha k$, and $r=\alpha c$.}]{\includegraphics[width=48mm, height=40mm]{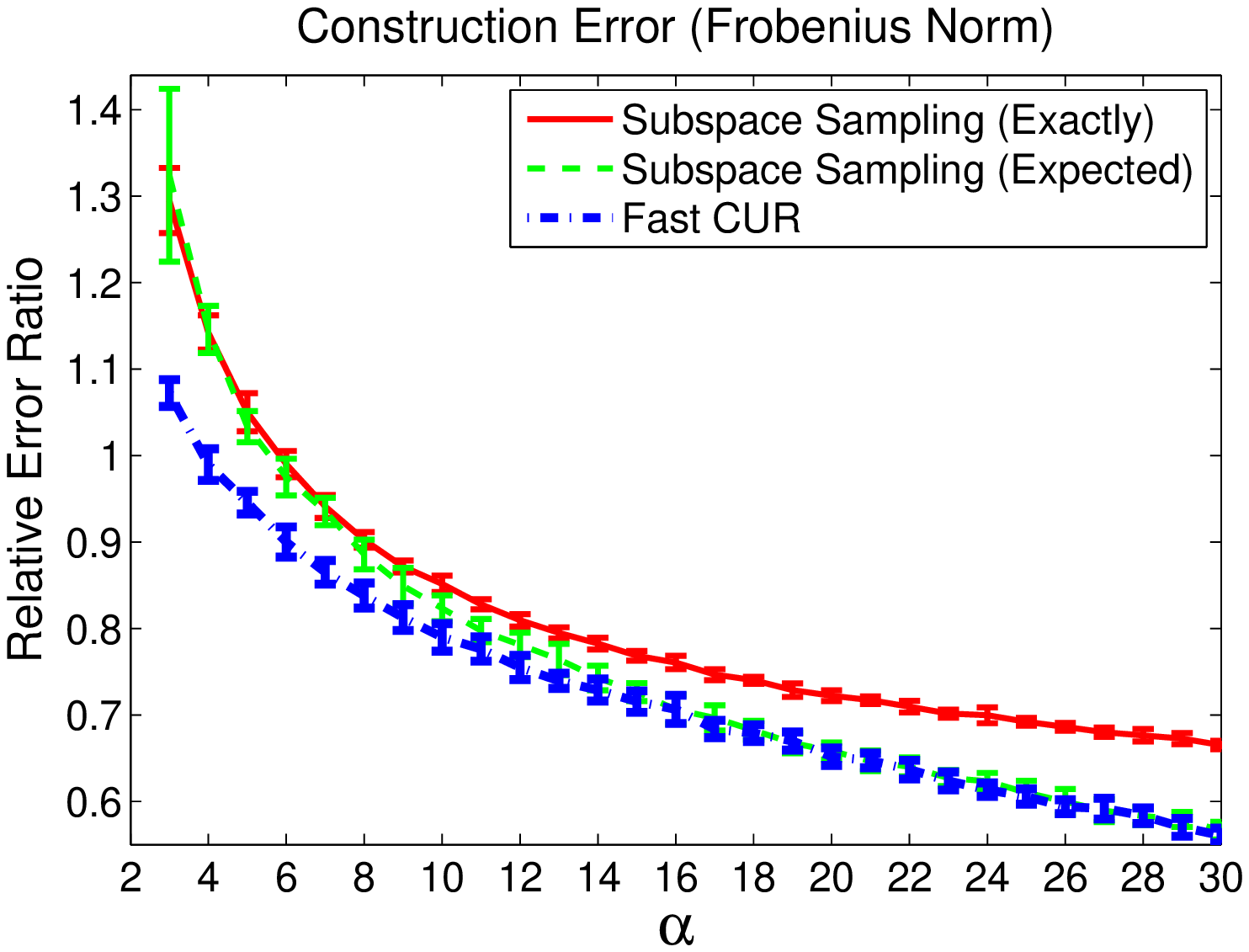}}~
\subfigure[\textsf{$k = 20$, $c=\alpha k$, and $r=\alpha c$.}]{\includegraphics[width=48mm, height=40mm]{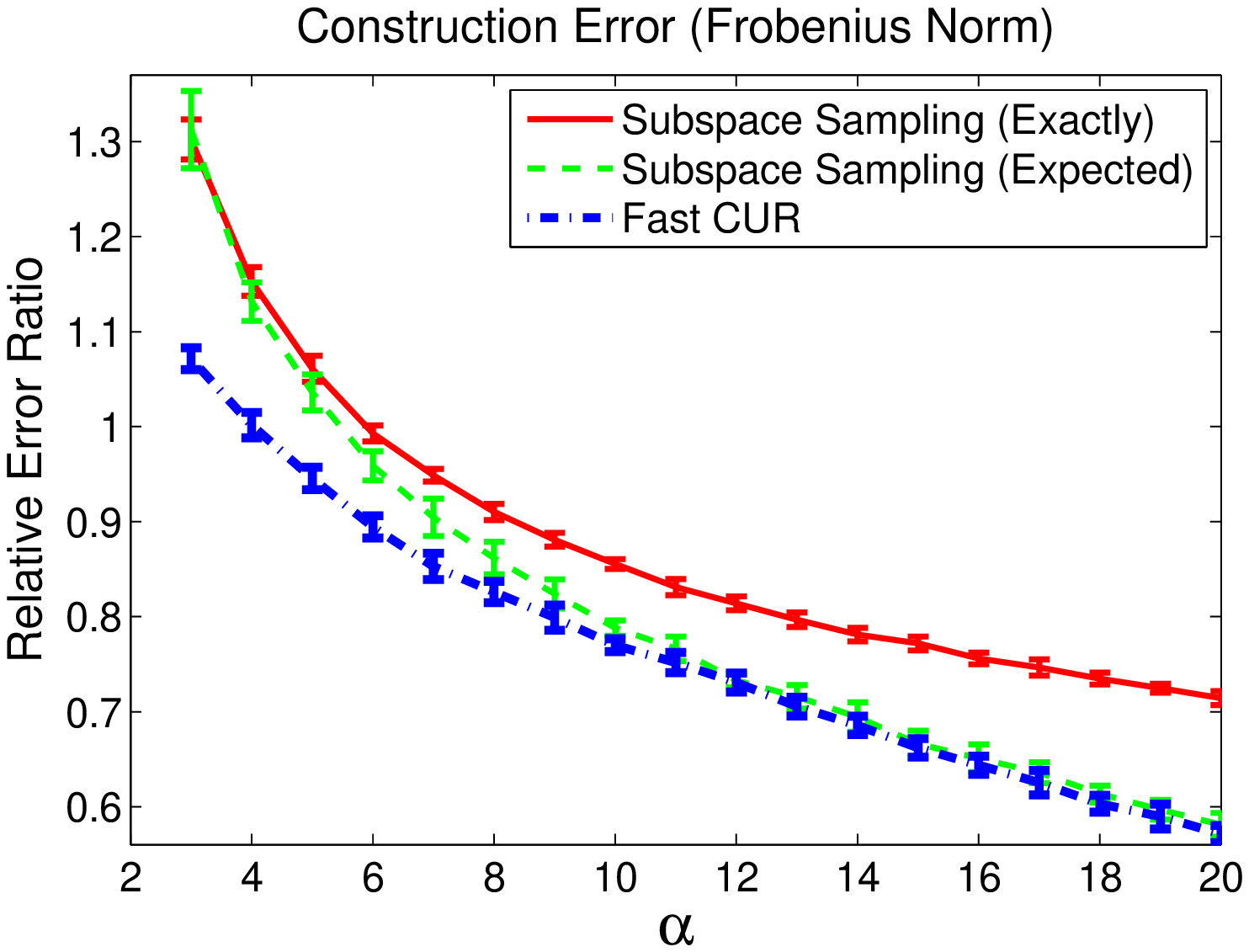}}~
\subfigure[\textsf{$k = 50$, $c=\alpha k$, and $r=\alpha c$.}]{\includegraphics[width=48mm, height=40mm]{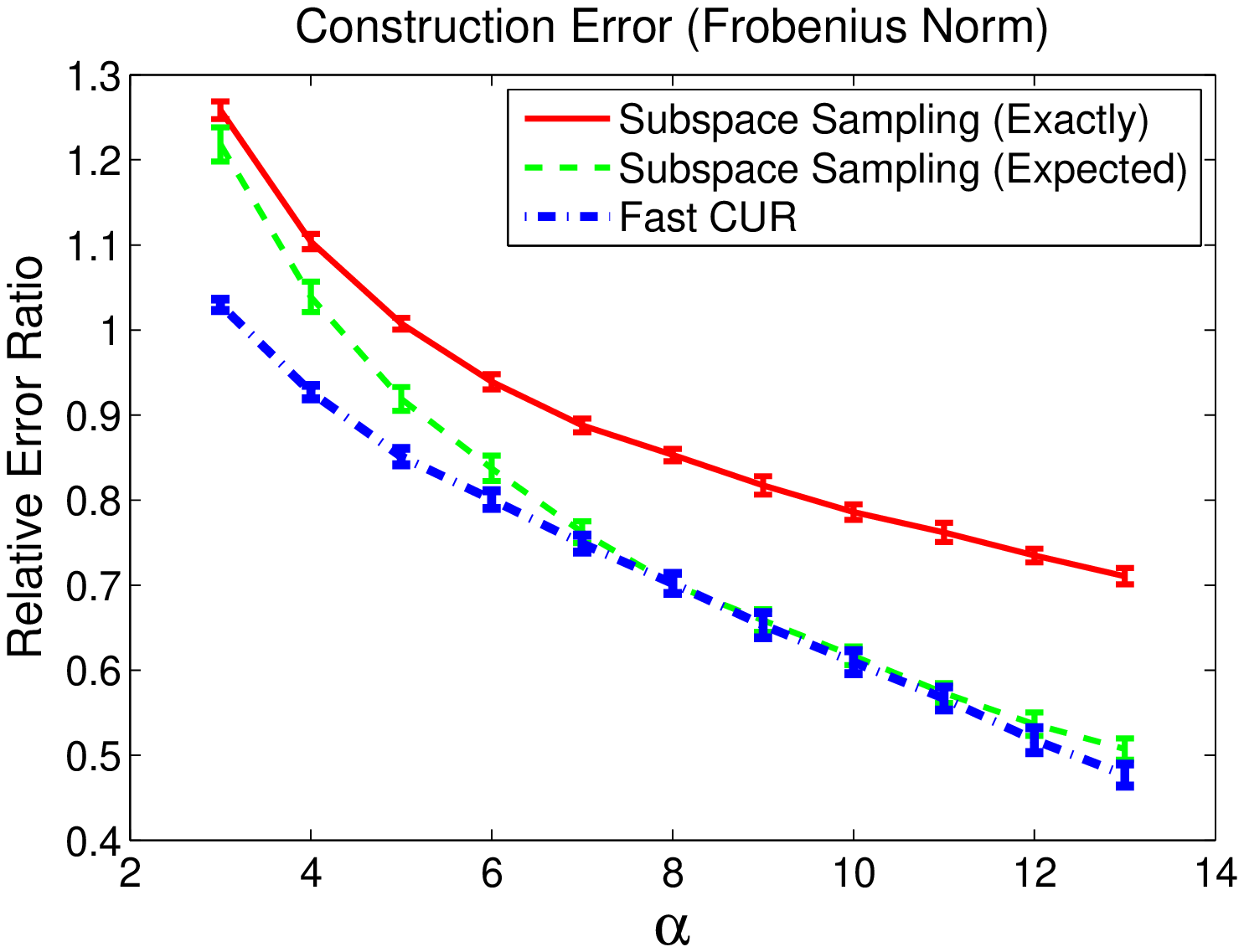}}
\end{center}
   \caption{Empirical results on the Arcene data set.}
\label{fig:arcene}
\end{figure*}
%---------------------------------Figure---------------------------------%

%---------------------------------Figure---------------------------------%
\begin{figure*}
\subfigtopskip = 0pt
\begin{center}
\centering
\includegraphics[width=48mm, height=40mm]{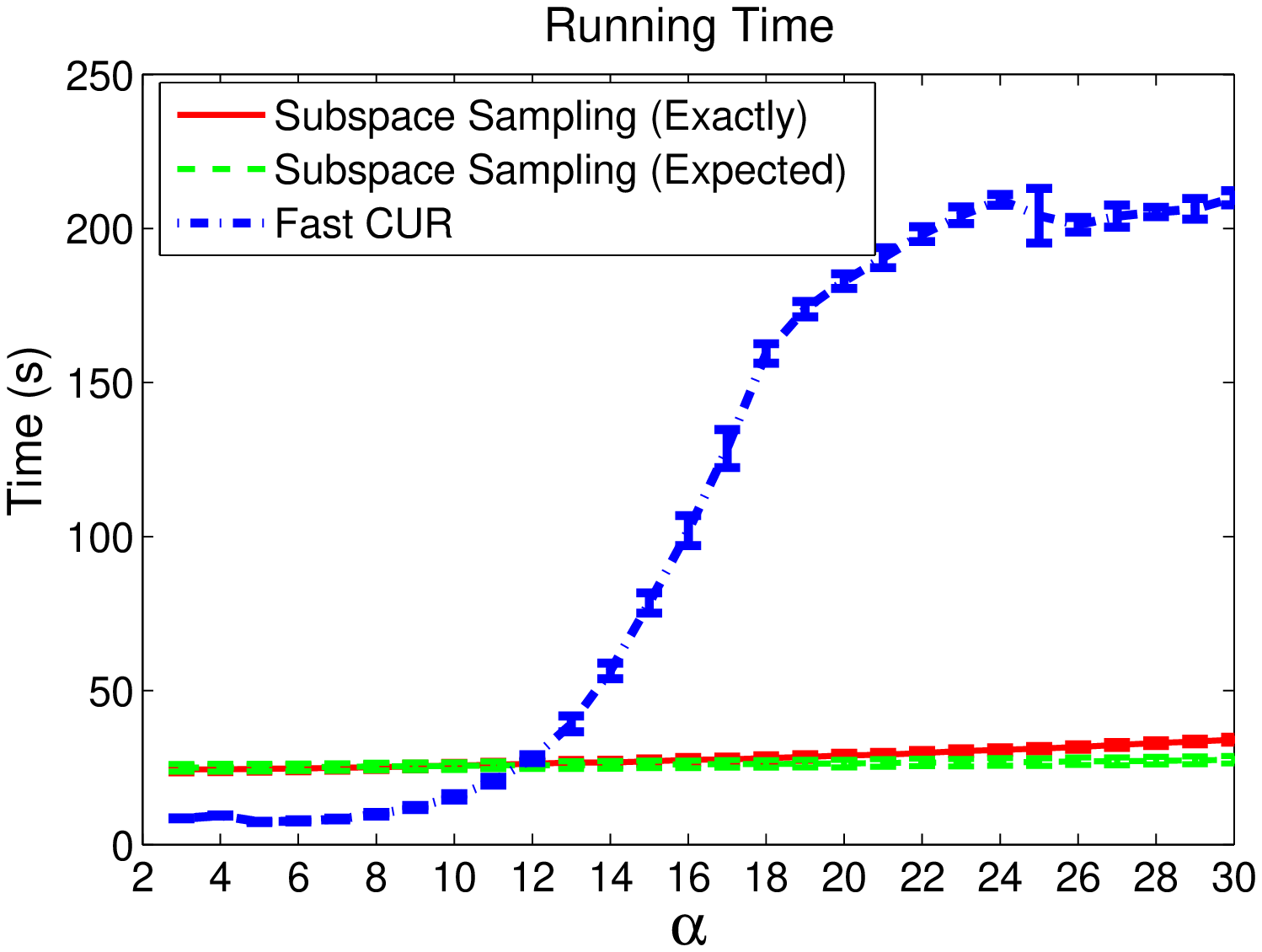}~
\includegraphics[width=48mm, height=40mm]{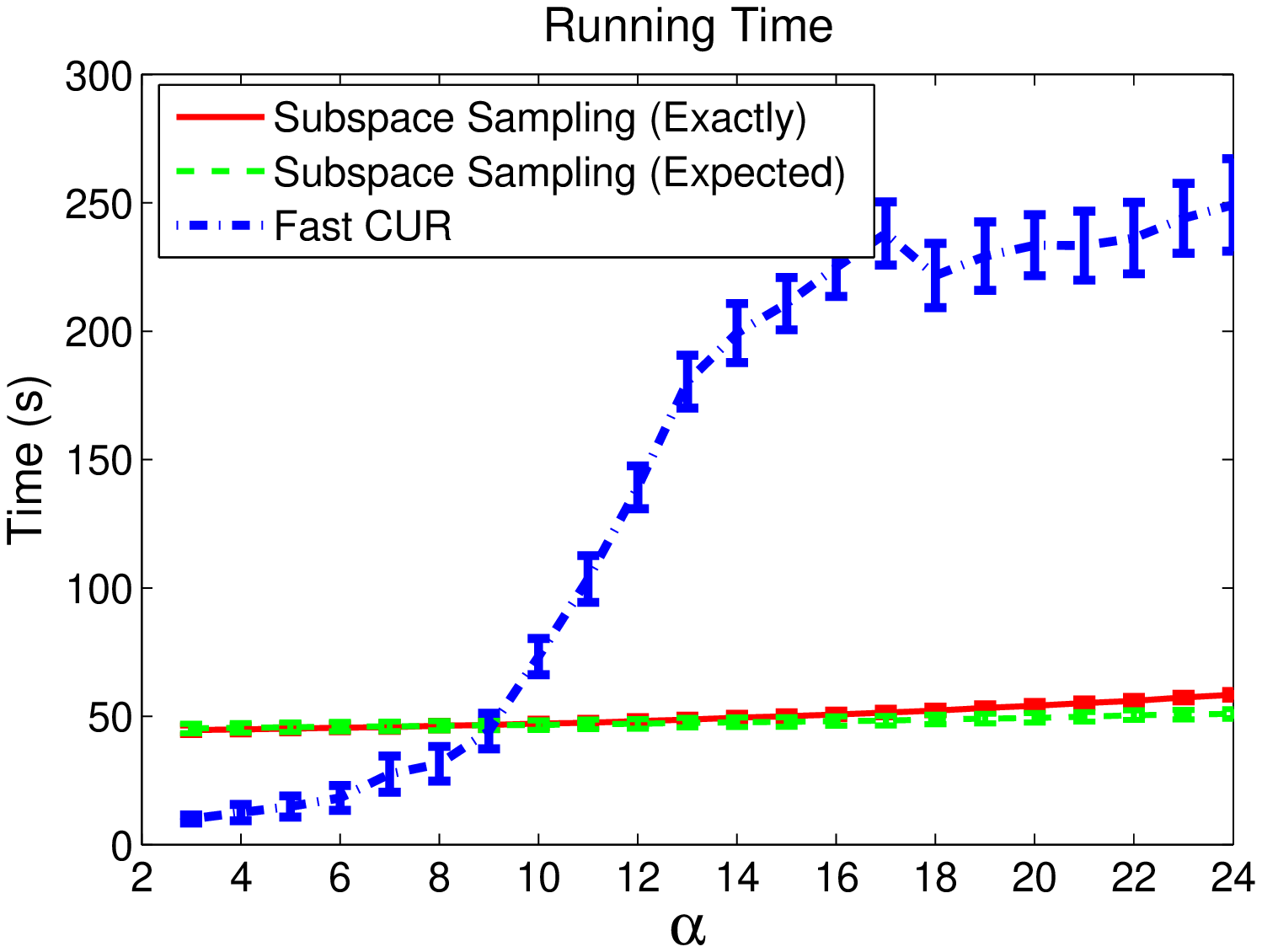}~
\includegraphics[width=48mm, height=40mm]{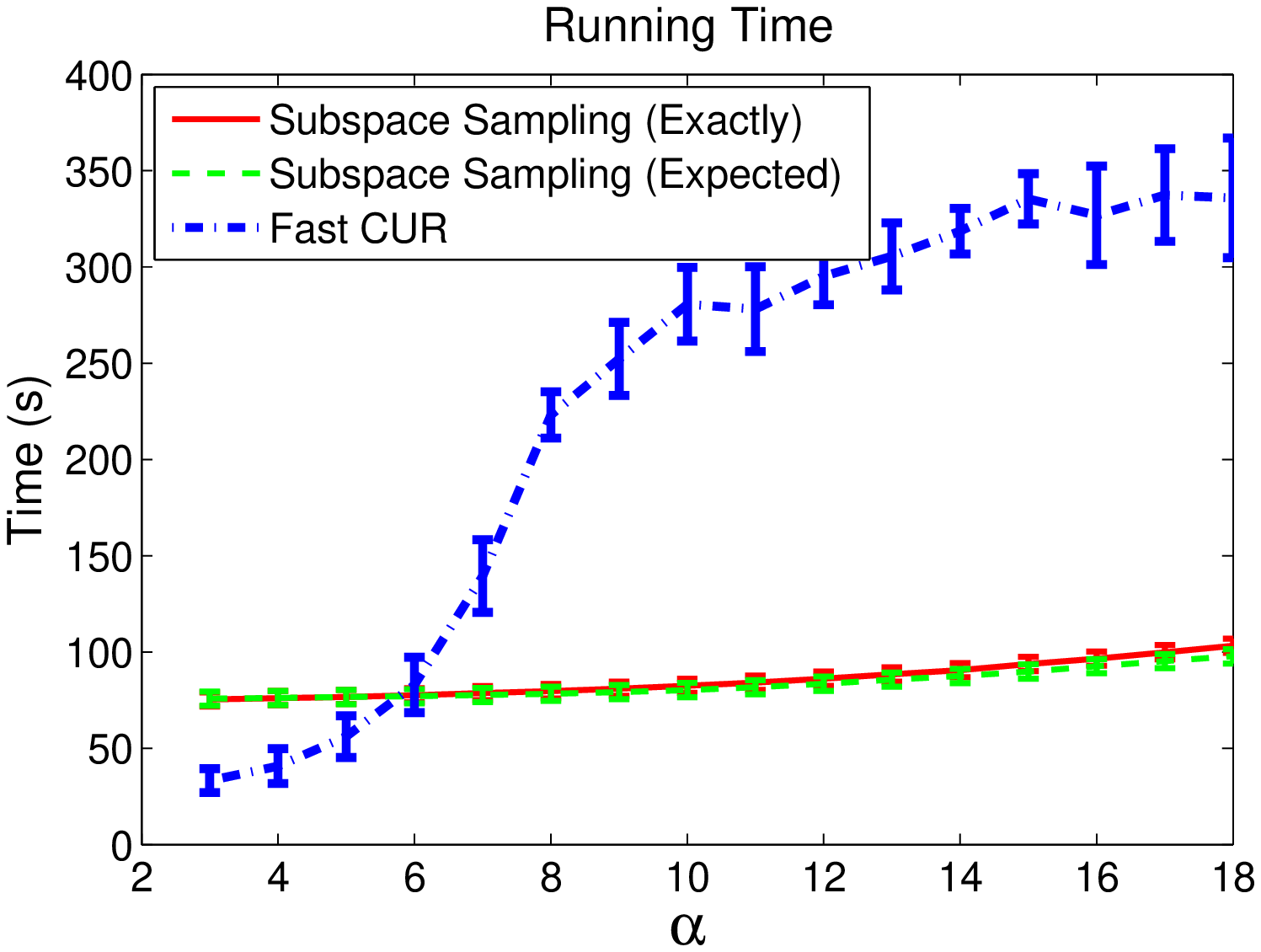} \\
\subfigure[\textsf{$k = 10$, $c=\alpha k$, and $r=\alpha c$.}]{\includegraphics[width=48mm, height=45mm]{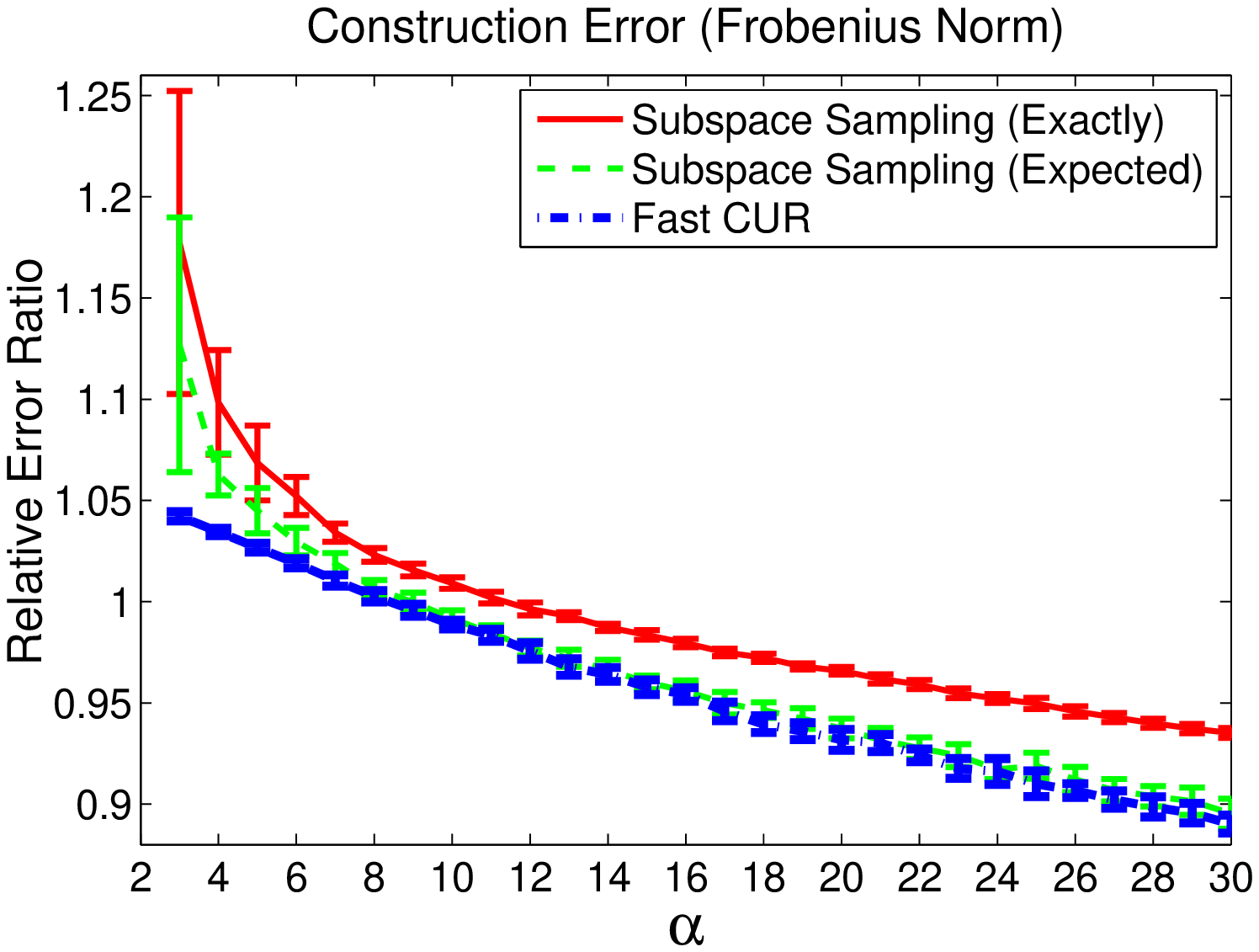}}~
\subfigure[\textsf{$k = 20$, $c=\alpha k$, and $r=\alpha c$.}]{\includegraphics[width=48mm, height=45mm]{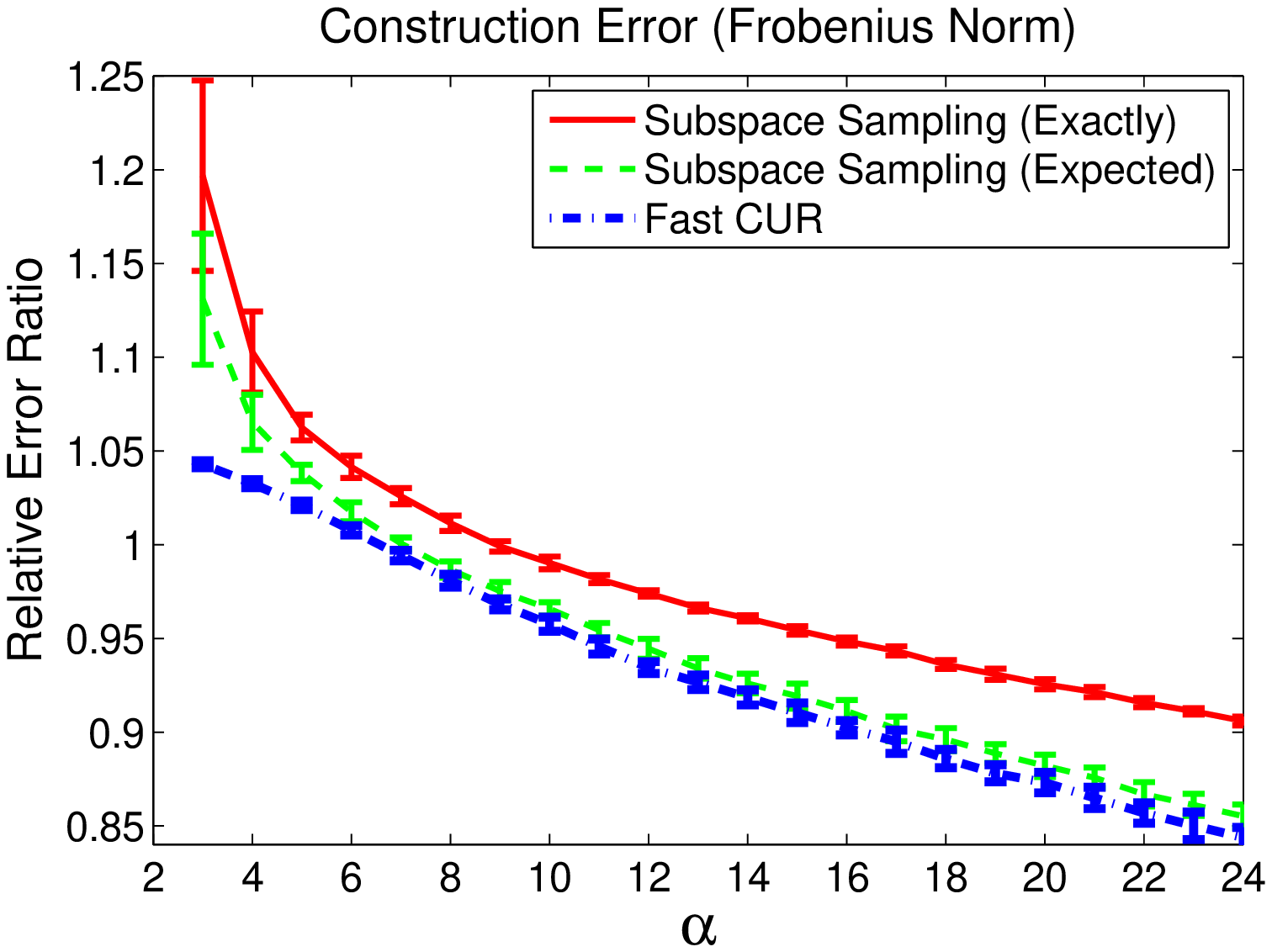}}~
\subfigure[\textsf{$k = 50$, $c=\alpha k$, and $r=\alpha c$.}]{\includegraphics[width=48mm, height=45mm]{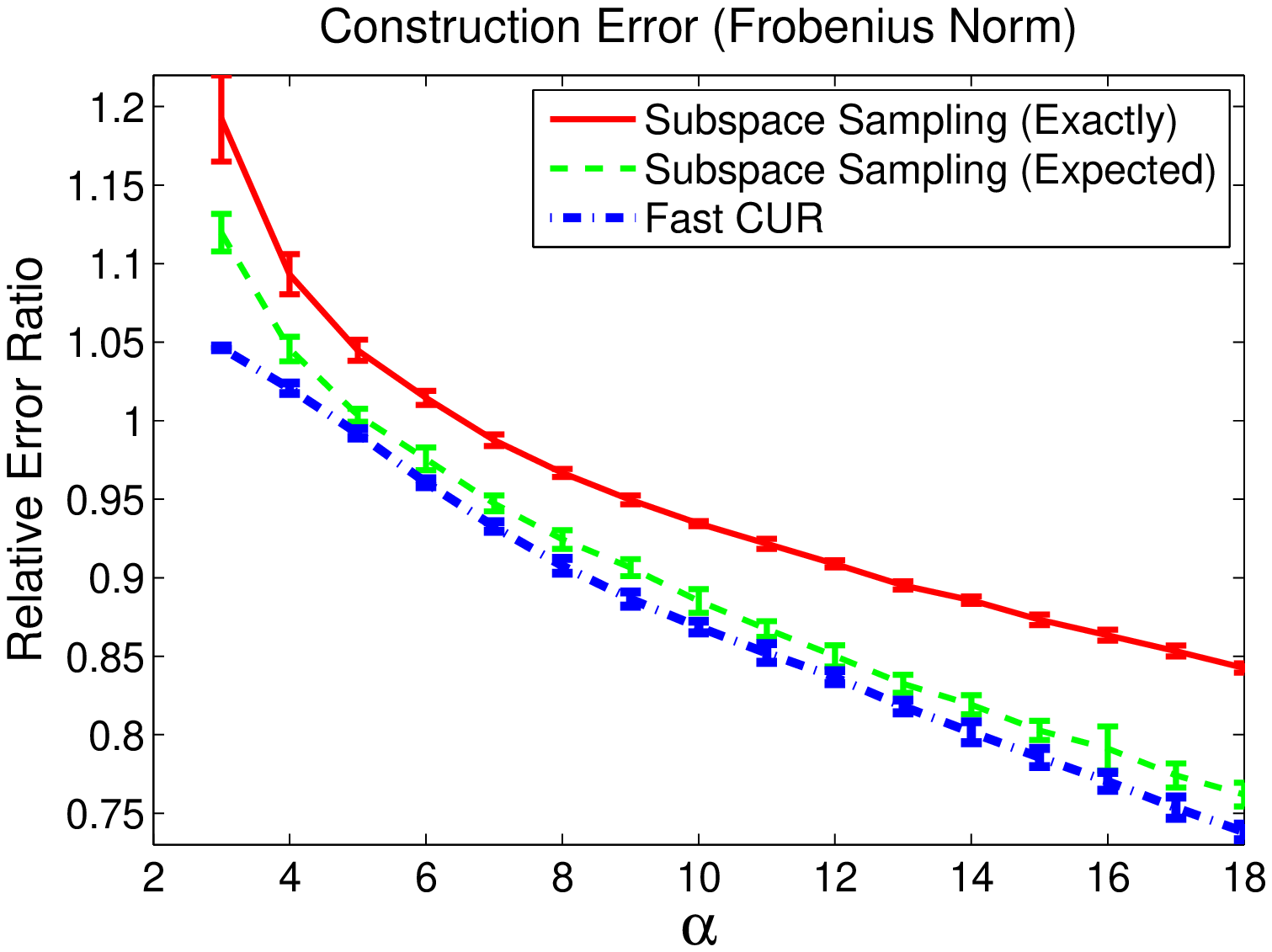}}
\end{center}
   \caption{Empirical results on the Dexter data set.}
\label{fig:dexter}
\end{figure*}
%---------------------------------Figure---------------------------------%

%---------------------------------Figure---------------------------------%
\begin{figure*}
\subfigtopskip = 0pt
\begin{center}
\centering
\includegraphics[width=48mm, height=40mm]{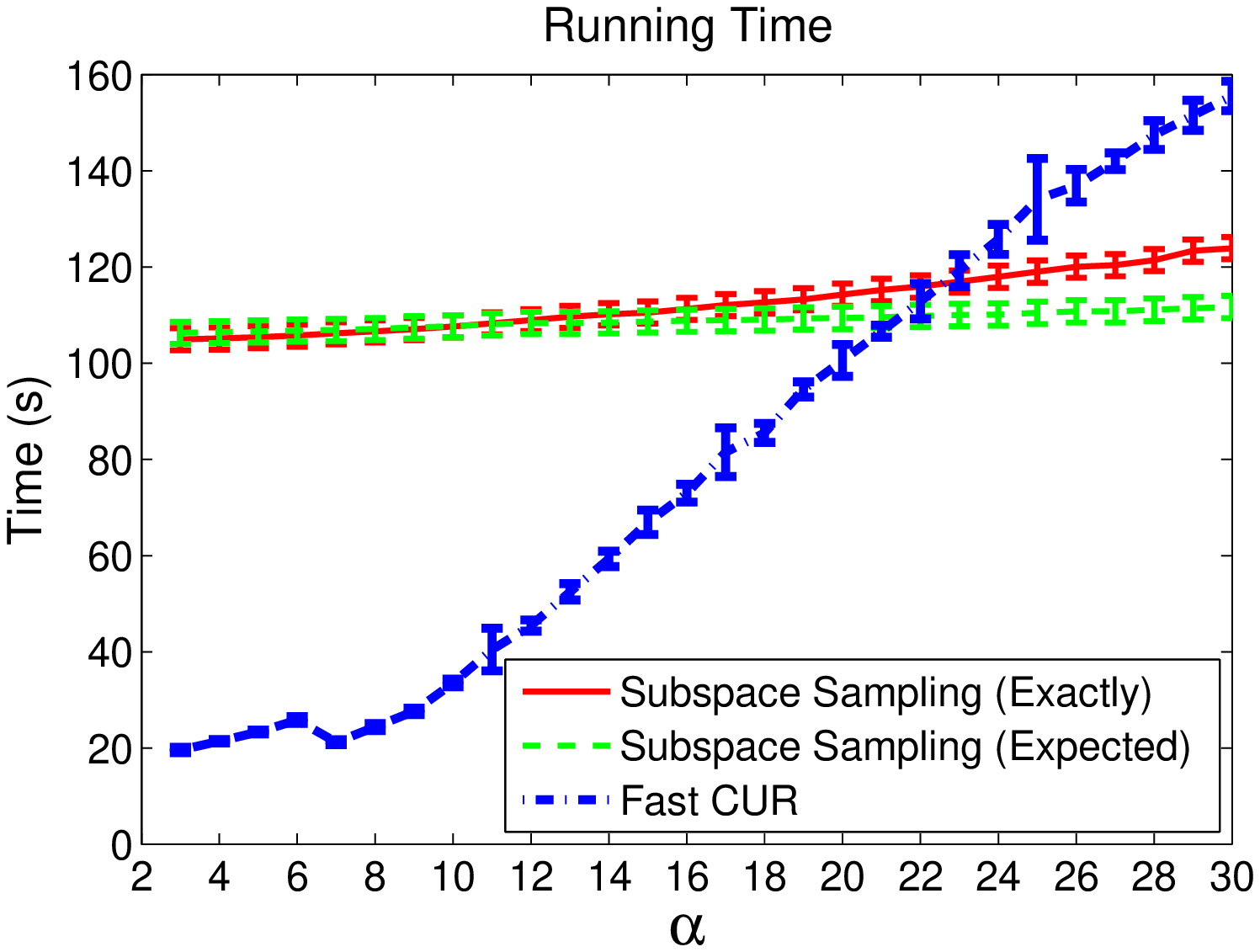}~
\includegraphics[width=48mm, height=40mm]{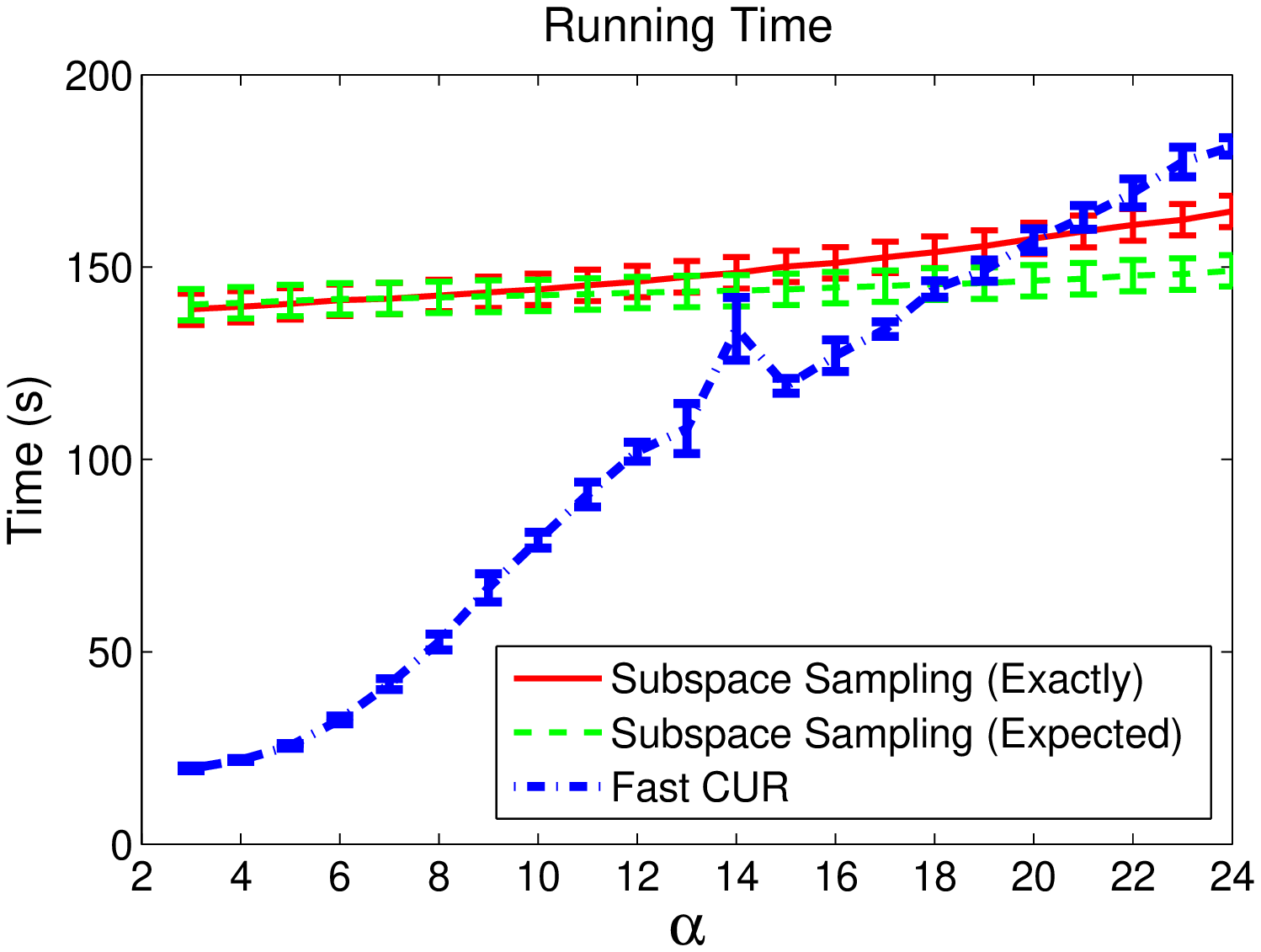}~
\includegraphics[width=48mm, height=40mm]{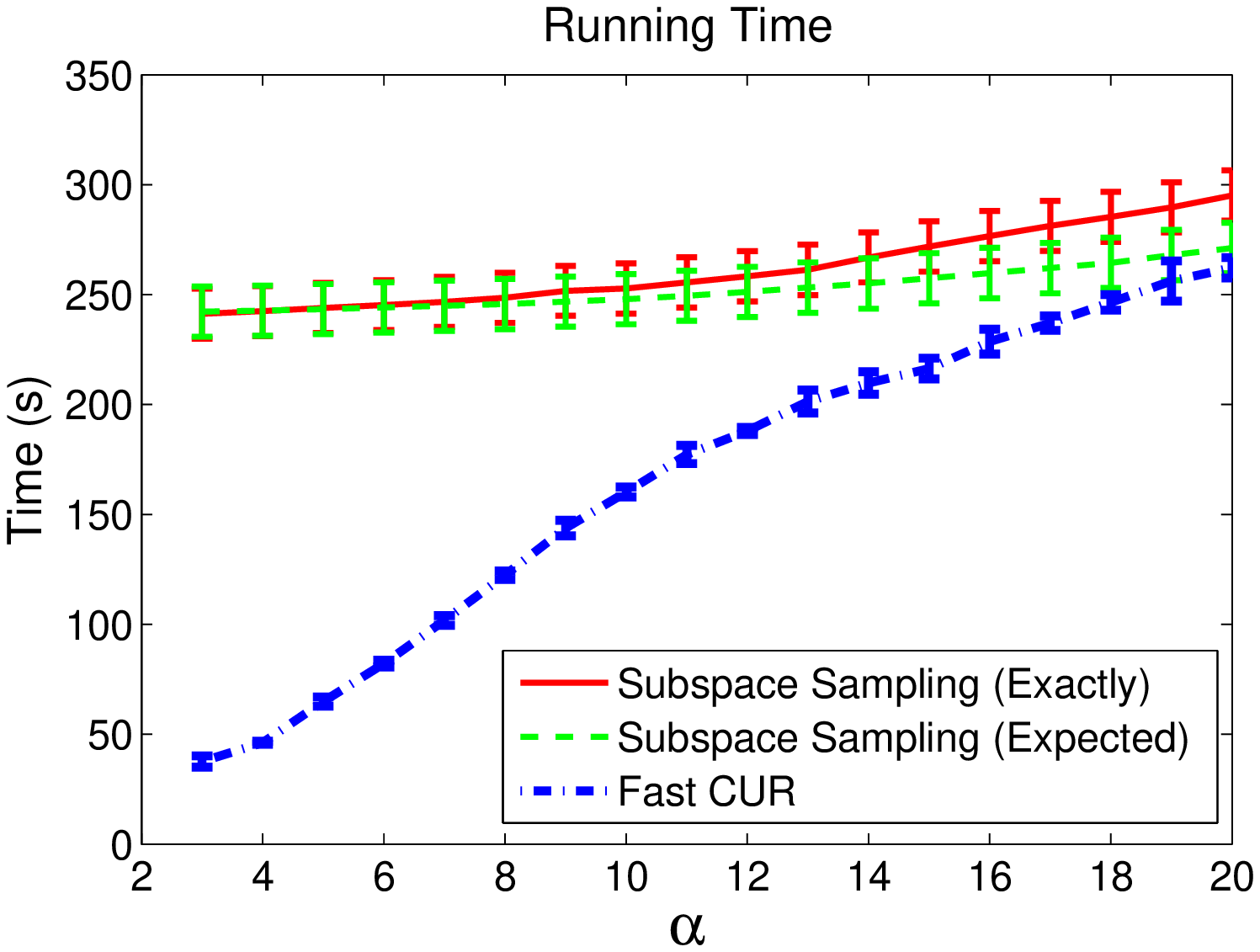} \\
\subfigure[\textsf{$k = 10$, $c=\alpha k$, and $r=\alpha c$.}]{\includegraphics[width=48mm, height=40mm]{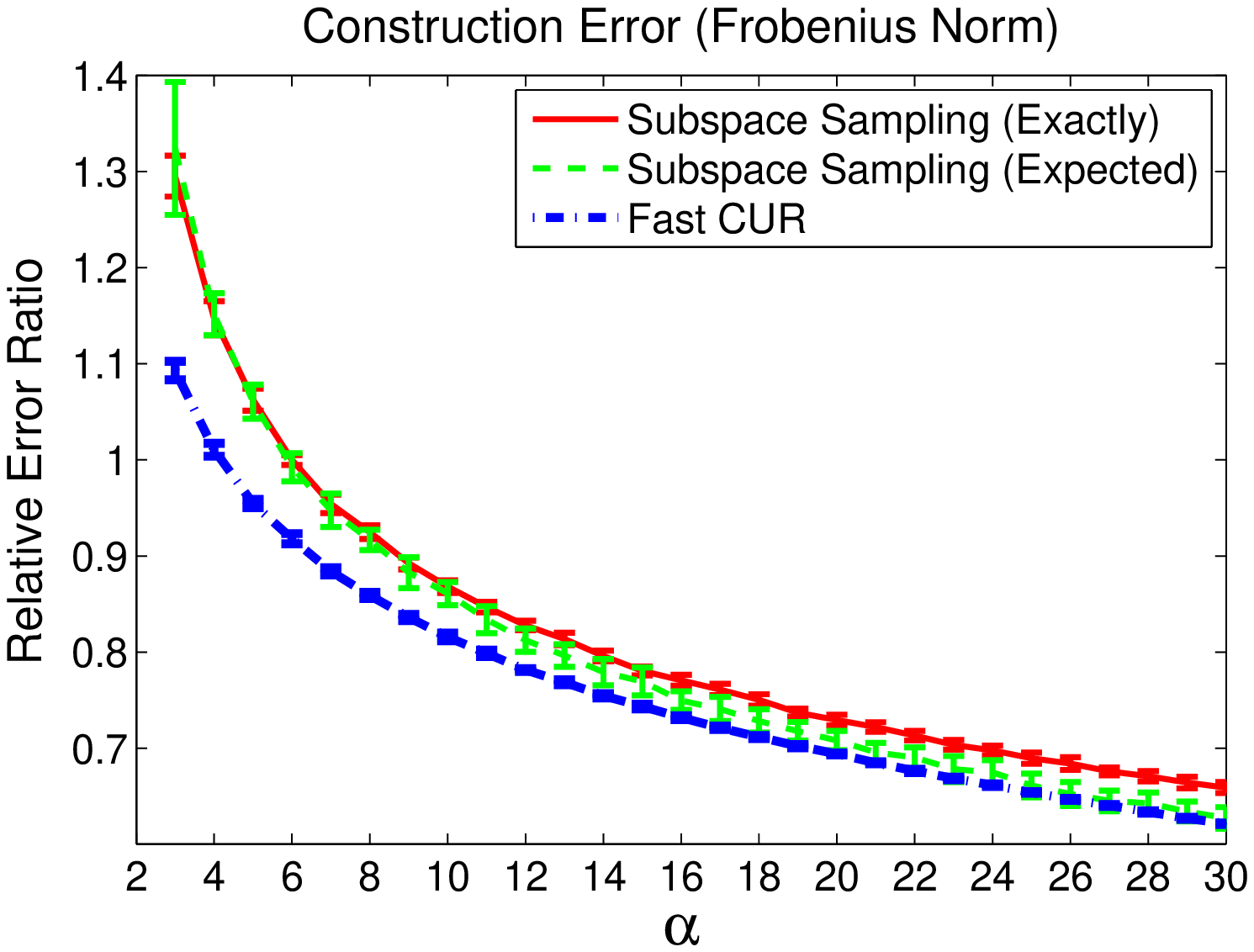}}~
\subfigure[\textsf{$k = 20$, $c=\alpha k$, and $r=\alpha c$.}]{\includegraphics[width=48mm, height=40mm]{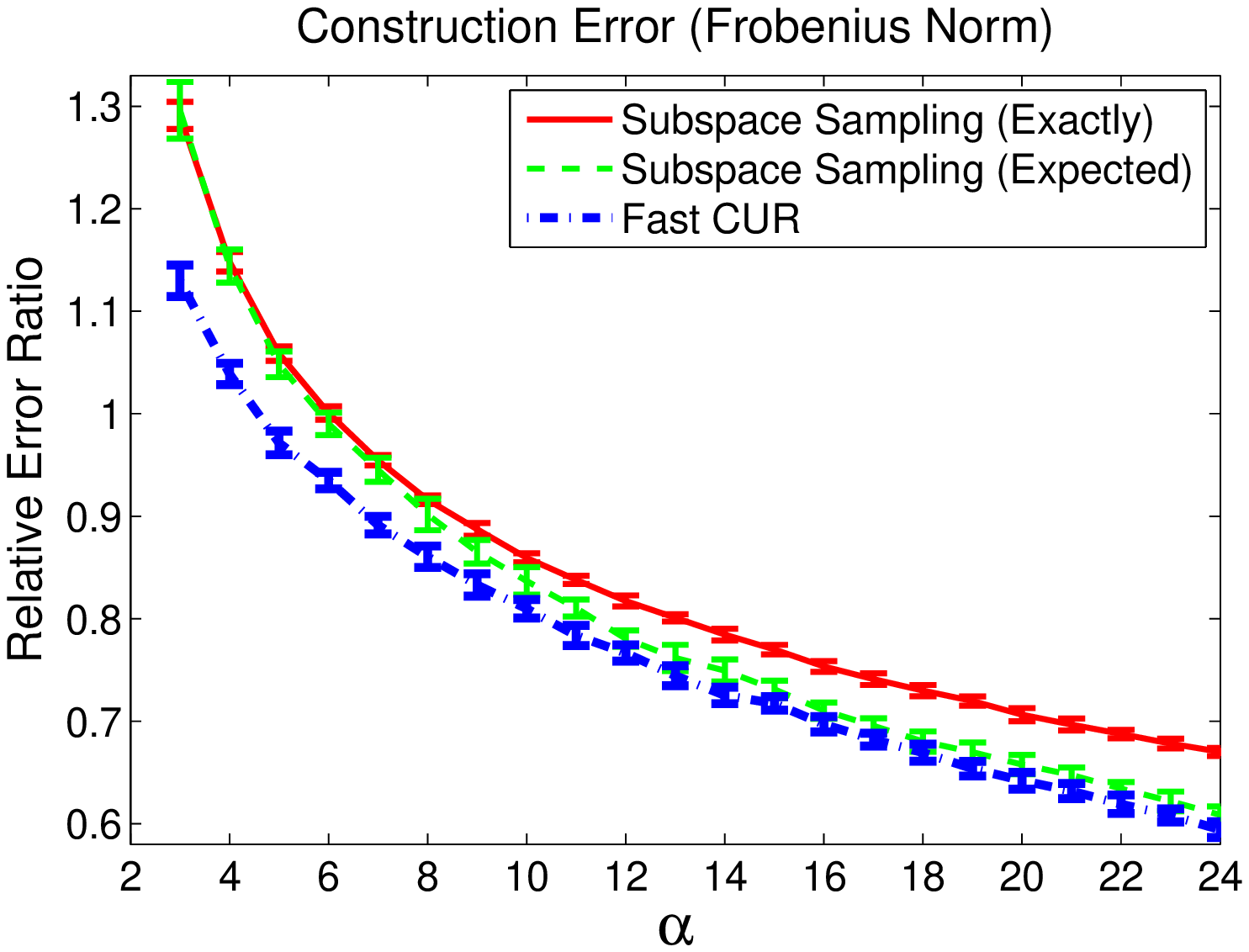}}~
\subfigure[\textsf{$k = 50$, $c=\alpha k$, and $r=\alpha c$.}]{\includegraphics[width=48mm, height=40mm]{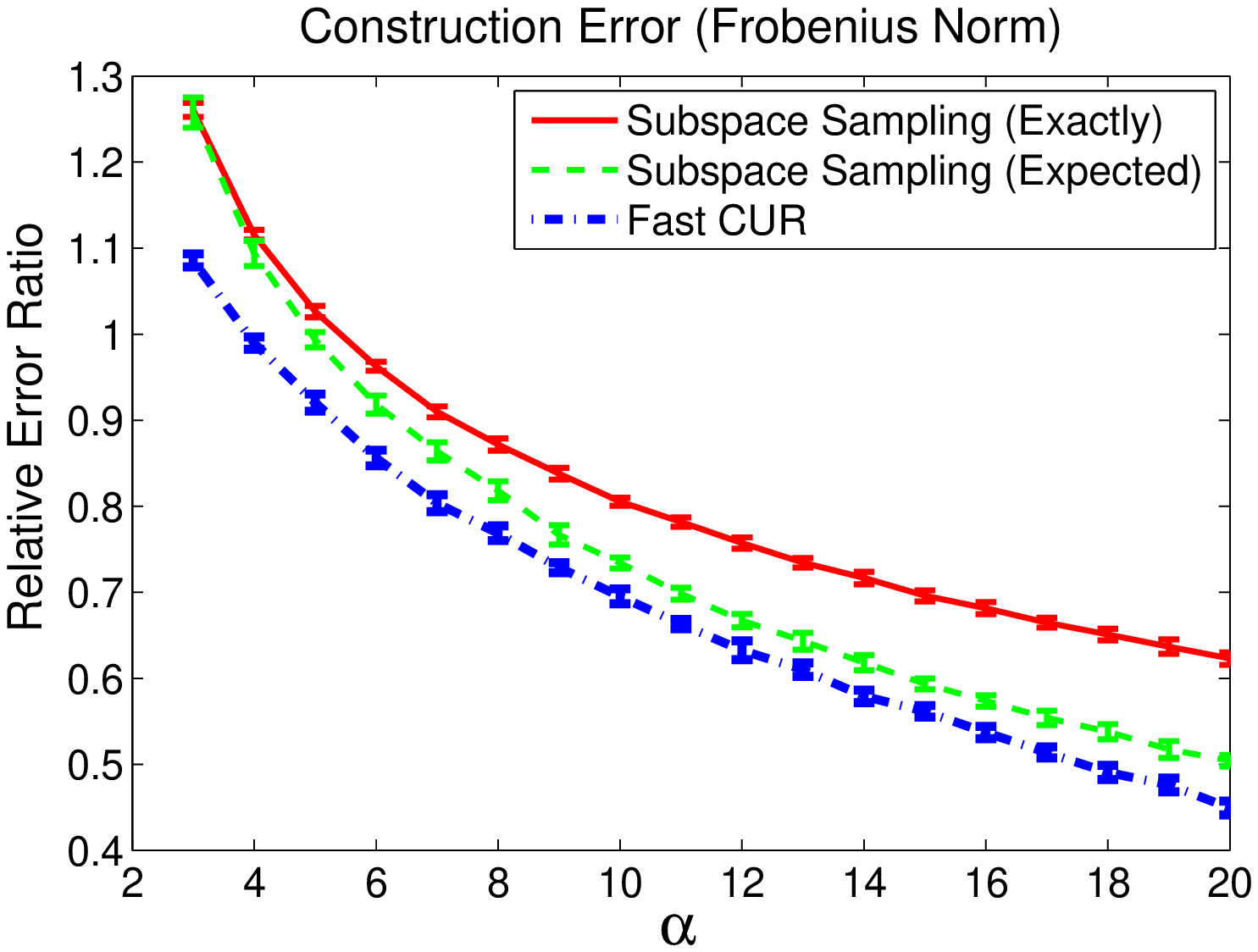}}
\end{center}
   \caption{Empirical results on the HMAX features of the PicasaWeb image data set.}
\label{fig:hmax}
\end{figure*}
%---------------------------------Figure---------------------------------%

%---------------------------------Figure---------------------------------%
\begin{figure*}
\subfigtopskip = 0pt
\begin{center}
\centering
\includegraphics[width=48mm, height=40mm]{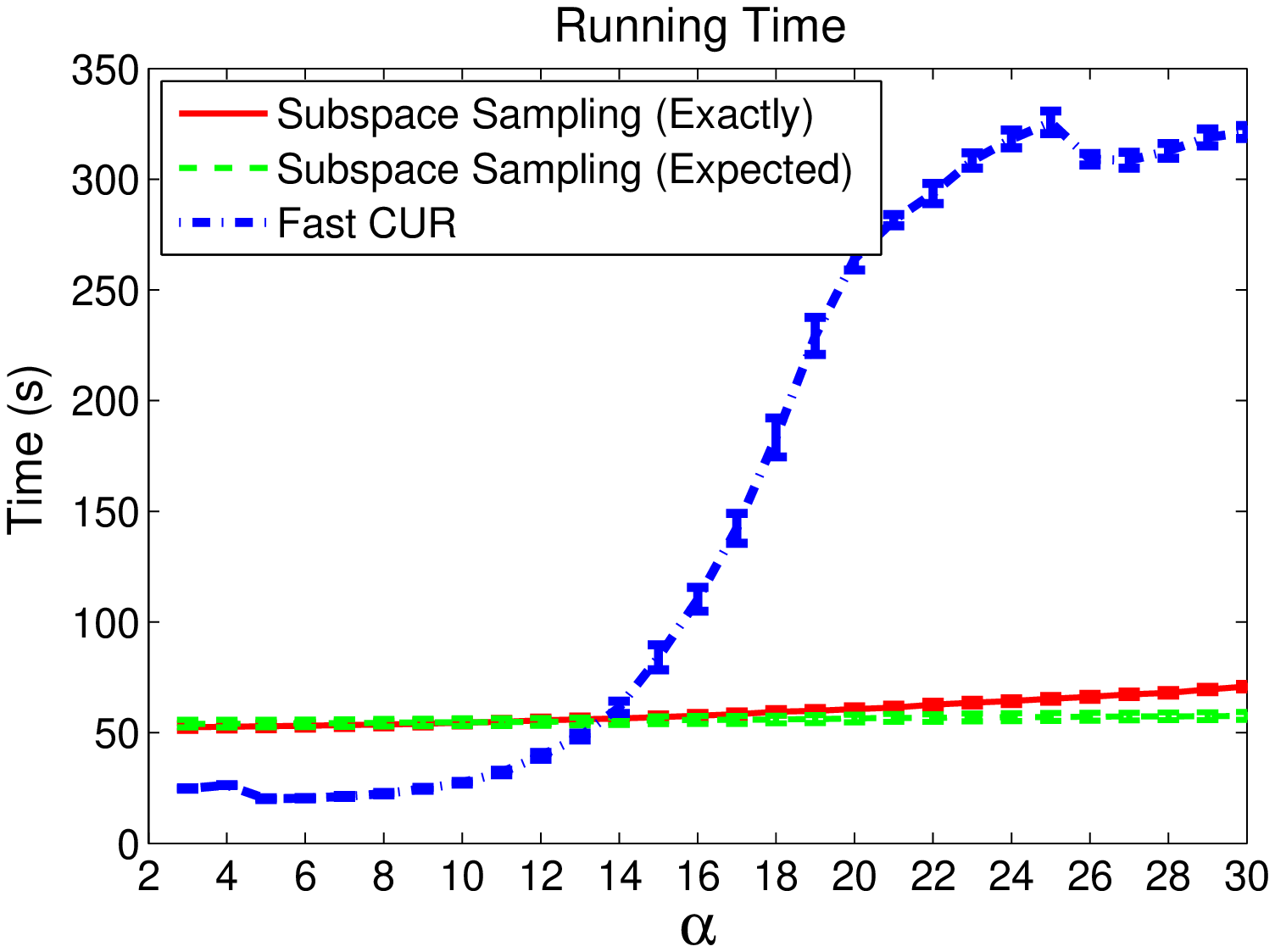}~
\includegraphics[width=48mm, height=40mm]{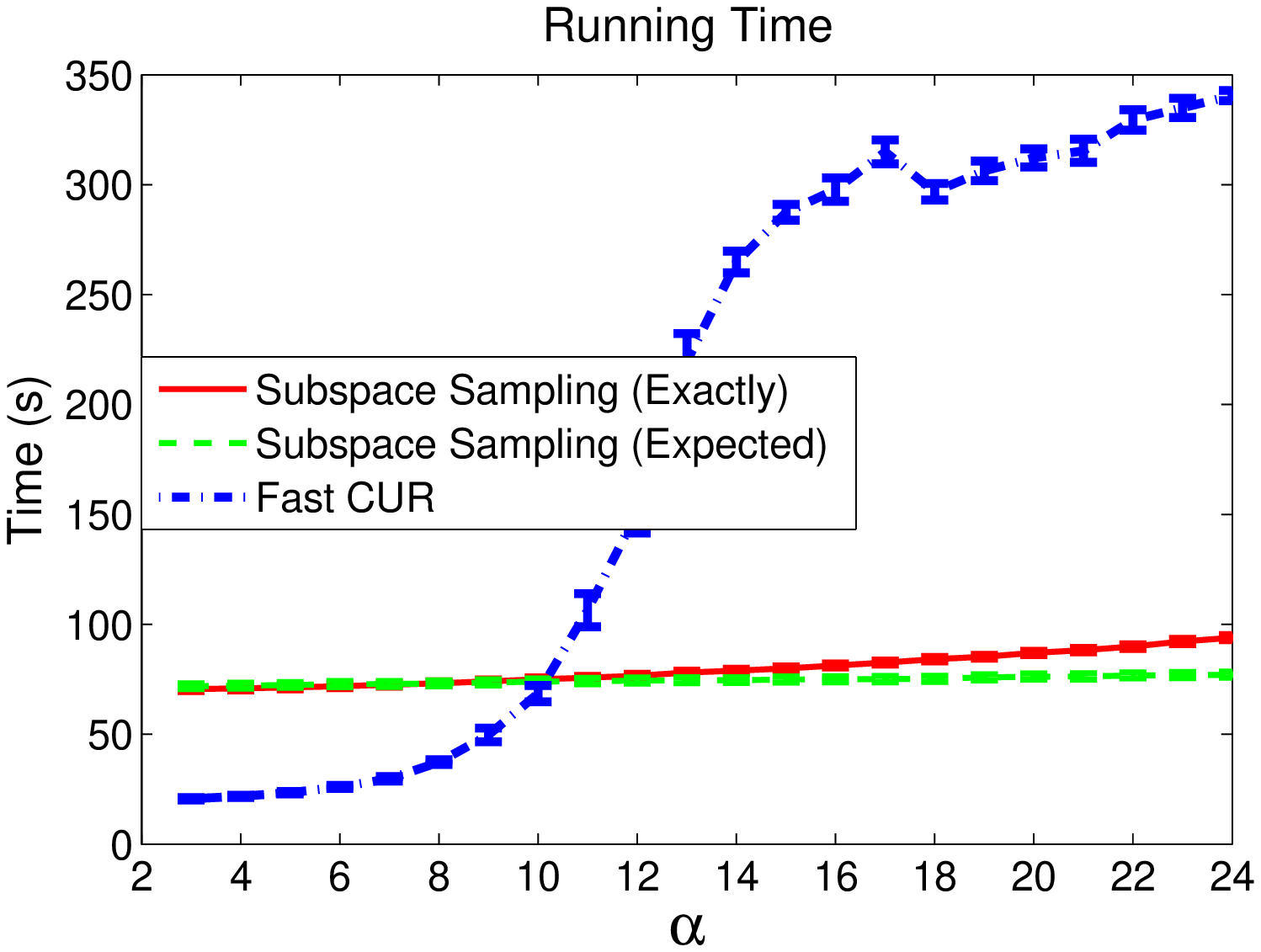}~
\includegraphics[width=48mm, height=40mm]{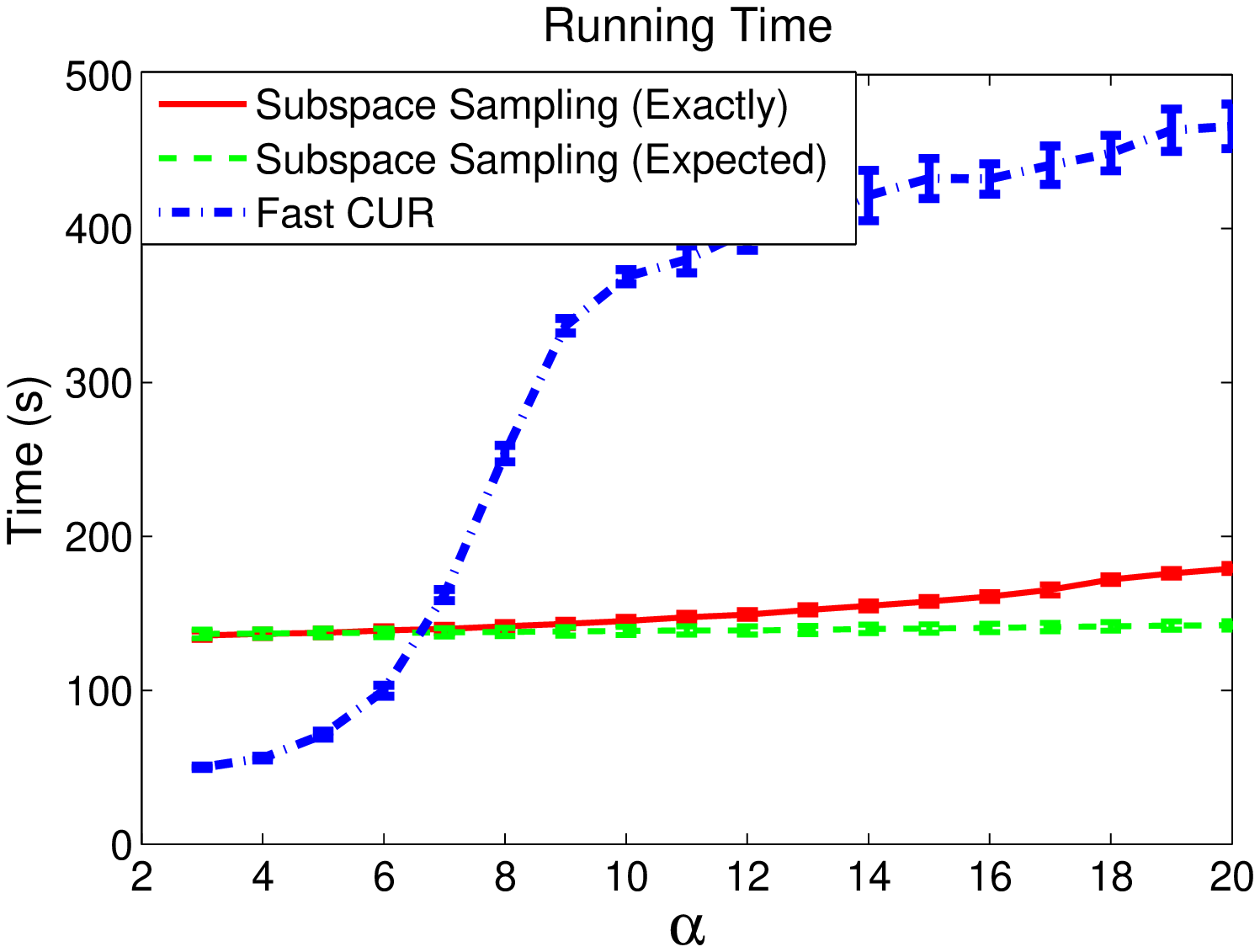} \\
\subfigure[\textsf{$k = 10$, $c=\alpha k$, and $r=\alpha c$.}]{\includegraphics[width=48mm, height=40mm]{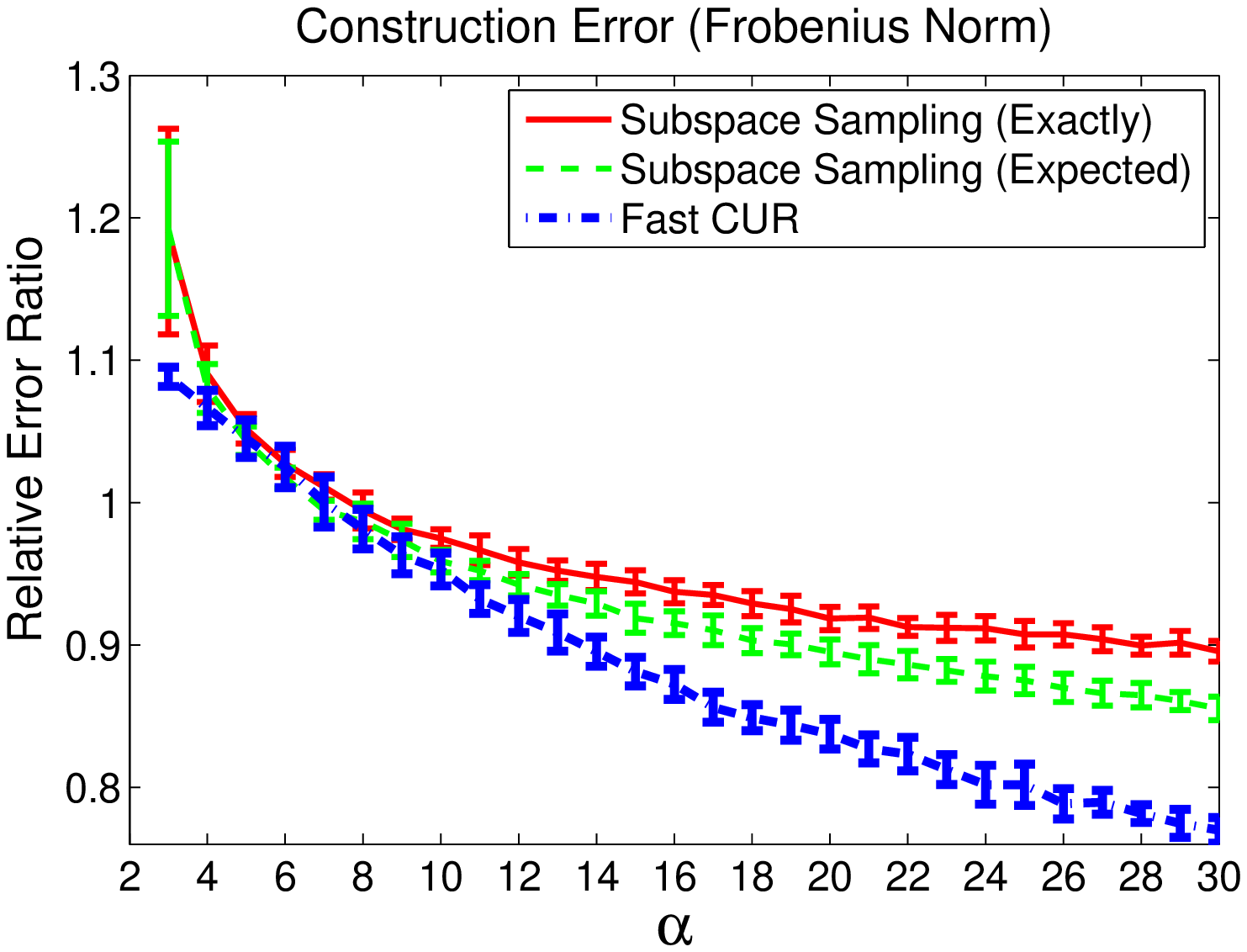}}~
\subfigure[\textsf{$k = 20$, $c=\alpha k$, and $r=\alpha c$.}]{\includegraphics[width=48mm, height=40mm]{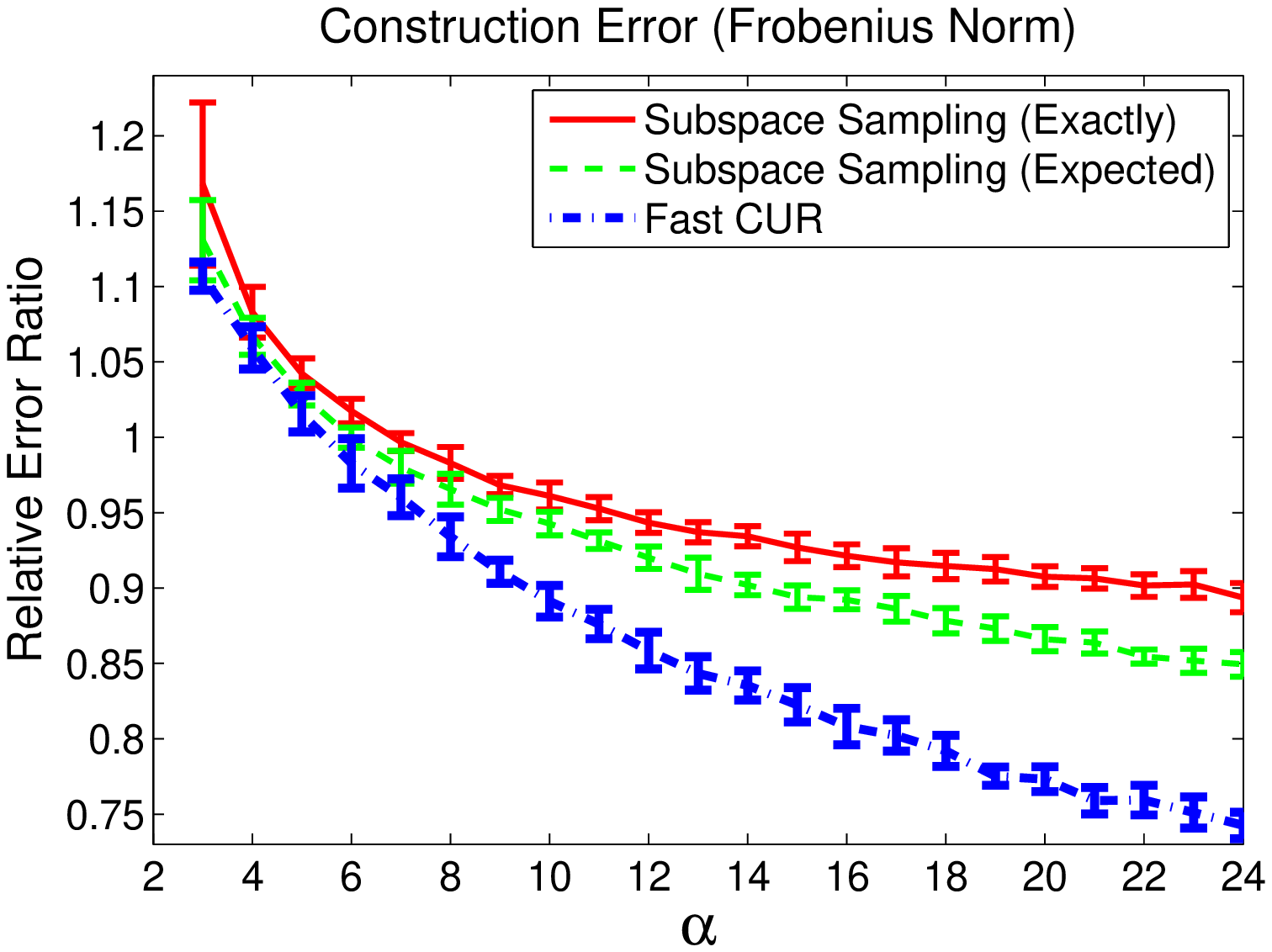}}~
\subfigure[\textsf{$k = 50$, $c=\alpha k$, and $r=\alpha c$.}]{\includegraphics[width=48mm, height=40mm]{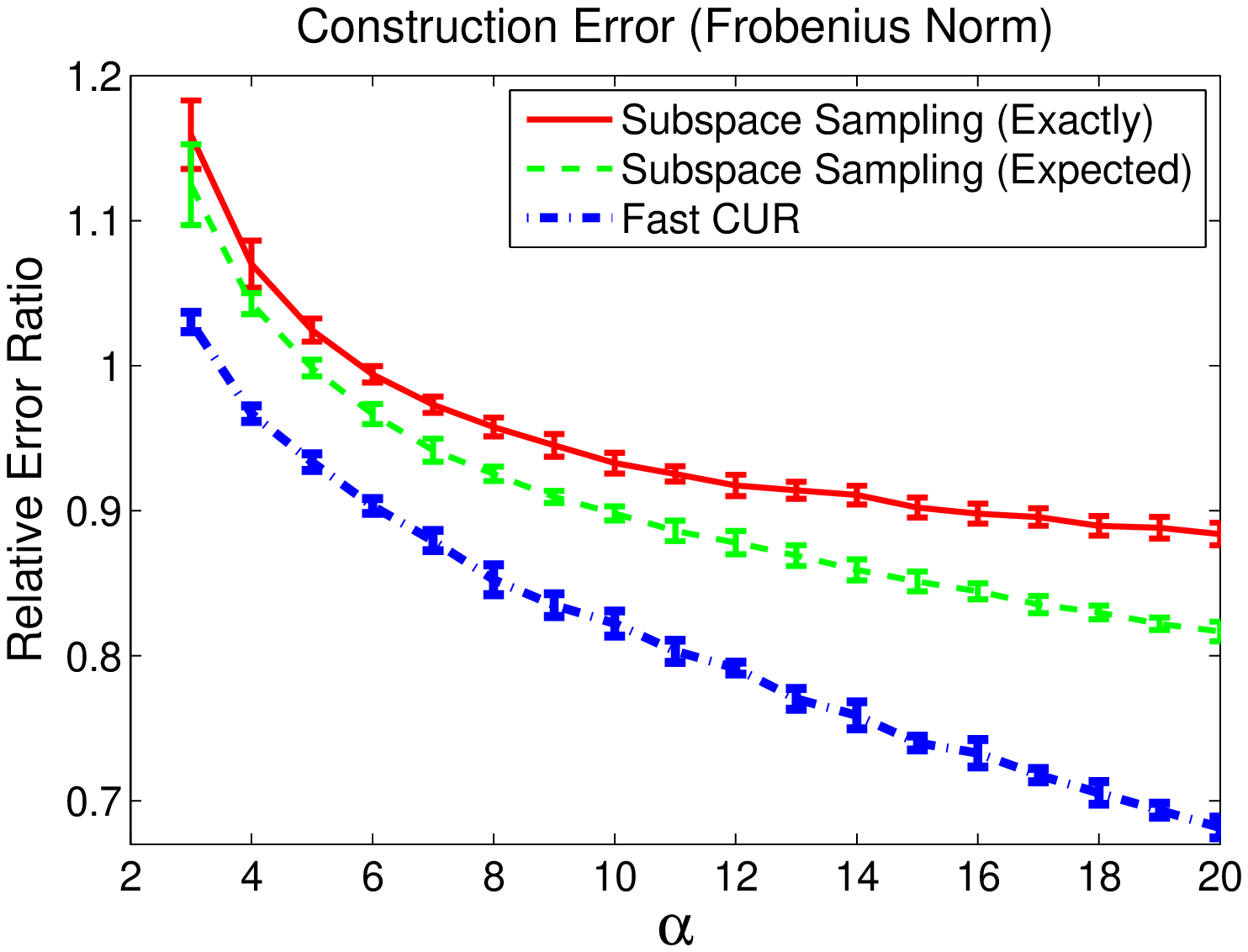}}
\end{center}
   \caption{Empirical results on the SIFT features of the PicasaWeb image data set.}
\label{fig:sift}
\end{figure*}
%---------------------------------Figure---------------------------------%

\subsection{Result Analysis}

The  results show that the fast \CUR algorithm
has much lower relative-error ratio than the subspace sampling algorithm.
The experimental results well match our theoretical analyses in Section~\ref{sec:fast_cur}.
As for the running time, the fast \CUR algorithm is more efficient when $c$ and $r$ are small.
When $c$ and $r$ become large, the fast \CUR algorithm becomes less efficient.
This is because the time complexity of the fast \CUR algorithm is linear in $\epsilon^{-4}$ and large $c$ and $r$ imply small $\epsilon$.
However, the purpose of \CUR is to select a small number of columns and rows from the data matrix, that is, $c \ll n$ and $r \ll m$.
Thus we are not interested in the cases where $c$ and $r$ are large compared with $m$ and $n$, e.g., say $k = 20$ and $\alpha = 10$.

\section{Discussions}

In this paper we have proposed a novel randomized algorithm for the \CUR matrix decomposition problem.
This algorithm is faster, more scalable, and more accurate than the state-of-the-art algorithm, i.e., the subspace sampling algorithm.
Our algorithm requires only $c = {2k}{\epsilon^{-1}}(1+o(1))$ columns and $r = {2c}{\epsilon^{-1}}(1+o(1))$ rows to achieve ($1{+}\epsilon$) relative-error ratio.
To achieve the same relative-error bound,
the subspace sampling algorithm requires $c = \OM(k \epsilon^{-2} \log k)$ columns and $r = \OM(c \epsilon^{-2} \log c )$ rows selected from the original matrix.
Our algorithm also beats the subspace sampling algorithms in time-complexity.
Our algorithm costs $\OM (mnk \epsilon^{-2/3} + (m+n)k^3 \epsilon^{-2/3} + mk^2 \epsilon^{-2} + nk^2\epsilon^{-4})$ time,
which is lower than $\OM(\min\{mn^2 , m^2 n\})$ of the subspace sampling algorithms when $k$ is small.
Moreover, our algorithm enjoys another advantage of avoiding loading the whole data matrix into main memory,
which also makes our algorithm more scalable.
Finally, the empirical comparisons have also demonstrated the effectiveness and efficiency of our algorithm.

However, there are several open questions involving the lower bound of the \CUR matrix decomposition problem.
First, what is the lower bound for the \CUR problem? Second,
is there any algorithm achieving such a lower bound?
\citet{boutsidis2011NOCshort} proved a lower bound for the column selection problem:
$\frac{\|\A - \C \C^\dag \A\|_F^2}{\|\A - \A_k \|_F^2} \geq 1+ \frac{k}{c} $.
We thus wonder if there is a similar lower bound on the ratio $\frac{\|\A - \C \C^\dag \A \R^\dag \R \|_F^2}{\|\A - \C \C^\dag \A\|_F^2}$,
e.g., say $(1+\frac{\rk(\C)}{r})$.
We shall address these questions  in future work.

\acks{
This work has been supported in part by the Natural Science Foundations of
China (No. 61070239) and the Google visiting faculty program.}

\appendix

\section{The Dual Set Sparsification Algorithm} \label{sec:dualset}

For the sake of completeness, we attach the dual set sparsification algorithm here and describe some implementation details.
The dual set sparsification algorithms are deterministic algorithms established in \cite{boutsidis2011NOC}.
The fast \CUR algorithm calls the {\it dual set spectral-Frobenius sparsification algorithm}~\citep[Lemma~13 in][]{boutsidis2011NOC} in both stages.
We show this algorithm in Algorithm~\ref{alg:dual_set_sparsification}
and its bounds in Lemma~\ref{lem:sparsification}.

\begin{algorithm}[tb]
   \caption{Deterministic Dual Set Spectral-Frobenius Sparsification Algorithm.}
   \label{alg:dual_set_sparsification}
\algsetup{indent=2em}
\begin{small}
\begin{algorithmic}[1]
   \STATE {\bf Input:} 	$\UM = \{\x_i\}_{i=1}^n \subset \RB^l$, ($l < n$);
   							$\VM = \{\v_i\}_{i=1}^n \subset \RB^k$, with $\sum_{i=1}^n \v_i \v_i^T = \I_k$ ($k < n$);
   							$k < r < n$;
   \STATE {\bf Initialize:} $\s_0 = \0$, $\A_0 = \0$;
   \STATE Compute $\|\x_i\|_2^2$ for $i = 1, \cdots, n$, and then compute $\delta_U = \frac{ \sum_{i=1}^n \|\x_i\|_2^2 }{ 1 - \sqrt{k/r} }$;
   \FOR{$\tau = 0$ to $r-1$}
   \STATE Compute the eigenvalue decomposition of $\A_{\tau}$;
   	\STATE Find an index $j$ in $\{1 , \cdots , n\}$ and compute a weight $t > 0$ such that
   				\begin{eqnarray}
   				\delta_U^{-1} \| \x_j \|_2^2 \; \leq \;  t^{-1} \;  \leq \;  \frac{\v_j^T \Big(\A_\tau - (L_\tau + 1) \I_k  \Big)^{-2} \v_j  }{ \phi (L_\tau + 1 , \A_\tau) - \phi(L_\tau, \A_\tau) }
   														- \v_j^T \Big(\A_\tau - (L_\tau + 1) \I_k \Big)^{-1} \v_j	\textrm{;}	 \nonumber
   				\end{eqnarray}
   				where
   				\begin{equation}
   				\phi (L, \A)   =   \sum_{i=1}^k \Big(\lambda_i (\A) - L \Big)^{-1} \textrm{, }	\qquad	\qquad
   				\quad L_\tau = \tau - \sqrt{rk} \textrm{;} \nonumber
   				\end{equation} \label{alg:dual_set_sparsification:4}
   	\STATE Update the $j$-th component of $\s_\tau$ and $\A_\tau$: $\quad \s_{\tau + 1} [j] = \s_{\tau} [j] + t$, $\quad \A_{\tau + 1} = \A_\tau + t \v_j \v_j^T$;
   \ENDFOR
   \RETURN $\s = \frac{1 - \sqrt{k/r}}{r} \s_r$.
\end{algorithmic}
\end{small}
\end{algorithm}

\begin{lemma}[Dual Set Spectral-Frobenius Sparsification] \label{lem:sparsification}
Let $\UM = \{\x_1, \cdots, \x_n\} \subset \RB^l$ $(l < n)$ contain the columns of an arbitrary matrix $\X \in \RB^{l\times n}$.
Let $\VM = \{\v_1, \cdots, \v_n\} \subset \RB^k$ $(k < n)$ be a decompositions of the identity,
i.e., $\sum_{i=1}^n \v_i \v_i^T = \I_k$.
Given an integer $r$ with $k < r < n$,
Algorithm~\ref{alg:dual_set_sparsification} deterministically computes
a set of weights $s_i \geq 0$ ($i = 1,\cdots, n$) at most $r$ of which are non-zero, such that
\[
\lambda_k \Big( \sum_{i=1}^n s_i \v_i \v_i^T \Big) \geq \Big( 1 - \sqrt{\frac{k}{r}} \Big)^2 \qquad \mbox{ and }  \qquad
\tr \Big( \sum_{i=1}^n s_i \x_i \x_i^T \Big) \leq \|\X\|_F^2.
\]
The weights $s_i$  can be computed deterministically in $\OM (rnk^2 + nl)$ time.
\end{lemma}

Here we would like to mention the implementation of Algorithm~\ref{alg:dual_set_sparsification},
which is not described  by \cite{boutsidis2011NOC} in details.
In each iteration the algorithm performs once eigenvalue decomposition: $\A_\tau = \W \Lam \W^T$.
Here $\A_\tau$ is guaranteed to be positive semi-definite in each iteration.
Since
\[
\Big(\A_\tau - \alpha \I_k \Big)^q = \W \Diag\Big( (\lambda_1 - \alpha)^q , \cdots, (\lambda_k - \alpha)^q \Big) \W^T \textrm{,}
\]
we can efficiently compute $( \A_\tau - (L_\tau + 1) \I_k )^q$ based on the eigenvalue decomposition of $\A_\tau$.
With the eigenvalues at hand, $\phi (L, \A_\tau)$ can also be computed directly.

The algorithm runs in $r$ iterations.
In each iteration, the eigenvalue decomposition of $\A_\tau$ requires $\OM(k^3)$,
and the $n$ comparisons in Line~\ref{alg:dual_set_sparsification:4} each requires $\OM(k^2)$.
Moreover, computing $\|x_i\|_2^2$ for each $x_i$ requires $\OM(nl)$.
Overall, the running time of Algorithm~\ref{alg:dual_set_sparsification} is at most $\OM(r k^3) + \OM(r n k^2) + \OM(nl) = \OM(rnk^2 + nl)$.

\section{Proofs} \label{sec:proofs}

\subsection{The Proof of Theorem~\ref{thm:adaptive_bound}}

Theorem~\ref{thm:adaptive_bound} can be equivalently expressed in Theorem~\ref{thm:adaptive_bound_2}.
In order to stick to the column space convention of~\cite{boutsidis2011NOC},
we prove Theorem~\ref{thm:adaptive_bound_2} instead of Theorem~\ref{thm:adaptive_bound}.

\begin{theorem} [Adaptive Sampling Algorithm] \label{thm:adaptive_bound_2}
Given a matrix $\A \in \RBmn$ and
a matrix $\R \in \RB^{r\times n}$ such that $\rk(\R) = \rk(\A \R^\dag \R) = \rho$ $(\rho \leq r \leq m)$,
let $\C_1 \in \RB^{m\times c_1}$ consist of $c_1$ columns of $\A$,
and define the residual $\B = \A - \C_1 \C_1^\dag \A$.
For $i = 1,\cdots, n$, let
\[
p_i = \|\bb_i\|_2^2 / \|\B\|_F^2,
\]
where $\bb_i$ is the $i$-th column of the matrix $\B$.
Sample further $c_2$ columns from $\A$ in $c_2$ i.i.d. trials,
where in each trial the $i$-th column is chosen with probability $p_i$.
Let $\C_2 \in \RB^{m\times c_2}$ contain the $c_2$ sampled columns
and  $\C = [\C_1,\C_2] \in \RB^{m\times (c_1 + c_2)}$ contain the columns of both $\C_1$ and $\C_2$,
all of which are columns of $\A$.
Then the following inequality holds:
\[
\EB \|\A - \C \Cmp \A \R^\dag \R \|_F^2 \leq \|\A - \A \R^\dag \R \|_F^2 + \frac{\rho}{c_2} \|\A - \C_1 \C_1^\dag \A\|_F^2.
\]
where the expectation is taken w.r.t. $\C_2$.
\end{theorem}

\begin{proof}
With a little abuse of symbols, we use bold uppercase letters to denote matrix random variables and bold lowercase to denote vector random variables,
without distinguishing between matrix/vector random variables and constant matrices/vectors.

We denote the $j$-th column of $\V_{\A\R^\dag \R, \rho} \in \RB^{n \times \rho}$ as $\v_{j}$,
and the $(i,j)$-th entry of $\V_{\A\R^\dag \R, \rho}$ as $v_{i j}$.
Define vector random variables $\x_{j,(l)} \in \RB^{m}$  such that for $j = 1,\cdots,n$ and $l = 1,\cdots,c_2$,
\begin{equation}
\x_{j,(l)} = \frac{v_{i j}}{p_i} \bb_i = \frac{v_{i j}}{p_i} \Big( \a_i - \C_1 \C_1^\dag \a_i \Big)
\quad \textrm{with probability } p_i \textrm{,} \quad \textrm{for } i = 1, \cdots, n \textrm{,} \nonumber
\end{equation}
Note that $\x_{j,(l)}$ is a linear function of a column of $\A$ sampled from the above defined distribution.
We have that
\begin{eqnarray}
\EB [\x_{j,(l)}] 				& = & \sum_{i=1}^n p_i \frac{v_{i j}}{p_i} \bb_i \quad = \quad \B \v_{j} \textrm{,} \nonumber\\
\EB \| \x_{j, (l)} \|_2^2 	& = & \sum_{i=1}^n p_i \frac{v^2_{i j}}{p_i^2} \|\bb_i\|_2^2
\quad =\quad   \sum_{i=1}^n \frac{v^2_{i j}}{\|\bb_i\|_2^2 / \|\B\|_F^2 } \|\bb_i\|_2^2 \quad =\quad  \|\B\|_F^2 \textrm{.}\nonumber
\end{eqnarray}
Then we let $\x_j = \frac{1}{c_2} \sum_{l=1}^{c_2} \x_{j,(l)}$, we have
\begin{eqnarray}
\EB [\x_{j}]  & = & \EB [\x_{j,(l)}] \quad = \quad \B \v_{j} \textrm{,} \nonumber \\
\EB \| \x_{j} - \B \v_{j} \|_2^2 & = & \EB \Big\| \x_{j} - \EB[\x_{j}] \Big\|_2^2
= \frac{1}{c_2} \EB \Big\| \x_{j, (l)}  - \EB [\x_{j, (l)} ]  \Big\|_2^2
= \frac{1}{c_2} \EB \| \x_{j, (l)}  - \B \v_{j}  \|_2^2  \textrm{.}\nonumber
\end{eqnarray}
According to the construction of $\x_1, \cdots, \x_\rho$,
we define the $c_2$ columns of $\A$ to be $\C_2 \in \RB^{m\times c_2}$.
Note that all the random variables $ \x_1 \cdots, \x_\rho$ lie in the subspace $\span(\C_1)+ \span(\C_2)$.
We define random variables
\begin{equation}
\w_j = \C_1 \C_1^\dag \A \R^\dag \R \v_j + \x_j = \C_1 \C_1^\dag \A \v_j + \x_j \textrm{,}
\qquad \textrm{for } j = 1,\cdots,\rho \textrm{,} \nonumber
\end{equation}
where the second equality follows from Lemma~\ref{lem:right_singular_vector} that $\A \R^\dag \R \v_j = \A \v_j$
if $\v_j$ is one of the top $\rho$ right singular vectors of $\A \R^\dag \R$.
Then we have that any set of random variables $\{\w_1,\cdots,\w_\rho \}$ lies in $\span(\C) = \span(\C_1) + \span(\C_2)$.
Let $\W = [\w_1, \cdots , \w_\rho] $ be a matrix random variable,
we have that $\span(\W) \subset \span(\C)$.
The expectation of $\w_j$ is
\begin{equation}
\EB [\w_j] = \C_1 \C_1^\dag \A \v_j + \EB [\x_j] = \C_1 \C_1^\dag \A \v_j + \B \v_j = \A \v_j \textrm{,} \nonumber
\end{equation}
therefore we have that
\begin{equation}
\w_j - \A \v_j = \x_j - \B \v_j \textrm{.}\nonumber
\end{equation}
The expectation of $\|\w_j - \A \v_j\|_2^2$ is
\begin{eqnarray}
\EB \|\w_j - \A \v_j \|_2^2 	& = & \EB \| \x_j - \B \v_j \|_2^2
												\quad = \quad \frac{1}{c_2} \EB \| \x_{j,(l)} - \B \v_j \|_2^2 \nonumber \\
												& = & \frac{1}{c_2} \EB \| \x_{j,(l)} \|_2^2 - \frac{2}{c_2} (\B \v_j )^T \EB [\x_{j,(l)}] + \frac{1}{c_2} \|\B \v_j \|_2^2 \nonumber \\
												& = & \frac{1}{c_2} \EB \| \x_{j,(l)} \|_2^2 - \frac{1}{c_2}  \| \B \v_j \|_2^2 \nonumber
												\quad = \quad \frac{1}{c_2} \|\B\|_F^2 - \frac{1}{c_2}  \| \B \v_j \|_2^2 \nonumber \\
												&\leq & \frac{1}{c_2} \|\B\|_F^2 \label{eq:adaptive_bound:1}
\end{eqnarray}

To complete the proof,
we let the matrix variable
\[
\F = (\sum_{q=1}^\rho \sigma_q^{-1} \w_q \u_q^T ) \A \R^\dag \R \textrm{,}
\]
where $\sigma_q$ is the $q$-th largest singular value of $\A \R^\dag \R$
and $\u_q$ is the corresponding left singular vector of $\A \R^\dag \R$.
The column space of $\F$ is contained in $\span (\W) \subset \span(\C)$,
and thus
\[
\|\A\R^\dag \R - \C \C^\dag \A \R^\dag \R\|_F^2 \leq \|\A\R^\dag \R - \W \W^\dag \A \R^\dag \R\|_F^2 \leq \| \A \R^\dag \R - \F \|_F^2 \textrm{.}
\]
We use $\F$ to bound the error $\|\A \R^\dag \R - \C \C^\dag \A \R^\dag \R \|_F^2$:
\begin{eqnarray}
\EB \| \A - \C \C^\dag \A \R^\dag \R \|_F^2 	& = & \EB \| \A - \A \R^\dag \R + \A \R^\dag \R - \C \C^\dag \A \R^\dag \R \|_F^2 \nonumber \\
																& = & \EB \Big[\| \A - \A \R^\dag \R \|_F^2 + \| \A \R^\dag \R - \C \C^\dag \A \R^\dag \R \|_F^2 \Big] \label{eq:adaptive_bound:2} \\
																&\leq & \| \A - \A \R^\dag \R \|_F^2 + \EB \| \A \R^\dag \R - \F \|_F^2 \textrm{,}	 \nonumber
\end{eqnarray}
where (\ref{eq:adaptive_bound:2}) follows from that $\A(\I - \R^\dag \R)$
is orthogonal to $(\I - \C \C^\dag) \A \R^\dag \R$.
Since $\A \R^\dag \R$ and $\F$ both lies on the space spanned by the right singular vectors of ${\A\R^\dag \R}$, i.e. $\{\v_j \}_{j=1}^\rho$,
we decompose $\A \R^\dag \R - \F$ along $\{\v_j \}_{j=1}^\rho$:
\begin{eqnarray}
\EB \| \A - \C \C^\dag \A \R^\dag \R \|_F^2	&\leq & \| \A - \A \R^\dag \R \|_F^2 + \EB \| \A \R^\dag \R - \F \|_F^2 \textrm{,}	\nonumber \\
																& = & \| \A - \A \R^\dag \R \|_F^2 +  \sum_{j = 1}^\rho \EB \Big\| (\A \R^\dag \R - \F) \v_{j} \Big\|_2^2  \nonumber \\
																& = & \| \A - \A \R^\dag \R \|_F^2 +  \sum_{j = 1}^\rho \EB \Big\| \A \R^\dag \R \v_{j} - (\sum_{q=1}^\rho \sigma_q^{-1} \w_q \u_q^T ) \sigma_j \u_j  \Big\|_2^2 \nonumber \\
																& = & \| \A - \A \R^\dag \R \|_F^2 +  \sum_{j = 1}^\rho \EB \Big\| \A \R^\dag \R \v_{j} - \w_j  \Big\|_2^2 \nonumber \\
																& = & \| \A - \A \R^\dag \R \|_F^2 +  \sum_{j = 1}^\rho \EB \| \A \v_j - \w_j  \|_2^2 \label{eq:adaptive_bound:4}  \\
																&\leq & \| \A - \A \R^\dag \R \|_F^2 +  \frac{\rho}{c_2} \|\B\|_F^2 \textrm{,} \label{eq:adaptive_bound:5}
\end{eqnarray}
where (\ref{eq:adaptive_bound:4}) follows from Lemma~\ref{lem:right_singular_vector}
and (\ref{eq:adaptive_bound:5}) follows from (\ref{eq:adaptive_bound:1}).
\end{proof}

\begin{lemma} \label{lem:right_singular_vector}
We are given a matrix $\A \in \RBmn$ and a matrix $\R \in \RB^{r\times n}$ such that
$\rk(\A \R^\dag \R) = \rk(\R) = \rho$ $(\rho \leq r \leq m)$.
Letting $\v_j \in \RB^n$ be the $j$-th top right singular vector of $\A \R^\dag \R$,
we have that
\begin{equation}
\A \R^\dag \R \v_j = \A \v_j \textrm{,} \qquad \textrm{for } j = 1, \cdots, \rho  \textrm{.} \nonumber
\end{equation}
\end{lemma}

\begin{proof}
First let $\V_{\R,\rho} \in \RB^{n \times \rho}$ contain the top $\rho$ right singular vectors of $\R$,
then the projection of $\A$ onto the row space of $\R$ is $\A \R^\dag \R = \A \V_{\R,\rho} \V_{\R,\rho}^T$.
Let the thin SVD of $\A \V_{\R,\rho} \in \RB^{m\times \rho}$ be $\tilde{\U} \tilde{\Si} \tilde{\V}^T$,
where $\tilde{\V} \in \RB^{\rho \times \rho}$.
Then the compact SVD of $\A \R^\dag \R$ is
\[
\A \R^\dag \R = \A \V_{\R,\rho} \V_{\R,\rho}^T
= \tilde{\U} \tilde{\Si} \tilde{\V}^T \V_{\R,\rho}^T \textrm{.}
\]
According to the definition,
$\v_j$ is the $j$-th column of $(\V_{\R,\rho} \tilde{\V}) \in \RB^{n\times \rho}$,
and thus $\v_j$ lies on the column space of $\V_{\R,\rho}$,
and $\v_j$ is orthogonal to $\V_{\R,\rho\perp}$.
Finally, since $\A - \A \R^\dag \R = \A \V_{\R,\rho\perp} \V_{\R,\rho\perp}^T$,
we have that $\v_j$ is orthogonal to $\A - \A \R^\dag \R$,
that is, $(\A - \A \R^\dag \R) \v_j = \0$,
which directly proves the lemma.
\end{proof}

\subsection{The Proof of Theorem~\ref{thm:modified_fast_row_selection}}

\cite{boutsidis2011NOC} proposed a randomized algorithm which achieves the expected relative-error bound in Lemma~\ref{lem:fast_fro_norm_construction}.
This algorithm is described in Line~\ref{alg:fast_cur:line3} to \ref{alg:fast_cur:line6} of Algorithm~\ref{alg:fast_cur}.
Lemma~\ref{lem:fast_fro_norm_construction} is a direct corollary of Lemma~\ref{lem:rand_svd} and Lemma~\ref{prop:deterministic_fro}.
If we apply the same algorithm to $\A^T$ to select $c$ rows of $\A$ to form $\R_1$,
that is, Line~\ref{alg:fast_cur:line11} to \ref{alg:fast_cur:line13} of Algorithm~\ref{alg:fast_cur},
then a very similar bound is guaranteed.

\begin{lemma}[\cite{boutsidis2011NOC}, Theorem~4] \label{lem:fast_fro_norm_construction}
Given a matrix $\A \in \RB^{m\times n}$ of rank $\rho$, a target rank $2 \leq k < \rho$, and $0 < \epsilon_0 < 1$,
there is a randomized algorithm to select $c_1 > k$ columns of $\A$ and form a matrix $\C_1 \in \RB^{m\times c_1}$ such that
\begin{equation}
\EB \| \A - \C_1 \C_1^\dag \A \|_F^2 \leq (1+\epsilon_0) \Big( 1 + \frac{1}{( 1 - \sqrt{k/c_1} )^2} \Big) \|\A - \A_k\|_F^2 \textrm{,} \nonumber
\end{equation}
where the expectation is taken w.r.t. $\C_1$.
The matrix $\C_1$ can be computed in $\OM(mnk\epsilon_0^{-1} + n c_1 k^2)$ time.
\end{lemma}

With Theorem~\ref{thm:adaptive_bound} and Lemma~\ref{lem:fast_fro_norm_construction},
we now prove Theorem~\ref{thm:modified_fast_row_selection} as follows.

\begin{proof}
This randomized algorithm has three steps:
approximate SVD via randomized projection~\citep{halko2011ramdom},
deterministic column selection via dual set sparsification algorithm~\citep{boutsidis2011NOC} shown in Lemma~\ref{prop:deterministic_fro},
and the adaptive sampling algorithm of Theorem~\ref{thm:adaptive_bound} proved in this paper.
This algorithm is a generalization of the near-optimal column selection algorithm of Lemma~\ref{prop:fast_column_selection}.

Given $\A \in \RBmn$ and a target rank $k < r_1$,
step~1 (Line~\ref{alg:fast_cur:line3} of Algorithm~\ref{alg:fast_cur}) compute an approximate truncated SVD of $\A$ in $\OM(mnk/\epsilon_0)$ time such that $\A_k \approx \tilde{\A}_k = \tilde{\U}_k \tilde{\Si}_k \tilde{\V}_k^T$.
Lemma~\ref{lem:rand_svd} shows that
\begin{equation}
\EB \|\A - \tilde{\U}_k \tilde{\Si}_k \tilde{\V}_k^T \|_F^2 \leq (1+\epsilon_0) \|\A - \A_k\|_F^2 \textrm{.} \nonumber
\end{equation}

Step~2 (Line~\ref{alg:fast_cur:line11} to \ref{alg:fast_cur:line13} of Algorithm~\ref{alg:fast_cur}) selects $r_1$ rows of $\A$ to construct $\R_1$ by the dual set sparsification algorithm taking $\UM$ and $\VM$ as input,
where $\UM$ contains all the $m$ columns of $(\A^T - \tilde{\A}_k^T) \in \RB^{n\times m}$,
$\VM$ contains all the $m$ columns of $\tilde{\U}_{\A,k}^T \in \RB^{k\times m}$.
Lemma~\ref{lem:fast_fro_norm_construction} shows that
\begin{equation}
\EB \| \A - \A \R_1^\dag \R_1 \|_F^2 \leq (1+\epsilon_0) \Big( 1 + \frac{1}{( 1 - \sqrt{k/r_1} )^2} \Big) \|\A - \A_k\|_F^2 \textrm{,} \nonumber
\end{equation}
where the expectation is taken w.r.t. $\R_1$.
Step~2 costs $\OM(mr_1 k^2 + mn)$ time.

Step~3 (Line~\ref{alg:fast_cur:line14} to \ref{alg:fast_cur:line16} of Algorithm~\ref{alg:fast_cur}) samples additional $r_2$ rows of $\A$ to construct $\R_2 \in \RB^{r_2 \times n}$ by the adaptive sampling algorithm of Theorem~\ref{thm:adaptive_bound}.
Let $\R =  [\R_1^T , \R_2^T]^T \in \RB^{(r_1 + r_2)\times n}$.
We apply Theorem~\ref{thm:adaptive_bound} and have that
\begin{eqnarray}
\EB_\R \|\A - \C \Cmp \A \R^\dag \R \|_F^2& = & \EB_{\R_1} \Big[ \EB_{\R_2} \Big[ \|\A - \C \Cmp \A \R^\dag \R \|_F^2 \Big| \R_1 \Big]  \Big]\nonumber \\
													& \leq & \EB_{\R_1} \Big[ \|\A - \C \Cmp \A \|_F^2 + \frac{\rho}{r_2} \|\A - \A \R_1^\dag \R_1 \|_F^2 \Big] \nonumber \\
													& \leq & \|\A - \C \Cmp \A \|_F^2 + \frac{\rho}{r_2} (1+\epsilon_0) \Big( 1 + \frac{1}{( 1 - \sqrt{k/r_1} )^2} \Big) \|\A - \A_k\|_F^2 \textrm{.} \nonumber
\end{eqnarray}
By setting $r_1 = \OM(k \epsilon^{-2/3})$, $r_2 \approx \frac{2 \rho}{\epsilon}$, and $\epsilon_0 = \epsilon^{2/3}$,
we conclude that
\begin{eqnarray}
\EB \|\A - \C \Cmp \A \R^\dag \R \|_F^2 	\leq \|\A - \C \Cmp \A \|_F^2 + \epsilon \|\A - \A_k\|_F^2 \textrm{.} \nonumber
\end{eqnarray}
The total computation time of the three steps is $\OM(mnk/\epsilon_0 + mr_1 k^2 + mn) = \OM( (mnk + mk^3) \epsilon^{-2/3} )$
\end{proof}

\subsection{The Proof of Theorem~\ref{cor:fast_cur}}

\begin{proof}
Since $\C$ is constructed by columns of $\A$,
the column space of $\C$ is contained in the column space of $\A$,
so $\rk(\C \C^\dag \A) = \rk(\C) = \rho \leq c$,
and thus the assumptions of Theorem~\ref{thm:modified_fast_row_selection} are satisfied.
Lemma~\ref{prop:fast_column_selection} and Theorem~\ref{thm:modified_fast_row_selection} together prove Theorem~\ref{cor:fast_cur}:
\begin{eqnarray}
\EB^2 \|\A - \C \U \R\|_F & \leq & \EB \|\A - \C \U \R\|_F^2
								\quad = \quad \EB_{\C,\R} \| \A - \C \Cmp \A \R^\dag \R \|_F^2 \nonumber \\
								& = & \EB_{\C} \Big[ \EB_{\R} \Big[ \| \A - \C \Cmp \A \R^\dag \R \|_F^2  \Big| \C \Big] \Big]  \nonumber \\
								&\leq & \EB_{\C} \Big[ \|\A - \C \Cmp \A\|_F^2 + \epsilon \|\A - \A_k\|_F^2 \Big]  \nonumber \\
								&\leq & (1+2\epsilon) \| \A - \A_k \|_k^2 \textrm{.} \nonumber
\end{eqnarray}
Finally we have $\EB \|\A - \C \U \R\|_F \leq (1+\epsilon) \| \A - \A_k \|_k$
because $1+2\epsilon \leq (1+\epsilon)^2$.

The time cost of the fast \CUR algorithm is the sum of Stage~1, Stage~2, and the Moore-Penrose inverse of $\C$ and $\R$,
i.e. $\OM( (mnk + nk^3) \epsilon^{-2/3} ) + \OM( (mnk + mk^3) \epsilon^{-2/3} ) + \OM(mc^2) + \OM(nr^2)
= \OM (mnk \epsilon^{-2/3} + (m+n)k^3 \epsilon^{-2/3} + mk^2 \epsilon^{-2} + nk^2\epsilon^{-4})$.
\end{proof}

\bibliography{cur}

\begin{thebibliography}{27}
\providecommand{\natexlab}[1]{#1}
\providecommand{\url}[1]{\texttt{#1}}
\expandafter\ifx\csname urlstyle\endcsname\relax
  \providecommand{\doi}[1]{doi: #1}\else
  \providecommand{\doi}{doi: \begingroup \urlstyle{rm}\Url}\fi

\bibitem[Agarwala(2007)]{agarwala2007efficient}
Aseem Agarwala.
\newblock Efficient gradient-domain compositing using quadtrees.
\newblock In \emph{SIGGRAPH 2007}, 2007.

\bibitem[Ben-Israel and Greville(2003)]{adi2003inverse}
Adi Ben-Israel and Thomas~N.E. Greville.
\newblock \emph{Generalized Inverses: Theory and Applications. Second Edition}.
\newblock Springer, 2003.

\bibitem[Boutsidis et~al.(2011{\natexlab{a}})Boutsidis, Drineas, and
  Magdon-Ismail]{boutsidis2011NOC}
Christos Boutsidis, Petros Drineas, and Malik Magdon-Ismail.
\newblock Near-optimal column-based matrix reconstruction.
\newblock \emph{CoRR}, abs/1103.0995, 2011{\natexlab{a}}.

\bibitem[Boutsidis et~al.(2011{\natexlab{b}})Boutsidis, Drineas, and
  Magdon-Ismail]{boutsidis2011NOCshort}
Christos Boutsidis, Petros Drineas, and Malik Magdon-Ismail.
\newblock Near optimal column-based matrix reconstruction.
\newblock In \emph{Proceedings of the 2011 IEEE 52nd Annual Symposium on
  Foundations of Computer Science}, FOCS '11, pages 305--314,
  2011{\natexlab{b}}.

\bibitem[Deerwester et~al.(1990)Deerwester, Dumais, Furnas, Landauer, and
  Harshman]{deerwester1990lsa}
Scott Deerwester, Susan~T. Dumais, George~W. Furnas, Thomas~K. Landauer, and
  Richard Harshman.
\newblock Indexing by latent semantic analysis.
\newblock \emph{Journal of The American Society for Information Science},
  41\penalty0 (6):\penalty0 391--407, 1990.

\bibitem[Deshpande and Rademacher(2010)]{deshpande2010efficient}
Amit Deshpande and Luis Rademacher.
\newblock Efficient volume sampling for row/column subset selection.
\newblock In \emph{Proceedings of the 2010 IEEE 51st Annual Symposium on
  Foundations of Computer Science}, FOCS '10, pages 329--338, 2010.

\bibitem[Deshpande et~al.(2006)Deshpande, Rademacher, Vempala, and
  Wang]{deshpande2006matrix}
Amit Deshpande, Luis Rademacher, Santosh Vempala, and Grant Wang.
\newblock Matrix approximation and projective clustering via volume sampling.
\newblock \emph{Theory of Computing}, 2\penalty0 (2006):\penalty0 225--247,
  2006.

\bibitem[Drineas(2003)]{drineas2003pass}
Petros Drineas.
\newblock Pass-efficient algorithms for approximating large matrices.
\newblock In \emph{In Proceeding of the 14th Annual ACM-SIAM Symposium on
  Dicrete Algorithms}, pages 223--232, 2003.

\bibitem[Drineas and Mahoney(2005)]{drineas2005nystrom}
Petros Drineas and Michael~W. Mahoney.
\newblock On the {N}ystr\"{o}m method for approximating a gram matrix for
  improved kernel-based learning.
\newblock \emph{Journal of Machine Learning Research}, 6:\penalty0 2153--2175,
  2005.

\bibitem[Drineas et~al.(2006)Drineas, Kannan, and Mahoney]{drineas04fastmonte}
Petros Drineas, Ravi Kannan, and Michael~W. Mahoney.
\newblock Fast monte carlo algorithms for matrices iii: Computing a compressed
  approximate matrix decomposition.
\newblock \emph{SIAM Journal on Computing}, 36\penalty0 (1):\penalty0 184--206,
  2006.

\bibitem[Drineas et~al.(2008)Drineas, Mahoney, and
  Muthukrishnan]{drineas2008cur}
Petros Drineas, Michael~W. Mahoney, and S.~Muthukrishnan.
\newblock Relative-error {CUR} matrix decompositions.
\newblock \emph{SIAM Journal on Matrix Analysis and Applications}, 30\penalty0
  (2):\penalty0 844--881, September 2008.

\bibitem[Frank and Asuncion(2010)]{uci2010}
A.~Frank and A.~Asuncion.
\newblock {UCI} machine learning repository, 2010.
\newblock URL \url{http://archive.ics.uci.edu/ml}.

\bibitem[Frieze et~al.(2004)Frieze, Kannan, and Vempala]{frieze2004fast}
Alan Frieze, Ravi Kannan, and Santosh Vempala.
\newblock Fast monte-carlo algorithms for finding low-rank approximations.
\newblock \emph{Journal of the ACM}, 51\penalty0 (6):\penalty0 1025--1041,
  November 2004.
\newblock ISSN 0004-5411.

\bibitem[Goreinov et~al.(1997{\natexlab{a}})Goreinov, Tyrtyshnikov, and
  Zamarashkin]{goreinov1997pseudoskeleton}
S.~A. Goreinov, E.~E. Tyrtyshnikov, and N.~L. Zamarashkin.
\newblock A theory of pseudoskeleton approximations.
\newblock \emph{Linear Algebra and Its Applications}, 261:\penalty0 1--21,
  1997{\natexlab{a}}.

\bibitem[Goreinov et~al.(1997{\natexlab{b}})Goreinov, Zamarashkin, and
  Tyrtyshnikov]{goreinov1997maximalvolume}
S.~A. Goreinov, N.~L. Zamarashkin, and E.~E. Tyrtyshnikov.
\newblock Pseudo-skeleton approximations by matrices of maximal volume.
\newblock \emph{Mathematical Notes}, 62\penalty0 (4):\penalty0 619--623,
  1997{\natexlab{b}}.

\bibitem[Guruswami and Sinop(2012)]{Guruswami2012optimal}
Venkatesan Guruswami and Ali~Kemal Sinop.
\newblock Optimal column-based low-rank matrix reconstruction.
\newblock In \emph{Proceedings of the Twenty-Third Annual ACM-SIAM Symposium on
  Discrete Algorithms}, SODA '12, pages 1207--1214. SIAM, 2012.

\bibitem[Halko et~al.(2011)Halko, Martinsson, and Tropp]{halko2011ramdom}
Nathan Halko, Per-Gunnar Martinsson, and Joel~A. Tropp.
\newblock Finding structure with randomness: Probabilistic algorithms for
  constructing approximate matrix decompositions.
\newblock \emph{SIAM Review}, 53\penalty0 (2):\penalty0 217--288, 2011.

\bibitem[Hopcroft and Kannan(2012)]{hopcroft2012computer}
John Hopcroft and Ravi Kannan.
\newblock \emph{Computer Science Theory for the Information Age}.
\newblock 2012.

\bibitem[Kuruvilla et~al.(2002)Kuruvilla, Park, and
  Schreiber]{kuruvilla2002vector}
Finny~G. Kuruvilla, Peter~J. Park, and Stuart~L. Schreiber.
\newblock Vector algebra in the analysis of genome-wide expression data.
\newblock \emph{Genome Biology}, 3:\penalty0 research0011--research0011.1,
  2002.

\bibitem[Lowe(1999)]{lowe1999sift}
David~G. Lowe.
\newblock Object recognition from local scale-invariant features.
\newblock In \emph{Proceedings of the International Conference on Computer
  Vision}, ICCV 99, 1999.

\bibitem[Mackey et~al.(2011)Mackey, Talwalkar, and Jordan]{mackey2011divide}
Lester Mackey, Ameet Talwalkar, and Michael~I. Jordan.
\newblock Divide-and-conquer matrix factorization.
\newblock In \emph{Advances in Neural Information Processing Systems 24}. 2011.

\bibitem[Mahoney and Drineas(2009)]{mahoney2009matrix}
Michael~W. Mahoney and Petros Drineas.
\newblock {CUR} matrix decompositions for improved data analysis.
\newblock \emph{Proceedings of the National Academy of Sciences}, 106\penalty0
  (3):\penalty0 697--702, 2009.
\newblock URL \url{http://www.pnas.org/content/106/3/697.abstract}.

\bibitem[Serre et~al.(2007)Serre, Wolf, Bileschi, Riesenhuber, and
  Poggio]{serre07robustobject}
Thomas Serre, Lior Wolf, Stanley Bileschi, Maximilian Riesenhuber, and Tomaso
  Poggio.
\newblock Robust object recognition with cortex-like mechanisms.
\newblock \emph{IEEE Transactions on Pattern Analysis and Machine
  Intelligence}, 29:\penalty0 411--426, 2007.

\bibitem[Sirovich and Kirby(1987)]{Sirovich87eigenface}
L.~Sirovich and M.~Kirby.
\newblock Low-dimensional procedure for the characterization of human faces.
\newblock \emph{Journal of the Optical Society of America A}, 4\penalty0
  (3):\penalty0 519--524, Mar 1987.

\bibitem[Turk and Pentland(1991)]{turk1991eigenface}
Matthew Turk and Alex Pentland.
\newblock Eigenfaces for recognition.
\newblock \emph{Journal of Cognitive Neuroscience}, 3\penalty0 (1):\penalty0
  71--86, 1991.

\bibitem[Tyrtyshnikov(2000)]{tyrtyshnikov2000incompletecross}
Eugene~E. Tyrtyshnikov.
\newblock Incomplete cross approximation in the mosaic-skeleton method.
\newblock \emph{Computing}, 64:\penalty0 367--380, 2000.

\bibitem[Wang et~al.(2012)Wang, Li, Chang, and Yang]{wang2012data}
Zhiyu Wang, Fangtao Li, Edward~Y. Chang, and Shiqiang Yang.
\newblock A data-driven study on image feature extraction and fusion.
\newblock \emph{Manuscript}, 2012.

\end{thebibliography}
%\bibliographystyle{plain}
%\end{small}

\end{document}